\documentclass[11pt,reqno]{amsart}
\usepackage[utf8]{inputenc} 
\usepackage[T1]{fontenc}    
\usepackage{hyperref}       
\usepackage{url}            
\usepackage{booktabs}       
\usepackage{amsfonts}       
\usepackage{nicefrac}       
\usepackage{microtype}      
\usepackage{xcolor}         

\usepackage{caption,subfigure}
\usepackage{amsmath,mathrsfs,amsthm,amssymb,graphicx,float,xfrac,fullpage,bbm}
\usepackage{natbib}
\newcommand\pef[1]{(\ref{#1})}

\usepackage{lipsum}

\newcommand\blfootnote[1]{%
  \begingroup
  \renewcommand\thefootnote{}\footnote{#1}%
  \addtocounter{footnote}{-1}%
  \endgroup
}

\makeatletter
\numberwithin{equation}{section}
\numberwithin{figure}{section}
\theoremstyle{plain}
\newtheorem{thm}{\protect\theoremname}[section]

\newtheorem{defn}[thm]{Definition}
\theoremstyle{remark}
\newtheorem{rem}{Remark}

\newcommand{\cA}{\mathcal{A}}
\newcommand{\cC}{\mathcal{C}}

\newcommand{\cM}{\mathcal{M}}

\newcommand{\cN}{\mathcal{N}}
\newcommand{\cE}{\mathcal{E}}

\newcommand{\cL}{\mathcal{L}}

\newcommand{\cP}{\mathcal{P}}

\newcommand{\R}{\mathbb{R}}

\newcommand{\E}{\mathbb{E}}

\newcommand{\indicator}[1]{\mathbbm{1}_{#1}}

\newcommand{\tensor}{\otimes}

\newcommand{\abs}[1]{\lvert#1\rvert}
\newcommand{\norm}[1]{\lvert\lvert#1\rvert\rvert}

\usepackage{bm}

\DeclareMathOperator{\Tr}{Tr}

\DeclareMathOperator{\diag}{diag}

\DeclareMathOperator{\sgn}{sgn}
\DeclareMathOperator{\OO}{\mathcal{O}}

\DeclareMathOperator{\Sym}{Sym}

\theoremstyle{plain}
\newtheorem{prop}[thm]{Proposition}
\newtheorem{lem}[thm]{Lemma}
\newtheorem{cor}[thm]{Corollary}

\theoremstyle{remark}

\theoremstyle{remark}
\newtheorem{obs}[thm]{Observation}
\newtheorem{remark}[thm]{Remark}
\theoremstyle{definition}

\DeclareMathOperator{\op}{{\rm op}}

\newcommand{\bE}{\mathbb{E}}

\newcommand{\bP}{\mathbb{P}}

\newcommand{\bR}{{\mathbb R}}
\newcommand{\bS}{\mathbb S}

\newcommand{\la}{\lambda}

\newcommand{\bu}{\mathbf{u}}

\makeatother

\providecommand{\theoremname}{Theorem}

\begin{document}
\title[Spectral alignment of SGD in high-dimensional classification]{Spectral alignment of stochastic gradient descent for high-dimensional classification tasks}

\author{G\'erard Ben Arous}
\author{Reza Gheissari}
\author{Jiaoyang Huang}
\author{Aukosh Jagannath}

\address[G\'erard Ben Arous]{Courant Institute, New York University}
\email{gba1@nyu.edu}

\address[Reza Gheissari]{Department of Mathematics, Northwestern University}
\email{gheissari@northwestern.edu}

\address[Jiaoyang Huang]{Department of Statistics and Data Science, University of Pennsylvania}\email{huangjy@wharton.upenn.edu}

\address[Aukosh Jagannath]{Department of Statistics and Actuarial Science, Department of Applied Mathematics, and Cheriton School of Computer Science, University of Waterloo}
\email{a.jagannath@uwaterloo.ca}

\maketitle

\begin{abstract}
We rigorously study the relation between the training dynamics via stochastic gradient descent (SGD) and the spectra of empirical Hessian and gradient matrices. We prove that in two canonical classification tasks for multi-class high-dimensional mixtures and either 1 or 2-layer neural networks, both the SGD trajectory and emergent outlier eigenspaces of the Hessian and gradient matrices align with a common low-dimensional subspace. Moreover, in multi-layer settings this alignment occurs per layer, with the final layer's outlier eigenspace evolving over the course of training, and exhibiting rank deficiency when the SGD converges to sub-optimal classifiers. This establishes  some of the rich predictions that have arisen from extensive numerical studies in the last decade about the spectra of Hessian and information matrices over the course of training in overparametrized networks.
\end{abstract}

\vspace{-.1cm}
\section{Introduction}
Stochastic gradient descent (SGD) and its many variants, are the backbone of modern machine learning algorithms (see e.g.,~\citet{Bottou99}). The training dynamics of neural networks, however, are still poorly understood in the non-convex and high-dimensional settings that are frequently encountered. A common explanation for the staggering success of neural networks, especially when overparameterized, is that the loss landscapes that occur in practice have many “flat” directions and a hidden low-dimensional structure within which the bulk of training occurs.\blfootnote{An extended abstract presenting results from this paper appeared in \emph{International Conference on Learning Representations (ICLR) 2024} under the title ``High-dimensional SGD aligns with emerging outlier eigenspaces''.}

To understand this belief, much attention has been paid to the Hessian of the empirical risk (and related matrices formed via gradients) along training. This perspective on the training dynamics of neural networks was proposed in~\citet{LeCun-EfficientBackprop}, and numerically analyzed in depth in~\citet{Sagun-Singularity-and-beyond,SagunEtAl}. The upshot of these studies was a broad understanding of the spectrum of the Hessian of the empirical risk, that we summarize as follows: 
\begin{enumerate}
    \item It has a \emph{bulk} that is dependent on the network architecture, and is concentrated around $0$, becoming more-so as the model becomes more overparametrized;
    \item It has (relatively) few \emph{outlier eigenvalues} that are dependent on the data, and evolve non-trivially along training while remaining separated from the bulk;
\end{enumerate}
Since those works, these properties of the Hessian and related spectra over the course of training have seen more refined and large-scale experimentation. ~\citet{Papyan-3-level} found a hierarchical decomposition to the Hessian for deep networks, attributing the bulk, emergent outliers, and a \emph{minibulk} to three different ``parts" of the Hessian. 
Perhaps most relevant to this work,~\citet{GD-in-tiny-subspace} noticed that gradient descent tends to quickly align with a \emph{low-dimensional} outlier subspace of the Hessian matrix, and stay in that subspace for long subsequent times. They postulated that this common low-dimensional structure to the SGD and Hessian matrix may be key to many classification tasks in machine learning. 
For a sampling of other empirical investigations of spectra of Hessians and information matrices along training, see e.g.,~\citet{Ghorbani-Hessian-eigenvalue-density,Papyan-Class-CrossClass,Li-etal-Hessian-based-analysis,MartinMahoney,EdgeOfStability-GD,xie2023on}.

From a theoretical perspective, much attention has been paid to the Hessians of deep networks using random matrix theory approaches. Most of this work has focused on the spectrum at a fixed point in parameter space, most commonly at initialization.
Early works in the direction include~\citet{Watanabe-FIM,Dauphin-et-al}. \citet{Choromanska-PMLR} noted similarities of neural net Hessians to spin glass Hessians, whose complexity (exponential numbers of critical points) has been extensively studied see e.g.,~\citet{ABA13,ABC13}. More recently, the expected complexity has been investigated in statistical tasks like tensor PCA and generalized linear estimation~\citet{BMMN17,maillard-landscape-complexity}. 
 In~\citet{PenningtonWorah-FIM} and~\citet{PenningtonBahri}, the empirical spectral distribution of Hessians and information matrices in single-layer neural networks were studied at initialization. ~\citet{liao2021hessian} studied Hessians of some non-linear models which they referred to as generalized GLMs, also at initialization, and, after our work first appeared,~\cite{GarrodKeating} studied the Hessian of a deep linear unconstrained feature model at the global minimizer of the loss. 
 
 Under the infinite-width neural tangent kernel limit of~\citet{JacotGabrielHongler-NTK},~\citet{ZhouWang} derived the empirical spectral distribution at initialization, and~\citet{Jacot-NTK-Hessian} studied its evolution over the course of training. In this limit the input dimension is kept fixed compared to the parameter dimension, while our interest in this paper is when the input dimension, parameter dimension, and number of samples all scale together. 

An important step towards understanding the evolution of Hessians along training in the high-dimensional setting, is understanding the training dynamics themselves. 
Since the classical work of~\cite{RobMon51}, there has been much activity studying limit theorems for stochastic gradient descent. In the high-dimensional setting, following~\citet{saad1995dynamics,saad1995line}, investigations have focused on finding a finite number of functions (sometimes called ``observables'' or ``summary statistics''), whose dynamics under the SGD are asymptotically autonomous in the high-dimensional limit. For a necessarily small sampling of this rich line of work,. we refer to~\cite{goldt2019dynamics,veiga2022phase,paquette2021sgd,pmlr-v195-arnaboldi23a,TanVershynin,BGJ22}. Of particular relevance, it was shown by~\cite{DamianLeeSoltanolkotabi-22,mousavi-hosseini2023neural} that for multi-index models, SGD predominantly lives in the low-dimensional subspace spanned by the ground truth parameters.

A class of tasks whose SGD is amenable to this broad approach is classification of Gaussian mixture models (GMMs). With various losses and linearly separable class structures, the minimizer of the empirical risk landscape with single-layer networks was studied in~\citet{Mignacco-GMM,loureiro2021learning}. A well-studied case of a Gaussian mixture model needing a two-layer network is under an XOR-type class structure; the training dynamics of SGD for this task were studied in~\citet{refinetti2021classifying} and~\citet{BGJ22} and it was found to have a particularly rich structure with positive probability of convergence to bad classifiers among other degenerate phenomena.

Still, a simultaneous understanding of high-dimensional SGD and the Hessian and related matrices' spectra along the training trajectory has remained largely open.

\subsection{Our contributions}
In this paper, we study the interplay between the training dynamics (via SGD) and the spectral compositions of the empirical Hessian matrix and an empirical gradient second moment matrix, or simply \emph{G-matrix}~\pef{eq:test-Hessian-Gram} (similar in spirit to an information matrix) over the course of training. 
We rigorously show the following phenomenology in two canonical high-dimensional classification tasks with $k$ ``hidden'' classes:  
 \begin{enumerate}
    \item Shortly into training, the empirical Hessian and empirical G-matrices have $C(k)$ many outlier eigenvalues. Their corresponding eigenvectors, along with the SGD trajectory, align with a common latent $C(k)$-dimensional subspace. In particular, the SGD and outlier eigenspaces align well with one another.
    Here $C(k)$ is explicit and depends on the model and the performance of the classifier to which the SGD converges.
    \item In multi-layer settings, this alignment happens within each layer, i.e., the first layer parameters align with the outlier eigenspaces of the corresponding blocks of the empirical Hessian and G-matrices, and likewise for the second layer parameters. 
    \item This alignment is \emph{not} predicated on success at the classification task: when the SGD converges to a sub-optimal classifier, the empirical Hessian and G matrices have lower rank outlier eigenspaces, and the SGD aligns with those rank deficient spaces. 
\end{enumerate}

The first model we consider is the basic example of supervised classification of general $k$-component Gaussian mixture models with $k$ linearly independent classes by a single-layer neural network. In Theorem~\ref{mainthm:SGD-aligns-with-Hessian}, we establish alignment of the form of Item 1 above, between each of the $k$ one-vs-all classifiers and their corresponding blocks in the empirical Hessian and G-matrices. See also the depictions in Figures~\ref{fig:KGMM-topspaces}--\ref{fig:KGMM-snapshots}. 
To show this, we show that the matrices have an outlier-minibulk-bulk structure throughout the parameter space, and derive limiting dynamical equations for the trajectory of appropriate summary statistics of the SGD trajectory. Importantly, the same low-dimensional subspace is at the heart of both the outlier eigenspaces and the summary statistics. At this level of generality, the SGD can behave very differently within the outlier eigenspace. As an example of the refined phenomenology that can arise, we further investigate the special case where the means are orthogonal; here the SGD aligns specifically with the single largest outlier eigenvalue, which itself has separated from the other $k-1$ outliers along training. 
This is proved in Theorem~\ref{mainthm:topeigenvector} and depicted in Figure~\ref{fig:DBBP-kgmm}. These results are presented in Section~\ref{subsec:1-layer-k-GMM}.

To demonstrate our results in more complex multi-layer settings, we consider supervised classification of a GMM version of the famous XOR problem of~\citet{minsky1969introduction}. This is one of the simplest models that requires a two-layer neural network to solve. We use a two-layer architecture with a second layer of width $K$. As indicated by Item 2 above, in Theorems~\ref{mainthm:XOR-lives-in-Hessian}--\ref{mainthm:XOR-all-live-in-subspace} the alignment of the SGD with the matrices' outlier eigenspaces occurs within each layer, the first layer having an outlier space of rank two, and the second layer having an outlier space of rank $4$ when the dynamics converges to an  optimal classifier. This second layer's alignment is especially rich, as when the model is overparametrized ($K$ large), its outlier space of rank $4$ is not present at initialization, and only separates from its rank-$K$ bulk over the course of training. This can be interpreted as a dynamical version of the well-known spectral phase transition in spiked covariance matrices of~\citet{baik2005phase}: see Figure~\ref{fig:DBBP-XOR} for a visualization. Moreover, the SGD for this problem is known to converge to sub-optimal classifiers with probability bounded away from zero~\citet{BGJ22}, and we find that in these situations, the alignment still occurs but the outlier eigenspaces into which the SGD moves are rank-deficient compared to the number of hidden classes, $4$: see Figure~\ref{fig:DBBP-XOR-deficient}. These results are presented in Section~\ref{subsec:XOR-GMM}.

\section{Main Results}\label{sec:main-results}

Let us begin by introducing the following general framework and notation. 
We suppose that we are given data from a distribution $\cP_Y$ over pairs $\mathbf Y = (y,Y)$ where $y\in \mathbb R^k$ is a one-hot ``label" vector that takes the value $1$ on a class (sometimes identified with the element of $[k]=\{1,...,k\}$ on which it is $1$), and $Y\in \mathbb R^d$ is a corresponding feature vector. In training we take as loss a function of the form $L(\mathbf{x},\mathbf{Y}): \mathbb R^p\times \mathbb R^{k+d} \to \mathbb R_+\,,$
where $\mathbf{x}\in \mathbb R^p$ represents the network parameter. (As we are studying supervised classification, in both settings this loss will be the usual cross-entropy loss corresponding to the architecture used.)

We imagine we have two data sets, a training set $(\mathbf{Y}^\ell)_{\ell=1}^{M}$ and a test set $(\widetilde {\mathbf{Y}}^\ell)_{\ell =1}^{\widetilde M}$, all drawn i.i.d.\ from $\cP_Y$.
Let us first define the stochastic gradient descent trained using $(\mathbf{Y}^{\ell})$. In order to ensure the SGD doesn't go off to infinity we add an $\ell^2$ penalty term (as is common in practice) with Lagrange multiplier $\beta$. The (online) stochastic gradient descent with initialization $\mathbf{x}_0$ and learning rate, or step-size, $\delta$, will be run using the training set $(\mathbf{Y}^\ell)_{\ell=1}^{M}$ as follows: 
\begin{align}\label{eq:SGD-def}
    \mathbf{x}_\ell = \mathbf{x}_{\ell-1} -  \delta \nabla L(\mathbf{x}_{\ell-1},\mathbf{Y}^\ell) - \beta \mathbf{x}_{\ell-1}\,.
\end{align}

Our aim is to understand the behavior of SGD with respect to principal subspaces, i.e., outlier eigenvectors, of the empirical Hessian matrix and empirical second moment matrix of the gradient. This latter matrix is exactly the information matrix when $L$ is the log-likelihood; in our paper $L$ is taken to be a cross-entropy loss, so we simply refer to this as the G-matrix henceforth. We primarily consider the empirical Hessian and empirical G-matrices generated out of the test data, namely: 
\begin{align}
    \nabla^2 \widehat R(\mathbf{x}) & =  \frac{1}{\widetilde M} \sum_{\ell = 1}^{\widetilde M} \nabla^2 L(\mathbf{x}, \widetilde{\mathbf{Y}}^\ell)\,, \qquad \text{and}\qquad  \widehat G(\mathbf{x})  = \frac{1}{\widetilde M} \sum_{\ell =1}^{\widetilde M} \nabla L(\mathbf{x},\widetilde{\mathbf{Y}}^\ell)^{\otimes 2}\,. \label{eq:test-Hessian-Gram}
\end{align}
 (Since we are working in the online setting, it is just as natural to generate these matrices with test data as with train data. See Remark~\ref{rem:test-vs-train} for how our results extend when training data is used.)
When the parameter space naturally splits into subsets of its indices (e.g., the first-layer weights and the second-layer weights), for a subset $I$ of the parameter coordinates, we use subscripts $\nabla^2_{I,I} \widehat R$ and $\widehat G_{I,I}$ to denote the block corresponding to that subset. Note that since the penalty term $\beta \|\mathbf{x}\|^2$ is not included in $L$, it does not show up in~\eqref{eq:test-Hessian-Gram}.  This convention matches the literature; note, however, that including this term would simply shift the spectrum of the Hessian by $2\beta$.

To formalize the notion of alignment between the SGD and the principal directions of the Hessian and G-matrices, we introduce the following language.  For a subspace $B$, we let $P_B$ denote the orthogonal projection onto $B$; for a vector $v$, we let $\|v\|$ be its $\ell^2$ norm; and for a matrix $A$, let $\|A\|=\|A\|_{\op}$ be its $\ell^2 \to \ell^2$ operator norm.
\begin{defn}\label{def:lives-in-span}
    The \emph{alignment} of a vector $v$ with  a subspace $B$ is the ratio $\rho(v,B) = \norm{P_{B}v}/\norm{v}$. 
    We say a vector $v$ \emph{lives in a subspace $B$} up to error $\varepsilon$ if 
    $\rho(v,B)\geq 1-\varepsilon$.
\end{defn}

For a matrix $A$, we let $E_{k}(A)$ denote the span of the top $k$ eigenvectors of $A$, i.e., the span of the $k$ eigenvectors of $A$ with the largest absolute values. We also use the following. 

\begin{defn}\label{def:matrix-lives-in-subspace}
    We say a matrix $A$ \emph{lives in} a subspace $B$. up to error $\varepsilon$ if there exists $M$ such that $\text{Im}(A-M)\subset B$ with $\|M\|_{\op}\le \varepsilon\|A\|_{\op}$, where $\|A \|_{\op}$ denotes the $\ell^2$-to-$\ell^2$ operator norm. 
\end{defn}

\subsection{Classifying linearly separable mixture models}\label{subsec:1-layer-k-GMM}

We begin by illustrating our results on (arguably) the most basic problem of high-dimensional multiclass classification, namely supervised classification of a $k$ component Gaussian mixture model with constant variance and linearly independent means using a single-layer network. (This is sometimes used as a toy model for the training dynamics of the last layer of a deep network via the common ansatz that the output of the second-to-last layer of a deep network behaves like a linearly separable mixture of Gaussians: see e.g., the neural collapse phenomenon posited by~\citet{PapyanDonoho}.)

\subsubsection{Data model}
Let $\cC = [k]$ be the collection of classes, with corresponding distinct  class means $(\mu_a)_{a\in [k]}\in \mathbb R^d$, covariance matrices $I_d/\lambda$, where $\lambda>0$ can be viewed a signal-to-noise parameter, and corresponding probabilities $0< (p_a)_{a\in [k]}  < 1$ such that $\sum_{a\in [k]} p_a =1$. The number of classes $k = O(1)$ is fixed (here and throughout the paper $o(1)$, $O(1)$ and $\Omega(1)$ notations are with respect to the dimension parameter $d$, and may hide constants that are dimension independent such as $k,\beta$).

For the sake of simplicity we take the means to be unit norm. Further, in order for the task to indeed be solvable with the single-layer architecture, we assume that the means are linearly independent, say with a fixed (i.e., $d$-independent) matrix of inner products $(\overline{m}_{ab})_{a,b} = (\mu_a \cdot \mu_b)_{a,b}$. 
Our data distribution $\cP_Y$ is a mixture of the form $\sum_{c}p_{c}\cN(\mu_{c},I_d/\lambda)$, with an accompanying class label $y\in \mathbb R^k$. Namely,  our data is given as $\mathbf{Y} = (y,Y)$ where: 
\begin{align}\label{eq:data-distribution}
	y \sim \sum_{a\in [k]} p_a \delta_{\mathbf 1_a}\,, \qquad \mbox{and} \qquad Y \sim \sum_{a\in [k]} y_{a} \mu_a + Z_\lambda\,,
\end{align}
and where $Z_\lambda \sim \mathcal N(0,I_d/\lambda)$. 

We perform classification by training a single-layer network formed by $k$ ``all-vs-one'' classifiers using the cross entropy loss (equivalently, we are doing multi-class logistic regression):
\begin{align}\label{eq:cross-entropy-loss}
L(\mathbf{x},\mathbf{Y})=-\sum_{c\in[k]}y_{c}{x}^{c}\cdot Y+\log\sum_{c\in [k]}\exp({x}^{c}\cdot Y)\,,
\end{align}
where $\mathbf{x}=(x^{c})_{c\in\cC}$ are the parameters, each of which
is a vector in $\R^{d}$, i.e., $\mathbf{x}\in\R^{dk}$. (Note that
we can alternatively view $\mathbf{x}$ as a $k\times d$ matrix.)

\subsubsection{Results and discussion}

Our first result is that after some linearly many steps, the SGD finds the subspace generated by the outlier eigenvalues of the Hessian and/or G-matrix of the test loss and lives there for future times.

\begin{thm}\label{mainthm:SGD-aligns-with-Hessian}
    Consider the mixture of $k$-Gaussians with loss function from~\pef{eq:cross-entropy-loss}, and SGD~\pef{eq:SGD-def} with learning rate $\delta = O(1/d)$, regularizer $\beta>0$, initialized from $\cN(0,I_d/d)$.  There exists $\alpha_0, \lambda_0$ such that if $\lambda\ge \lambda_0$,  and $\widetilde M \ge \alpha_0 d$, the following hold. For every $\varepsilon>0$, there exists $T_0(\varepsilon)$ such that for any fixed time horizon $T_0 < T_f < M/d$, with probability $1-o_d(1)$, 
    \begin{align*}
        \mathbf{x}_\ell^c & \text{ lives in } E_k(\nabla_{cc}^2 \widehat R(\mathbf{x}_\ell)) \text{ and in } E_k(\widehat G_{cc} (\mathbf{x}_\ell))\,,
    \end{align*}
    for every $c\in [k]$, up to $O(\varepsilon + \lambda^{-1})$ error, 
    for all $\ell \in [T_0\delta^{-1} ,T_f \delta^{-1}]$. 
\end{thm}

\begin{figure}[t]
    \centering

   \subfigure{\includegraphics[width=.35\textwidth]{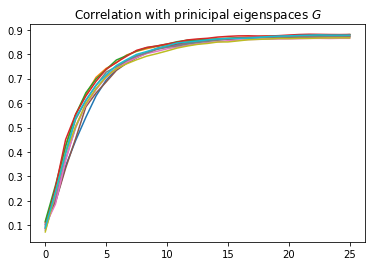}}
   \qquad
\subfigure{        \includegraphics[width=.35\textwidth]{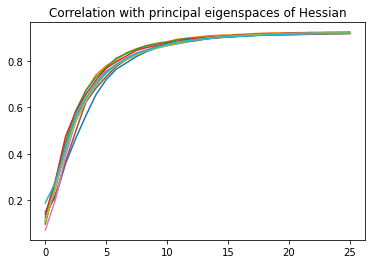}}
      \caption{The alignment of the SGD trajectory $\mathbf{x}_\ell^c$ with $E_k(\nabla^2_{cc}\widehat R(\mathbf{x}_\ell))$ (left) and $E_k(\widehat G_{cc}(\mathbf{x}_\ell))$ (right), for $c\in [k]$ (shown in different colors). The $x$-axis is rescaled time, $\ell \delta$. The parameters are $k=10$ classes in dimension $d=1000$ with $\lambda=10$, $\beta = 0.01$, and $\delta = 1/d$.}\label{fig:KGMM-topspaces}
\end{figure}

This result is demonstrated in Figure~\ref{fig:KGMM-topspaces} which plots the alignment of the training dynamics $\mathbf{x}_\ell^c$ with the principal eigenspaces of the Hessian and $G$ for each $c\in [k]$. As we see the alignment increases to near 1 rapidly for all blocks in both matrices. This theorem, and all our future results, are stated using a random Gaussian initialization, scaled such that the norm of the parameters is $O(1)$ in $d$. The fact that this is Gaussian is not relevant to the results, and similar results hold for other uninformative initializations with norm of $O(1)$.

Theorem~\ref{mainthm:SGD-aligns-with-Hessian} follows from the following theorem that describes the SGD trajectory, its Hessian and G-matrix (and their top $k$ eigenspaces), all live up to $O(1/\lambda)$ error in $\text{Span}(\mu_1,...,\mu_k)$.

\begin{thm}\label{mainthm:all-align-with-means}
	In the setup of Theorem~\ref{mainthm:SGD-aligns-with-Hessian}, the following live in $\text{Span}(\mu_1,...,\mu_k)$ up to $O(\varepsilon + \lambda^{-1})$ error with probability $1-o_d(1)$: 
	\begin{enumerate}
		\item The state of the SGD along training, $\mathbf{x}_\ell^c$ for every $c$;
		\item The $b,c$ blocks of the empirical test Hessian, $\nabla^2_{bc} \widehat R(\mathbf{x}_\ell)$ for all $b,c\in [k]$; 
		\item The $b,c$ blocks of the empirical test G-matrix $\widehat G_{bc}(\mathbf{x}_\ell)$ for all $b,c\in [k]$. 
	\end{enumerate}
\end{thm}

We demonstrate this result in Figure~\ref{fig:KGMM-snapshots}, which shows the coordinate-wise values of a fixed block of the SGD, the Hessian, and the G-matrix. 

\begin{figure}[t]
    \centering
   \subfigure{\includegraphics[width=.3\textwidth]{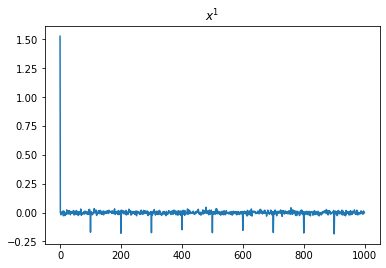}}
\subfigure{\includegraphics[width=.3\textwidth]{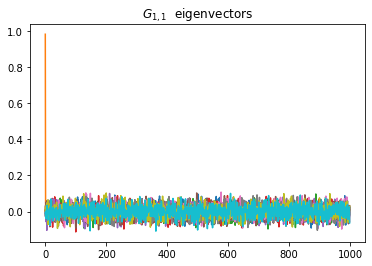}}
\subfigure{\includegraphics[width=.3\textwidth]{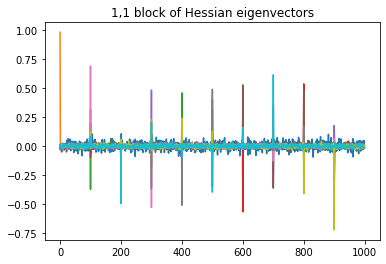}}
      \caption{From left to right: Plot of entries of $\mathbf{x}_\ell^1$ and the $k$ leading eigenvectors  (in different colors)  of $\nabla_{11}^2\widehat R(\mathbf{x}_\ell)$ and $\widehat G_{11}(\mathbf{x}_\ell)$ respectively at the end of training, namely $\ell = 50\cdot d=25,000$ steps. Here the $x$-axis represents the coordinate index. The parameters are the same as in Fig.~\ref{fig:KGMM-topspaces} and the means are $\mu_i = e_{i*50}$.}\label{fig:KGMM-snapshots}
\end{figure}

Inside the low-rank space spanned by $\mu_1,...,\mu_k$, the training dynamics,  Hessian and G-matrix spectra can display different phenomena depending on the relative locations of $\mu_1,...,\mu_k$ and weights $p_1,...,p_k$. To illustrate the more refined alignment phenomena, let us take as a concrete example $p_c = \frac{1}{k}$ for all $c$, and  $\mu_1,...,\mu_k$ orthonormal. 

With this concrete choice, we can analyze the limiting dynamical system of the SGD without much difficulty and its relevant dynamical observables have a single stable fixed point, to which the SGD converges in linearly many, i.e., $O(\delta^{-1})$, steps. This allows us to show more precise alignment that occurs within $\text{Span}(\mu_1,...,\mu_k)$ over the course of training. 

\begin{thm}\label{mainthm:topeigenvector}
In the setting of Theorem~\ref{mainthm:SGD-aligns-with-Hessian}, with the means $(\mu_1,...,\mu_k)$ being orthonormal, the estimate $\mathbf{x}_\ell^c$ has $\Omega(1)$, positive, inner product with the top eigenvector of both $\nabla^2_{cc} \widehat R(\mathbf{x}_\ell)$ and $\widehat G_{cc}(\mathbf{x}_\ell)$ (and negative, $\Omega(1)$ inner product with the $k-1$ next largest eigenvectors). 
Also, the top eigenvector of $\nabla^2_{cc} \widehat R(\mathbf{x}_\ell)$, as well as that of $\widehat G_{cc}(\mathbf{x}_\ell)$, live in $\text{Span}(\mu_c)$ up to $O(\varepsilon + \lambda^{-1})$ error. 
\end{thm}

Put together, the above three theorems describe the following rich scenario for classification of the $k$-GMM. At initialization, and throughout the parameter space in each class-block, the Hessian and G-matrices decompose into a rank-$k$ outlier part spanned by $\mu_1,...,\mu_k$, and a correction term of size $O(1/\lambda)$ in operator norm. Furthermore, when initialized randomly, the SGD is not aligned with the outlier eigenspaces, but does align with them in a short $O(\delta^{-1})$ number of steps. Moreover, when the means are orthogonal, each class block of the SGD $\mathbf{x}_\ell^c$ in fact correlates strongly with the specific mean for its class $\mu_c$, and simultaneously, in the Hessian and G-matrices along training, the eigenvalue corresponding to $\mu_c$ becomes distinguished from the other $k-1$ outliers. 
We illustrate these last two points in Figure~\ref{fig:DBBP-kgmm}. We also refer the reader to Section~\ref{sec:additional figures} for further numerical demonstrations. 
Each of these phenomena appear in more general contexts, and in even richer manners, as we will see in the following section.

\begin{figure}[t]
    \centering
\includegraphics[width=.35\textwidth]{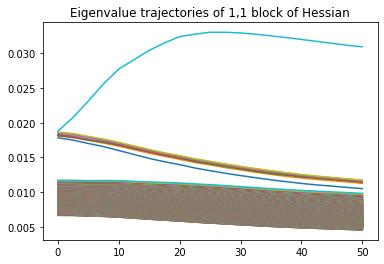}
\qquad 
\includegraphics[width=.35\textwidth]{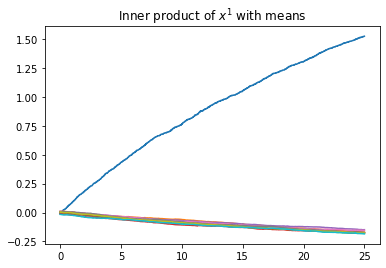}
      \caption{Left: the eigenvalues (in different colors) of $\nabla^2 \widehat R_{11}(\mathbf{x}_\ell)$ over the course of training.  The leading $k$ eigenvalues are separated from the bulk at all times, and the top eigenvalue, corresponding to $\mu_1$ separates from the remaining eigenvalues soon after initialization. Right: the inner product of $\mathbf{x}_\ell^1$ with the means $\mu_1,...,\mu_k$ undergoes a similar separation over the course of training. Parameters are the same as in preceding figures.}
    \label{fig:DBBP-kgmm}
\end{figure} 

\subsection{Classifying XOR-type mixture models via two-layer networks }\label{subsec:XOR-GMM}
With the above discussion in mind, let us now turn to more complex classification tasks that are not linearly separable and require the corresponding network architecture to be multilayer. 

\subsubsection{Data model}
For our multilayer results, we consider the problem of classifying a $4$-component Gaussian mixture 
whose class labels are in a so-called XOR form. 
More precisely, consider a mixture of four Gaussians with means $\mu,-\mu, \nu,-\nu$ where $\|\mu\|= \|\nu\|=1$  and, say for simplicity, are orthogonal, and variances $I_d/\lambda$. There are two classes, class label $1$ for Gaussians with mean $\pm \mu$, and $0$ for Gaussians with mean $\pm \nu$. To be more precise, our data distribution $\cP_Y$ is 
\begin{align}\label{eq:XOR-data-distribution}
	y \sim \frac{1}{2} \delta_0 + \frac{1}{2}\delta_1 \quad \text{and}\quad Y \sim \begin{cases} \frac{1}{2} \cN(\mu ,I_d/\lambda) +\frac{1}{2} \cN(-\mu ,I_d/\lambda) & \quad y = 1 \\  \frac{1}{2} \cN(\nu ,I_d/\lambda) +\frac{1}{2} \cN(-\nu ,I_d/\lambda) & \quad y =0\end{cases}\,. 
\end{align}

This is a Gaussian version of the XOR problem of~\citet{minsky1969introduction}. It is one of the simplest examples of a classification task requiring a multi-layer network to express a good classifier. 

We therefore use a two-layer architecture with the intermediate layer having width $K\ge 4$ (any less and a Bayes-optimal classifier would not be expressible), ReLu activation function $g(x) = x\vee 0$ and then sigmoid activation function $\sigma(x) = \frac{1}{1+e^{-x}}$. The parameter space is then $\mathbf{x}=(W,v) \in \mathbb R^{Kd + K}$, where the first layer weights are denoted by $W = W(\mathbf{x})\in \mathbb R^{K\times d}$ and the second layer weights are denoted by $v = v(\mathbf{x}) \in \mathbb R^K$. 
We use the binary cross-entropy loss on this problem,   
\begin{align}\label{eq:XOR-loss}
	L(\mathbf{x}; \mathbf{Y}) = - y v \cdot g(WY) + \log (1+e^{ v\cdot g(WY)}) \,,
\end{align}  
with $g$ applied entrywise. 
The SGD for this classification task was studied in some detail in~\citet{refinetti2021classifying} and~\citet{BGJ22}, with ``critical" step size $\delta = \Theta(1/d)$.

\subsubsection{Results and discussion}
We begin our discussion with the analogue of Theorem~\ref{mainthm:SGD-aligns-with-Hessian} in this setting.  As the next theorem demonstrates, in this more subtle problem, the SGD still finds and lives in the principal directions of the empirical Hessian and G-matrices.\footnote{As the second derivative of ReLU is only defined in the sense of distributions, care must be taken when studying these objects. This singularity is not encountered by the SGD trajectory almost surely.}
Here, the principal directions vary significantly across the parameter space, and the alignment phenomenon differs depending on which fixed point the SGD converges to. Moreover, the relation between the SGD and the principal directions of the Hessian and G-matrices can be seen per layer.

\begin{thm}\label{mainthm:XOR-lives-in-Hessian}
	Consider the XOR GMM mixture with loss function~\pef{eq:XOR-loss} and the corresponding SGD \pef{eq:SGD-def} with $\beta \in (0,1/8)$, learning rate $\delta = O(1/d)$, initialized from $\cN(0,I_d/d)$. There exist $\alpha_0,\lambda_0$ such that if $\lambda \ge \lambda_0$, and $\widetilde M\ge \alpha_0 d$, the following hold.  For every $\varepsilon>0$, there exists $T_0(\varepsilon)$ such that for any fixed time horizon $T_0 <T_f<M/d$, with probability $1-o_d(1)$, for all $i\in \{1,...,K\}$, 
	\begin{enumerate}
		\item $W_i(\mathbf{x}_\ell)$ lives in $E_2(\nabla^2_{W_i W_i} \widehat R(\mathbf{x}_\ell))$ and in $E_2(\widehat G_{W_i W_i}(\mathbf{x}_\ell))$, and
		\item $v(\mathbf{x}_\ell)$ lives in $E_4(\nabla^2_{vv}\widehat R(\mathbf{x}_\ell))$ and $E_4(\widehat G_{vv}(\mathbf{x}_\ell))$,
	\end{enumerate}  
  up to $O(\varepsilon + \lambda^{-1/2})$\footnote{In this theorem and in Theorem~\ref{mainthm:XOR-all-live-in-subspace}, the big-$O$ notation also hides constant dependencies on the initial magnitudes of $v(\mathbf{x}_0)$, and on $T_f$.} error, for all $\ell \in [T_0 \delta^{-1},T_f\delta^{-1}]$. 
\end{thm}

\begin{rem}
   The reader may notice that in Theorems~\ref{mainthm:SGD-aligns-with-Hessian}--\ref{mainthm:XOR-lives-in-Hessian}, the criticality vs.\ sub-criticality ($\delta = \Theta(1/d)$ vs.\ $\delta = o(1/d)$) of the step-size does not affect the main alignment results. This is because the correction to the limiting SGD trajectory due to criticality of the step-size is of order $O(1/\lambda)$ and is getting absorbed into the other error terms of the theorems. It would be interesting to make these errors more precise to probe the influence of the criticality of the SGD step-size on the alignment phenomenon.
\end{rem}

\begin{rem}\label{rem:beta>1/8}
    The restriction to $\beta<1/8$ in Theorems~\ref{mainthm:XOR-lives-in-Hessian}--\ref{mainthm:XOR-all-live-in-subspace} is because when $\beta>1/8$ the regularization is too strong for the SGD to be meaningful; in particular, the SGD converges ballistically to the origin in parameter space, with no discernible preference for the directions corresponding to $\mu,\nu$ as the other directions. The above theorems are still valid there if the notion of error for living in a space from Definition~\ref{def:lives-in-span} were additive instead of multiplicative (i.e., $\|v - P_B v\| \le \varepsilon$).  
\end{rem}

This theorem is demonstrated in Figure~\ref{fig:XOR-topspaces}. There we have plotted the alignment of the rows in the intermediate layer with the space spanned by the top two eigenvectors of the corresponding first-layer blocks of the Hessian and G-matrices, and similarly for the final layer.

\begin{figure}[t]
    \centering
    \subfigure[]{\includegraphics[width=.235\textwidth]{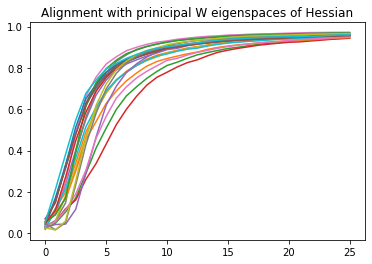}}
           \subfigure[]{ \includegraphics[width=.235\textwidth]{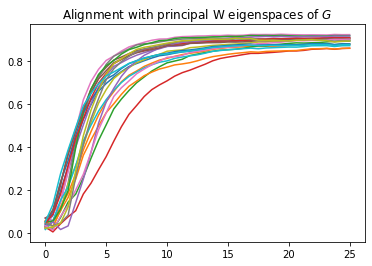}}
\subfigure[]{\includegraphics[width=.235\textwidth]{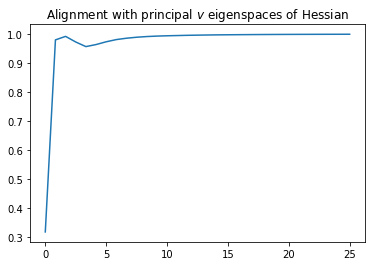}}
   \subfigure[]{\includegraphics[width=.235\textwidth]{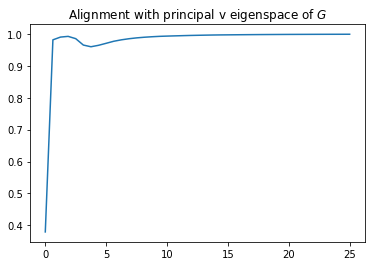}}
      \caption{(a) and (b) depict the alignment of the first layer weights $W_i(\mathbf{x}_\ell)$ for $i=1,...,K$ (in different colors) with the principal subspaces of the corresponding blocks of the Hessian and G-matrices,  i.e., with $E_2(\nabla^2_{W_i W_i} \widehat R(\mathbf{x}_\ell))$ and $E_2(\widehat G_{W_i W_i}(\mathbf{x}_\ell))$. (c) and (d) plot the second-layer alignment, namely of $v(\mathbf{x}_\ell)$ with $E_4(\nabla^2_{vv} \widehat R(\mathbf{x}_\ell))$ and $E_4(\widehat G_{vv}(\mathbf{x}_\ell))$. Parameters are $d=1000$, $\lambda=10$, and $K=20$}
    \label{fig:XOR-topspaces} 
\end{figure}

As before, the above theorem follows from the following theorem that describes both the SGD trajectory, its Hessian, and its G-matrix, living up to $O(\varepsilon+ \lambda^{-1/2})$ error in their first-layer blocks in $\text{Span}(\mu,\nu)$ and in their second layer blocks in 
$$\text{Span}(g(W(\mathbf{x}_\ell) \mu), g(-W(\mathbf{x}_\ell) \mu),g(W(\mathbf{x}_\ell) \nu), g(-W(\mathbf{x}_\ell) \nu))\,,$$
where $g$ is applied entrywise. 

\begin{thm}\label{mainthm:XOR-all-live-in-subspace}
In the setting of Theorem~\ref{mainthm:XOR-lives-in-Hessian}, up to $O(\varepsilon+ \lambda^{-1/2})$ error with probability $1-o_d(1)$, the following live in $\text{Span}(\mu,\nu)$, 
\begin{itemize}
	\item The first layer weights, $W_i(\mathbf{x}_\ell)$ for each $i\in \{1,...,K\}$,
	\item The first-layer empirical test Hessian $\nabla_{W_i W_i}^2 \widehat R(\mathbf{x}_\ell)$ for each $i\in \{1,...,K\}$,
	\item The first-layer empirical test G-matrix $\widehat G_{W_i W_i}(\mathbf{x}_\ell)$ for each $i\in \{1,...,K\}$,
\end{itemize}
and the following live in $\text{Span}(g(W(\mathbf{x}_\ell) \mu), g(-W(\mathbf{x}_\ell) \mu),g(W(\mathbf{x}_\ell) \nu), g(-W(\mathbf{x}_\ell) \nu))$
\begin{itemize}
	\item The second layer weights $v(\mathbf{x}_\ell)$,
	\item The second-layer empirical test Hessian $\nabla^2_{vv} \widehat R(\mathbf{x}_\ell)$,
	\item The second-layer empirical test G-matrix $\widehat G_{vv}(\mathbf{x}_\ell)$.
\end{itemize}
\end{thm}

\begin{figure}[t]
    \centering
 \includegraphics[width=.35\textwidth]{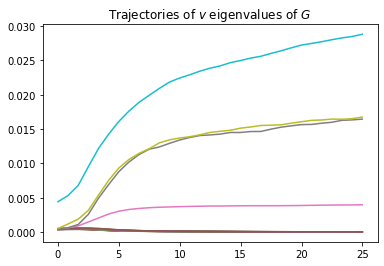}
   \qquad
 \includegraphics[width=.35\textwidth]{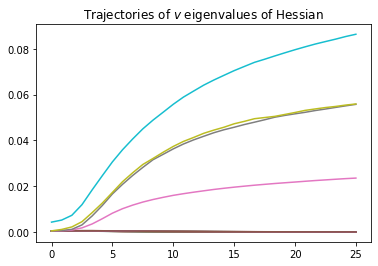}
      \caption{ 
      The eigenvalues (in different colors) of the $vv$ blocks of the Hessian and G-matrices over time from a random initialization. Initially, there is one outlier eigenvalue due to the positivity of the ReLU activation. Along training, four outlier eigenvalues separate from the bulk, corresponding to the four ``hidden" classes in the XOR problem. Parameters are the same as in Figure~\ref{fig:XOR-topspaces}.}
    \label{fig:DBBP-XOR}
\end{figure}

Let us discuss a bit more the phenomenology of the alignment in the second layer. First of all, we observe that the subspace in which the alignment occurs is a random---depending on the initialization and trajectory (in particular, choice of fixed point the SGD converges to)---4-dimensional subspace of $\mathbb R^K$. 
Furthermore, if we imagine $K$ to be much larger than $4$ so that the model is overparametrized, at initialization, unlike the $1$-layer case studied in Section~\ref{subsec:1-layer-k-GMM}, the Hessian and G-matrices do not exhibit any alignment in their second layer blocks. In particular, the second layer blocks of the Hessian and G-matrices look like (non-spiked) Gaussian orthogonal ensemble and Wishart matrices\footnote{These are two of the most classical random matrix models, see~\cite{AGZ} for more.} in $K$ dimensions at initialization, and it is only over the course of training that they develop $4$ outlier eigenvalues as the first layer of the SGD begins to align with the mean vectors and the vectors $(g(W(\mathbf{x}_\ell)  \vartheta))_{\vartheta \in \{\pm \mu,\pm \nu\}}$ in turn get large enough to generate outliers. This crystallization of the last layer around these vectors over the course of training is reminiscent of the neural collapse phenomenon described in~\citet{PapyanDonoho} (see also~\citet{NeuralCollapse-MSE,ZhuNeuralCollapse}). The simultaneous emergence, along training, of outliers in the Hessian and G-matrix spectra can be seen as a dynamical version of what is sometimes referred to as the BBP transition after~\citet{baik2005phase} (see also~\citet{Peche06}). This dynamical transition is demonstrated in Figure~\ref{fig:DBBP-XOR}

Finally, we recall that~\citet{BGJ22} found a positive probability (uniformly in $d$, but shrinking as the architecture is overparametrized by letting $K$ grow) that the SGD converges to sub-optimal classifiers from a random initialization. When this happens, $g(W(\mathbf{x}_\ell)\vartheta)$ remains small for the hidden classes $\vartheta\in \{\pm \mu,\pm \nu\}$ that are not classifiable with SGD output. In those situations, Theorem~\ref{mainthm:XOR-all-live-in-subspace} shows that the outlier subspace in the $vv$-blocks that emerges will have rank smaller than $4$, whereas when an optimal classifier is found it will be of rank $4$: see Figure~\ref{fig:DBBP-XOR-deficient}. Knowing that the classification task entails a mixture of $4$ means, this may provide a method for devising a stopping rule for classification tasks of this form by examining the rank of the outlier eigenspaces of the last layer Hessian or G-matrix. While the probability of such sub-optimal classification is bounded away from zero, the probability  goes to zero exponentially as the model is overparametrized via $K\to\infty$, and that gives more chances to allow the second layer SGD, Hessian, and G-matrices to exhibit full $4$-dimensional principal spaces. This serves as a concrete and provable manifestation of the \emph{lottery ticket hypothesis} of~\citet{LotteryTicket}. 

\begin{figure}[t]
    \centering
\subfigure[]{        \includegraphics[width=.235\textwidth]{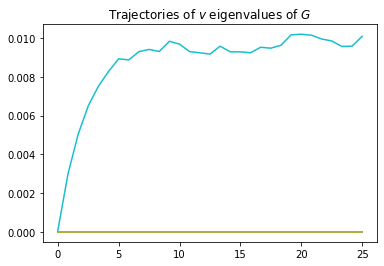}}
   \subfigure[]{     \includegraphics[width=.235\textwidth]{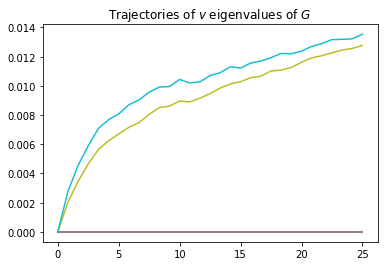}}
\subfigure[]{        \includegraphics[width=.235\textwidth]{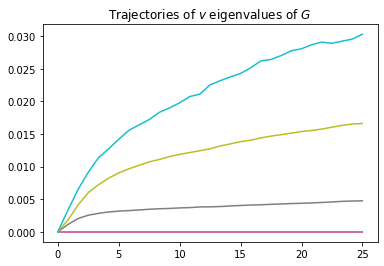}}
\subfigure[]{        \includegraphics[width=.235\textwidth]{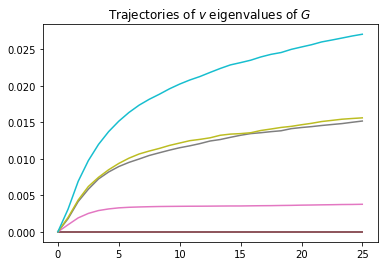}}
      \caption{Evolution of eigenvalues in the $v$ component of G over time in rank deficient cases. Here SGD is started from initializations that converge to suboptimal classifiers (this has uniformly positive, $K$-dependent, probability under a random initialization). From left to right, the SGD's classifier varies in the number of hidden classes it discerns, from $1$ to $4$. There is still a dynamical spectral transition, now with only a corresponding number of emerging outlier eigenvalues.}
    \label{fig:DBBP-XOR-deficient}
\end{figure}

\subsection{Outline and ideas of proof}

The proofs of our main theorems break into three key steps. 
\begin{enumerate}
    \item In Sections~\ref{s:1-layer-pop}--\ref{s:multilayer-pop}, we show that the \emph{population} Hessian and G-matrices, have bulks (and possibly minibulks) that are $O(1/\lambda)$ in operator norm, and finite $C(k)$-rank parts. We can explicitly characterize the low-rank part's eigenvalues and eigenvectors up to $O(1/\lambda)$ corrections, as functions of the parameter space, to see exactly where their emergence as outliers occurs depending on the model. 
    \item In Section~\ref{sec:SGD-analysis}, we analyze the SGD trajectories for the $k$-GMM and XOR classification problems. We do this using the limiting effective dynamics  theorem proven in~\citet{BGJ22} for finite families of summary statistics. We derive ODE limits for these summary statistics, notably for the general $k$-GMM which was not covered in that paper: see Theorem~\ref{thm:k-GMM-ballistic-limit}. We then pull back the limiting dynamics to finite $d$ and expand its $\lambda$-finite trajectory about its $\lambda = \infty$ solution. These latter steps involve understanding some of the stability properties of the dynamical system limits of the SGD's summary statistics. 
    \item In Section~\ref{sec:Hessian-concentration}, we prove concentration for the empirical Hessian and G-matrices about their population versions, in operator norm, throughout the parameter space. In some related settings, including binary mixtures of Gaussians,~\citet{MBM18} established concentration of the empirical Hessian about the population Hessian uniformly in the parameter space assuming polylogarithmic sample complexity. Our proofs are based on $\epsilon$-nets and concentration inequalities for uniformly sub-exponential random variables, albeit with some twists due, for instance, to non-differentiability of the ReLU function. 
\end{enumerate}
In Section~\ref{s:main-theorem-proofs}, we combine these steps to establish alignment of the $\lambda,\alpha$-finite matrices' outlier eigenspaces, and the SGD trajectory, with the common ``ground truth" subspace spanned by outlier eigenvectors of the $\lambda = \infty$ population matrices. In particular, in the examples we study, this is also the span of the class means or their images under the (time-dependent) first layer transformation.

\begin{remark}\label{rem:test-vs-train}
Our results are stated for the empirical Hessian and G-matrices generated using test data, along the SGD trajectory generated from training data. Since we are considering online SGD, the empirical Hessian and G-matrices with training data are no more relevant than those generated from test data, but the reader may still wonder whether the same behavior holds. A straightforward modification of our arguments in Section~\ref{sec:Hessian-concentration} is given in Section~\ref{sec:training-data} to extend our results for the $k$-GMM model to Hessian and G-matrices generated from train data, if we assume that $M\gtrsim d\log d$ rather than simply $M\gtrsim d$. It is an interesting mathematical question to drop this extra logarithmic factor. The extension in the XOR case is technically more involved due to the lack of regularity of the ReLU function. Also see Section~\ref{sec:additional figures} for numerical demonstrations that the phenomena in the empirical matrices generated from train data are identical to those generated from test data. 
\end{remark}

\begin{remark}A natural question is whether our results apply to more general mixture distributions than Gaussian ones. Unfortunately, the isotropy of the Gaussian is used crucially in the study of SGD trajectories in their high-dimensional limits. Establishing a form of dynamical universality for the summary statistic trajectories would be of interest. 
\end{remark}

\subsection{Global notation}
Throughout the paper, we are imagining the dimension $d$ to be sufficiently large, the sample complexity and inverse step size scaling with $d$, and our results hold for all large~$d$. Towards that, when we use $f\lesssim g$ or $f = O(g)$, we mean $f\le Cg$ for a constant $C$ depending on fixed parameters, e.g., the number of classes $k$, the width of the second layer in the XOR case $K$, and the regularizer $\beta$. When there are other parameters on which we want to emphasize the dependence, we include that as a subscript, e.g., as $f\lesssim_r g$ if $r$ is the radius of a ball in parameter space to which we are confining ourselves. Also, throughout the paper, for a vector $v$, we use $\|v\|$ to denote its $\ell^2$ norm, and for a matrix $A$, use $\|A\|$ to denote its ($\ell^2\to\ell^2$) operator norm.

\section{Analysis of the population matrices: 1-layer networks}\label{s:1-layer-pop}

In this section, we study the Hessian and G-matrix of the \emph{population} loss for the $k$-GMM problem whose data distribution and loss were given in~\pef{eq:data-distribution}--\pef{eq:cross-entropy-loss}. Specifically, we give the Hessian and G-matrices' $\lambda$-large expansion, and showing they have low rank structures with top eigenspaces generated by $O(1/\lambda)$ perturbations of the mean vectors $(\mu_1,...,\mu_k)$. In Section~\ref{sec:Hessian-concentration}, we will show that the empirical matrices are well concentrated about the population matrices.

\subsection{Preliminary calculations and notation}\label{s:pre}

It helps to first fix some preliminary notation, and give expressions for derivatives of the loss function of~\pef{eq:cross-entropy-loss}. Differentiating that, we get  
\begin{align}\label{eq:DL}
\nabla_{x_{c}}L=(-y_{c}+\frac{\exp(x^{c}\cdot Y)}{\sum_{a}\exp(x^{a}\cdot Y)})Y\,,
\end{align}
the Hessian matrix
\begin{align}\label{eq:D2L}
\nabla_{x_{b}}\nabla_{x_{c}}L=\Big(\frac{\exp(x^{b}\cdot Y)}{\sum_a\exp(x^{a}\cdot Y)}\delta_{bc}-\frac{\exp(x^{b}\cdot Y)\exp(x^{c}\cdot Y)}{(\sum_{a}\exp(x^{a}\cdot Y))^{2}}\Big)Y\tensor Y\,,
\end{align}
and the G-matrix
\begin{align}\label{eq:DDL}
    \nabla_{x^b}L\tensor \nabla_{x^c}L=
    \left(y_c y_b-\pi_Y(b)y_c-y_b\pi_Y(c)+\pi_Y(c)\pi_Y(b)\right)Y\otimes Y\,.
\end{align}

Given the above, the following probability distribution over the $[k]$ classes naturally arises, for which we reserve the notation $\pi_{Y}(\cdot)\in\cM_{1}([k])$ (which is a probability measure on $[k]$): 
\begin{align}\label{eq:pi-dist}
\pi_Y(c) = \pi_{Y}(c;\mathbf{x}):=\frac{\exp(x^{c}\cdot Y)}{\sum_{a\in [k]}\exp(x^{a}\cdot Y)}.
\end{align}
Note the dependency of $\pi_{Y}$ on the point in parameter space $\mathbf{x}$. Since in this section $\mathbf x$ can be viewed as fixed, we will suppress this dependence from the notation. In the rest of this section, we will denote $Z\sim \cN(0, I_d/\la)$. We also denote
\begin{align}\label{eq:pi-expectation-covariance}
    \langle x^B\rangle=\langle x^B\rangle_{\pi_Y}=\sum_a x^a \pi_Y(a)\,,
    \quad
    \langle x^B;x^B\rangle=\langle x^B;x^B\rangle_{\pi_Y}=\sum_a (x^a\otimes x^a) \pi_Y(a)-\langle x\rangle_{\pi_Y}^{\otimes 2}\,,
\end{align}
where $B\sim \pi_Y$.

We consider the population Hessian $\nabla^2 \Phi=\nabla^2\bE[L]$, block by block, using $\nabla_{bc} \Phi$ to denote the $d\times d$ block in $\nabla_x^2\Phi$ corresponding to $\nabla_{x^b}\nabla_{x^c}\Phi$. Then, thanks to \pef{eq:D2L}, the $bc$ block of the population Hessian is of the form
\[
\nabla_{bc}\Phi=\E\left[(\pi_{Y}(b)\delta_{bc}-\pi_{Y}(b)\pi_{Y}(c))Y\tensor Y\right]\,.
\]
We also consider the population G-matrix $\Gamma=\bE[\nabla L^{\otimes 2}]$, block by block, using $\Gamma_{bc}$ to denote the $d\times d$ block corresponding to $\bE[\nabla_{x^b} L\otimes \nabla_{x^c}L]$. Then, thanks to \pef{eq:DDL}, the $bc$ block of the population Hessian is of the form
\[
\Gamma_{bc}=\bE[\left(y_c y_b-\pi_Y(b)y_c-y_b\pi_Y(c)+\pi_Y(c)\pi_Y(b)\right)Y\otimes Y]\,.
\]

For both population Hessian matrix and G-matrix, We will study the off-diagonal blocks $a\ne c$ and the diagonal ones $a=c$ separately. 

\subsection{Analysis of the population Hessian matrix}
We now compute exact expressions for the blocks of the population Hessian as $\lambda$ gets large: see Lemmas~\ref{l:offblock1}--\ref{l:dblock1}. 

We begin by studying the off-diagonal blocks: $b\neq c$, for which,  
\begin{align}\begin{split}\label{e:off-diagonalPhi}
\nabla_{bc}\Phi & =\E\left[(\pi_Y(b)\delta_{cb}-\pi_Y(b)\pi_{Y}(c))Y^{\otimes 2}\right] =-\sum_{l}p_{l}\E\left[\pi_{Y_l}(b)\pi_{Y_l}(c)Y_l^{\otimes 2}\right]\,,
\end{split}\end{align}
where $Y_l  = \mu_l + Z$, i.e., it is distributed like $Y$ given class choice $l$. 
It helps here to recall the well-known Gaussian integration by parts formula: For $f$ that is differentiable with derivative of at most exponential growth at infinity, 
\[
\E[f(Z_\lambda) Z_\lambda] = \frac{1}{\lambda}\E[\nabla f(Z_\lambda)]\,.
\]

Our goal is to show the following.

\begin{lem}\label{l:offblock1}
   The off-diagonal blocks of the population Hessian satisfy
\begin{align*}
    \nabla_{bc} \Phi(\mathbf{x}) = \E \Big\{\pi_Y(c)\pi_Y(a)\Big[ \big(\mu_{y}+\frac{1}{\lambda}[(x^{c}+x^{a})-2\left\langle x^{B}\right\rangle]\big)^{\tensor2}+\frac{1}{\lambda}I_d-\frac{2}{\lambda^{2}}\left\langle x^{B};x^{B}\right\rangle\Big]\Big\}\,.
\end{align*}
    In particular they are shifts by the identity of a matrix of rank at most $k^2$.
\end{lem}
\begin{proof}
We decompose each term in \pef{e:off-diagonalPhi} as three terms
\begin{align*}
\E\left[\pi_{Y_l}(b)\pi_{Y_l}(c)\right]\mu_{l}^{\otimes 2}+2\Sym(\E\left[\pi_{Y_l}(b)\pi_{Y_l}(c)Z\right]\tensor\mu_{l})+\E\left[\pi_{Y_l}(b)\pi_{Y_l}(c)Z^{\otimes 2}\right]=:
({\rm i})+({\rm ii})+({\rm iii})\,,
\end{align*}
where we've used here multi-linearity of $\tensor$. Here $\Sym(c\tensor b)=(c\tensor b+b\tensor c)/2$. Let's look at this term-by-term. We leave term (i) as is. For Term (ii), we notice by Gaussian integration by parts, 

\begin{align*}
 \E\left[\pi_{Y_l}(c)\pi_{Y_l}(b)Z\right]\tensor\mu_{l}
  & =\frac{1}{\lambda}\E\left[\nabla_{Z}(\pi_{Y_l}(c)\pi_{Y_l}(b))\right]\tensor\mu_{l} \\
  &=\frac{1}{\lambda}\E\left[((x^{b}+x^{c})-2\left\langle x^{B}\right\rangle_{\pi_{Y_l}}) \pi_{Y_l}(b)\pi_{Y_l}(c)\right]\tensor\mu_{l} \\
  & =\frac{1}{\lambda}(\E[\pi_{Y_l}(b)\pi_{Y_l}(c)](x^{b}+x^{c})\tensor\mu_{l}-2\E_y\left[\left\langle x^{B}\right\rangle_{\pi_{Y_l}} \pi_{Y_l}(b)\pi_{Y_l}(c)\right]\tensor\mu_{l})\,.
\end{align*}
where we recall~\pef{eq:pi-expectation-covariance} and used the derivative calculation 
\begin{align}
\nabla_Z \pi_{Y_l}(a)=\nabla_Z\frac{\exp[x^a\cdot (\mu _l+Z)]}{\sum_b \exp[x^b\cdot (\mu _l+Z)]}=(x^a-\langle x^B\rangle_{\pi_{Y_l}})\pi_{Y_l}(a)\,. \label{eq:grad-pi}
\end{align}

This tells us that 
\[
({\rm ii})=\frac{2}{\lambda}\Sym\Big(\E\left[\left((x^{b}+x^{c})-2\left\langle x^{B}\right\rangle_{\pi_{Y_l}} \right)\pi_{Y_l}(b)\pi_{Y_l}(c)\right]\tensor\mu_{l}\Big),
\]
notice that this is of rank at most $2$ and of order $O(1/\lambda)$. 

Finally for term (iii), we integrate by parts again, to get, for every $i,j$, 
\begin{align}\begin{split}\label{e:intbypart}
\E\left[\pi_{Y_l}(b)\pi_{Y_l}(c)Z_{i}Z_{j}\right] & =\frac{1}{\lambda}\E[\partial_{Z_i}(\pi_{Y_l}(b)\pi_{Y_l}(c))Z_{j}]+\frac{1}{\lambda}\E[\pi_{Y_l}(b)\pi_{Y_l}(c)\delta_{ij}]\\
 & =\frac{1}{\lambda^{2}}\E[\partial_{Z_i Z_j}(\pi_{Y_l}(b)\pi_{Y_l}(c))]+\frac{1}{\lambda}\E[\pi_{Y_l}(b)\pi_{Y_l}(c)\delta_{ij}].
\end{split}\end{align}
so term (iii) is given by
\[
\frac{1}{\lambda} \E[\pi_{Y_l}(b)\pi_{Y_l}(c) I_d]+\frac{1}{\lambda^{2}}\E[\nabla_{Z}^{2}\left(\pi_{Y_l}(b)\pi_{Y_l}(c)\right)]\,.
\]
Examining this second term, 
\begin{align*}
\E[ & \partial_{Z_i Z_j} (\pi_{Y_l}(b)\pi_{Y_l}(c)) ]
   =\E\Big[\partial_{Z_j}\Big\{ \big(x_{i}^{b}+x_{i}^{c}-2\langle x_{i}^{B}\rangle_{\pi_{Y_l}} \big)\pi_{Y_l}(b)\pi_{Y_l}(c)\Big\} \Big]\\
 & =\E\Big[\big(x_{i}^{b}+x_{i}^{c}-2\langle x_{i}^{B}\rangle_{\pi_{Y_l}} \big)\big((x_{j}^{b}+x_{j}^{c})-2\langle x_{j}^{B}\rangle_{\pi_{Y_l}} \big){\pi_{Y_l}}(b){\pi_{Y_l}}(c)\Big]-2\E\big[\langle x_{i}^{B};x_{j}^{B}\rangle_{\pi_{Y_l}} {\pi_{Y_l}}(b){\pi_{Y_l}}(c)\big]\,,
\end{align*}
where we first used~\pef{eq:grad-pi}, then used that 
\begin{align}\label{eq:grad-pi-x^B}
\nabla_Z\langle x^B\rangle_{\pi_{Y_l}} &= \nabla_Z\frac{\sum_bx^b\exp[x^b\cdot(\mu_l+Z)]}{\sum_b \exp[x^b\cdot (\mu_l+Z)]} = \langle x^B\tensor x^B\rangle_{\pi_{Y_l}} -\langle x^B\rangle_{\pi_{Y_l}}^{\tensor 2} = \langle x^B;x^B\rangle_{\pi_{Y_l}}\,.
\end{align}

As such, 
\begin{align*}
\frac{1}{\la^2}\E[\nabla_{Z}^{2}({\pi_{Y_l}}(b){\pi_{Y_l}}(c)) ]& =\frac{1}{\lambda^{2}}\Big\{ \E\Big[(x^{b}+x^{c}-2\langle x^{B}\rangle_{\pi_{Y_l}} )^{\otimes 2}{\pi_{Y_l}}(b){\pi_{Y_l}}(c)\Big]\\
 & \qquad\quad -2\E\Big[\langle x^{B};x^{B}\rangle_{\pi_{Y_l}} {\pi_{Y_l}}(b){\pi_{Y_l}}(c)\Big]\Big\}\,.
\end{align*}
This term can be seen to be of rank at most $k^2$ (each $\langle x^B\rangle$ is a weighted sum of $(x^a)_{a}$).

Combining all three terms above, we get that the off-diagonal block of the population Hessian matrix is given by the sum over $y\in [k]$ of 
\begin{align}\begin{split}\label{e:pipiYY}
 & \E\left[{\pi_{Y_l}}(b){\pi_{Y_l}}(c)\right]\mu_{l}^{\otimes 2}\! + \!\frac{2}{\lambda}\Sym\Big(\E\big[\big((x^{b}+x^{a})-2\langle x^{B}\rangle_{\pi_{Y_l}} \big)\pi_{\pi_{Y_l}}(b)\pi_{\pi_{Y_l}}(c)\big]\tensor\mu_{l}\Big)\!+\!\frac{1}{\lambda}\E[\pi_{\pi_{Y_l}}(b)\pi_{\pi_{Y_l}}(c)]I_d\\
 & +\frac{1}{\lambda^{2}}\left(\E\left[((x^{b}+x^{a})-2\left\langle x^{B}\right\rangle_{\pi_{Y_l}} )^{\otimes 2}\pi_{\pi_{Y_l}}(b)\pi_{\pi_{Y_l}}(c)\right]-2\E\left[\left\langle x^{B};x^{B}\right\rangle_{\pi_{Y_l}} \pi_{\pi_{Y_l}}(b)\pi_{\pi_{Y_l}}(c)\right]\right).
\end{split}\end{align}
Summing this in $y$ we get that an off-diagonal block is of
the form
\begin{align}\label{e:ABC}
-(A+\frac{1}{\lambda}B+\frac{1}{\lambda^{2}}C)\,,
\end{align}
where 
\begin{align*}
A&=\sum_{l}p_{l}\E\left[\pi_{\pi_{Y_l}}(c)\pi_{\pi_{Y_l}}(a)\right]\mu_{l}^{\otimes 2}\,, \\ 
B&=\sum_{l} p_y \E\left[\pi_{{Y_l}}(a)\pi_{Y_l}(c)\right]I_d+2\sum_{l}p_{l}\Sym\Big(\E_{l}\left[\left(x^{c}+x^{a}-2\left\langle x^{B}\right\rangle_{\pi_{Y_l}} \right)\pi_{Y_l}(c)\pi_{Y_l}(a)\right]\tensor\mu_{l}\Big)\,, \\
C&=\sum_l p_l \E\left[(x^{c}+x^{a}-2\left\langle x^{B}\right\rangle_{\pi_{Y_l}} )^{\otimes 2}\pi_{\pi_{Y_l}}(c)\pi_{\pi_{Y_l}}(a)\right]-2\E\left[\left\langle x^{B};x^{B}\right\rangle_{\pi_{Y_l}}\pi_{Y_l}(c)\pi_{Y_l}(a)\right].
\end{align*}
In summary to leading order it is $A$ which is rank $k$. To next
order it is $O(1/\lambda)$ and that term is a full rank (identity)
plus a rank at most $2k$ term. To next order it is $O(1/\lambda^{2})$
and this is a covariance-type quantity with respect to the Gibbs probability $\pi_{Y_y}$, with rank at most $k^2$. 

We can group the expression \pef{e:ABC} further 
as 
\begin{align*}
\sum_l p_l \E \Big\{\pi_{Y_l}(c)\pi_{Y_l}(a)\Big[ \big(\mu_{l}+\frac{1}{\lambda}[x^{c}+x^{a}-2\left\langle x^{B}\right\rangle_{\pi_{Y_l}}]\big)^{\tensor2}+\frac{1}{\lambda}I_d-\frac{2}{\lambda^{2}}\left\langle x^{B};x^{B}\right\rangle_{\pi_{Y_l}}\Big]\Big\}\,.
\end{align*}
This is of the form rank $k$ plus rank $k^{2}$ shifted by the identity
(adding one more eigenvalue). Incorporating the average over $l$ into the expectation, this is exactly the claimed expression.  
\end{proof}

\subsubsection{On-diagonal blocks}\label{s:on-diagonal}
In this section, we study the $aa$ diagonal blocks 
\begin{align}\label{e:diagonalPhi}
\nabla_{aa}\Phi & =\E\left[(\pi_{Y}(a)(1-\pi_{Y}(a))Y^{\otimes 2}\right] = \sum_{l} p_l \E\left[(\pi_{Y_l}(a)(1-\pi_{Y_l}(a))Y_l^{\otimes 2}\right] \,.
\end{align}
We prove the following large $\lambda$ expansion.

\begin{lem}\label{l:dblock1}
    The diagonal $aa$-block of the population Hessian $\nabla_{aa}\Phi$ equals
\[
\begin{aligned}
    \E\Big[\pi_{Y}(a)\Big(\mu & +\frac{1}{\lambda}(x^a-\langle x^B\rangle_{\pi_{Y}})^{\tensor 2}-\frac{1}{\lambda^2}\langle x^B;x^B\rangle_{\pi_{Y}}\Big)\Big]\\
    &-\E \Big[ \pi(a)^2\Big((\mu+\frac{2}{\lambda}(x^a-\langle x^B\rangle_{\pi_{Y}})^{\tensor 2}-\frac{2}{\lambda^2}\langle x^B;x^B\rangle_{\pi_{Y}}\Big)\Big] + \frac{1}{\lambda}\E \big[\pi_{Y}(a)(1-\pi_{Y}(a))\big] I_d\,.
\end{aligned}
\]
In particular, it is a shift by the identity of a rank at-most $k^2$ matrix.
\end{lem}
\begin{proof}
By \pef{e:diagonalPhi}, the diagonal $aa$-block is given by an average over $l\in [k]$ of 
\[
\E[\pi_{Y_l}(a)(1-\pi_{Y_l}(a))Y^{\otimes 2}]=\E[\pi_{Y_l}(a)(\mu_l + Z)^{\otimes 2}]-\E [\pi_{Y_l}(a)^2(\mu_l+ Z)^{\otimes 2}]\,.
\]
Note that the second term was exactly what was computed Lemma~\ref{l:offblock1}, setting $b=c$. It remains to compute the first.
To this end, we proceed as before and for each fixed $l$, write the above as 
\[
\E[\pi_{Y_l}(a)]\mu_l^{\tensor 2}+2\Sym(\E[\pi_{Y_l}(a)Z]\tensor \mu_l)+\E[\pi_{Y_l}(a)Z^{\otimes 2}]\,.
\]
We will integrate the second and third terms by-parts. By the gradient calculations of~\pef{eq:grad-pi},
\begin{align}\begin{split}\label{e:pzterm}
\E[\pi_{Y_l}(a)Z]=\frac{1}{\lambda}\E[\nabla_Z \pi_{Y_l}(a)]=\frac{1}{\lambda}\E[(x^a-\langle x^B\rangle_{\pi_{Y_l}})\pi_{Y_l}(a)]\,,
\end{split}\end{align}
and by the calculation of~\pef{eq:grad-pi-x^B}, 
\begin{align}\label{e:pzzterm}
\E[\pi_{Y_l}(a) Z^{\otimes 2}] & = \frac{1}{\lambda}\E [\pi_{Y_l}(a)] I_d+\frac{1}{\lambda^2}\E[\nabla_Z^2 \pi_{Y_l}(a)
] \nonumber\\ 
& 
= \frac{1}{\lambda}\E[\pi_{Y_l}(a)] I_d+\E [\pi_{Y_l}(a) ((x^a-\langle x^B\rangle_{\pi_{Y_l}})^{\tensor 2} - \langle x^B;x^B\rangle_{\pi_{Y_l}})]\,.
\end{align}
Combining the above expressions yields 
\[
\E[\pi_{Y_l}(a) Y_l^{\otimes 2}] = \E[\pi_{Y_l}(a)((\mu_l+\frac{1}{\lambda}(x^a-\langle x^B\rangle_{\pi_{Y_l}})^{\tensor 2}+\frac{1}{\lambda} I_d -\frac{1}{\lambda^2}\langle x^B;x^B\rangle_{\pi_{Y_l}})].
\]
On the otherhand, Lemma~\ref{l:offblock1}, yields
\[
\E[\pi_{Y_l}(a)^2 Y_l^{\otimes 2}] = \E[\pi_{Y_l}(a)^2((\mu_l+\frac{2}{\lambda}(x^a-\langle x^B\rangle_{\pi_{Y_l}})^{\tensor 2}+\frac{1}{\lambda} I_d -\frac{2}{\lambda^2}\langle x^B;x^B\rangle_{\pi_{Y_l}})]\,.
\]
Combining these two and averaging over $l$ yields the desired.
\end{proof}

\subsection{Analysis of the population G-matrix}
We now compute an exact expansion of the population G-matrix as $\lambda$ gets large.

\begin{lem}\label{l:Gblock1}
For any $b,c$, the $bc$ block of the population G-matrix can be written in the form:
\[
\E [\nabla_{x^b} L\tensor \nabla_{x^c} L] = A+\frac{B}{\lambda}+\frac{C}{\lambda^2}\,,
\]
where $A$ is in the span of $(\mu_c)_c$ and $B,C$ have operator norm bounded by $1$. 
In particular,
\begin{align*}
        A &= \delta_{bc} p_b \mu_b^{\tensor 2}
        - p_c \E\{\pi_{Y_c}(b)[\mu_c +\frac{1}{\lambda}(x^b-\langle x^B\rangle_{\pi_{Y_c}})]^{\tensor 2}\}
        - p_b \E\{\pi_{Y_b}(c)[\mu_b+\frac{1}{\lambda}(x^c-\langle X^B\rangle_{\pi_{Y_b}})]^{\tensor 2}\} \nonumber\\
        &\qquad+\sum_l p_l \E\{ \pi_{Y_l}(b)\pi_{Y_l}(c) [\mu_l -\frac{1}{\lambda}(x^b+x^c-2\langle x^B\rangle_{\pi_{Y_l}})]^{\tensor 2}\}\,,\\
    B&= \big(\delta_{bc} p_b  - (p_b\E[\pi_{Y_b}(c)] + p_c\E[\pi_{Y_c}(b)])+\sum_l p_l\E[\pi_{Y_l}(b)\pi_{Y_l}(c)] \big) I_d\,,\\
     C&=   -p_c \E[\langle x^B;x^B\rangle_{\pi_{Y_c}} \pi_{Y_c}(b)]-p_b \E[\langle x^B;x^B\rangle_{\pi_{Y_b}} \pi_{Y_b}(c)]
    -2\sum_l p_l \E[\langle x^B;x^B\rangle_{\pi_{Y_l}} \pi_{Y_l}(b)\pi_{Y_l}(c)]\,.
\end{align*}
\end{lem}
\begin{proof}
    
We recall from \pef{eq:DDL}, there are four terms in the $bc$ block of the G-matrix:
\begin{align}\label{e:decomp}
   \bE[ \nabla_{x^b}L\tensor\nabla_{x^c}L]
   =:(i)-(ii)-(iii)+(iv)\,,
\end{align}
where 
\begin{align*}
    &(i):=\bE[y_c y_b Y\otimes Y],
    \quad
    (ii):=\bE[\pi_Y(b)y_c Y\otimes Y],
    \\
    &(iii):=\bE[y_b\pi_Y(c) Y\otimes Y],
    \quad
    (iv):=\bE[\pi_Y(c)\pi_Y(b)Y\otimes Y]\,.
\end{align*}
The first term in~\pef{e:decomp} is easy to compute
\begin{align}\label{e:I}
    (i)=\bE[y_b y_c Y\otimes Y]
    =\delta_{bc}p_b(\mu_b\otimes \mu_b+I_d/\lambda)\,.
\end{align}
The second and third terms are similar to each other, we will only compute the second term $(ii)$,
\begin{align}\begin{split}\label{e:II}
    (ii)&=\bE[y_c \pi_Y(b) Y\otimes Y]
    =p_c\bE\left[ \pi_{Y_c}(b) (\mu_c+Z)^{\otimes 2}\right]\\
    &=p_c\bE[\pi_{Y_c}(b)\mu_c^{\otimes 2}]
    +2p_c\Sym \bE[\pi_{Y_c}(b)Z\otimes \mu_c]
    +p_c\bE[\pi_{Y_c}(b)Z^{\otimes 2}]\,.
\end{split}\end{align}
Using the calculations of~\pef{e:pzterm} and~\pef{e:pzzterm}, we conclude that
\begin{align}\label{e:II2}\begin{split}
    (ii)
    &=p_c \bE[\pi_{Y_c}(b)(\mu_c^{\otimes 2}+I_d/\lambda)]+\frac{1}{\lambda}2\Sym(\bE[(x^b-\langle x^B\rangle_{\pi_
{Y_c}}
    ]) \pi_{Y_c}(b)]\otimes \mu_c)].\\
    & \quad +\frac{2}{\lambda^2}p_c\bE[((x^b-\langle x^B\rangle_{\pi_{Y_c}})^{\otimes 2}-\langle x^B; x^B\rangle_{\pi_{Y_c}})\pi_{Y_c}(b)]\,.
\end{split}\end{align}
For the last term $(iv)$ in \pef{e:decomp}, it has been computed in  \pef{e:pipiYY} that $(iv)$ is the expectation over $l$ of
\begin{align}\begin{split}\label{e:pipiYYcopy}    
&\bE\left[\pi_{Y_l}(b)\pi_{Y_l}(c)\right]\mu_{l}^{\otimes 2}\\
&+ \frac{1}{\lambda}\left[2\Sym(\bE\left[\left((x^{b}+x^{c})-2\langle x^B\rangle_{\pi_{Y_l}}\right)\pi_{Y_l}(b)\pi_{Y_l}(c)\right]\tensor\mu_{l})\right]+
  \frac{1}{\lambda}\bE[\pi_{Y_l}(b)\pi_{Y_l}(c)]I_d\\
  &+\frac{1}{\lambda^{2}}\left(\bE\left[((x^{b}+x^{c})-2\langle x^B\rangle_{\pi_{Y_l}})^{\tensor 2}\pi_{Y_l}(b)\pi_{Y_l}(c)\right]-2\bE\left[\langle x^B;x^B\rangle_{\pi_{Y_l}} \pi_{Y_l}(b)\pi_{Y_l}(c)\right]\right)\,.
\end{split}\end{align}

By plugging \pef{e:I}, \pef{e:II2}, its analogue for (iii), and \pef{e:pipiYYcopy} into \pef{e:decomp}, we conclude that the block  $\bE[G_{x^b x^c}]$ of the population G-matrix is given by
\[
    \bE[G_{x^b x^c}]=A'+\frac{B'}{\lambda}+\frac{C'}{\lambda^2}\,,
\]
where 
\begin{align*}
    A'=\delta_{bc}p_b \mu_b^{\otimes 2} -p_c\bE[\pi_{Y_c}(b)]\mu_c^{\otimes 2}-p_b\bE[\pi_{Y_b}(c)]\mu_b^{\otimes 2} +\sum_l p_l\bE[\pi_{Y_l}(b)\pi_{Y_l}(c)]\mu_l^{\otimes 2}\,,
\end{align*}
and 
\begin{align*}
    B'&=\left(\delta_{bc}p_b -(p_b\bE[\pi_{Y_b}(c)]+p_c\bE[\pi_{Y_c}(b)])+\sum_l p_l\bE[\pi_{Y_l}(b)\pi_{Y_l}(c)]\right)I_d\\
    & \quad -2p_c\Sym(\bE[(x^b-\langle x^B\rangle_{\pi_{Y_c}}) \pi_{Y_c}(b)]\otimes \mu_c)]-2 p_b\Sym(\bE[(x^c-\langle x^B\rangle_{\pi_{Y_b}}) \pi_{Y_b}(c)]\otimes \mu_b)]\\
  & \quad +\sum_l 2p_l\Sym(\bE\left[\left((x^{b}+x^{c})-2\langle x^B\rangle_{\pi_{Y_l}}\right)\pi_{Y_l}(b)\pi_{Y_l}(c)\right]\tensor\mu_{l})\,,
\end{align*}
and 
\begin{align*}
    C'&=-p_c\bE[((x^b-\langle x^B\rangle_{\pi_{Y_c}})^{\otimes 2}-\langle x^B; x^B\rangle_{\pi_{Y_c}})\pi_{Y_c}(b)]
    -p_b\bE[((x^c-\langle x^B\rangle_{\pi_{Y_b}})^{\otimes 2}-\langle x^B; x^B\rangle_{\pi_{Y_b}})\pi_{Y_b}(c)]\\
    & \quad +\sum_l p_l \left(\bE\left[((x^{b}+x^{c})-2\langle x^B\rangle_{\pi_{Y_l}} )^{\tensor 2}\pi_{Y_l}(b)\pi_{Y_l}(c)\right]-2\bE\left[\langle x^B;x^B\rangle_{\pi_{Y_l}} \pi_{Y_l}(b)\pi_{Y_l}(c)\right]\right)\,.
\end{align*}
Grouping tensor-squares we obtain the desired decompostion 
in terms of $A,B,$ and $C$.
\end{proof}

\section{Analysis of population matrices: the 2-layer case}\label{s:multilayer-pop}

In this section we analyze the population Hessian and G-matrices for the 2-layer XOR model, and especially its $\lambda$ large behavior by viewing it as a perturbation of its $\lambda = \infty$ value. Specifically, we compute its large $\lambda$ expansion, and uncover an underlying low-rank structure. We will show that the empirical Hessian and G-matrices concentrate about their population versions in Section~\ref{sec:Hessian-concentration}. 

\subsection{Preliminary calculations}
Recall the data model for the 2-layer XOR GMM from~\pef{eq:XOR-data-distribution} and the loss function~\pef{eq:XOR-loss}, with $\sigma$ being the sigmoid function and $g$ being ReLU.  The 2-layer architecture has intermediate layer width $K$, so that the first layer weights $W$ form a $K\times d$ matrix, and the second layer weights $v$ form a $K$-vector. 
Observe that 
\begin{align*}
\nabla_{v}L & =-(y-\widehat{y})g(WY)\,, \qquad \text{and} \qquad \nabla_{W_{i}}L =-(y-\widehat{y})g'(W_{i}\cdot Y)v_{i}Y\,,
\end{align*}
where 
\begin{align}\label{eq:yhat}
    \widehat{y}=\frac{e^{v\cdot g(WY)}}{1+e^{v\cdot g(WY)}} = \sigma(v\cdot g(WY))\,.
\end{align}

The $vv$-block of the G-matrix is the $K\times K$ matrix given by
\begin{align}\label{e:gv}
    \nabla_{v}L ^{\otimes 2}&  =(y-\widehat{y})^2 g(WY)^{\otimes 2}\,.
\end{align}
The $W_i W_j$ block of the G-matrix corresponding to $W$ is given by 
\begin{align}\label{e:gw}
\nabla_{W_{i}}L \otimes \nabla_{W_j} L & =(y-\widehat{y})^2 v_i v_j g'(W_{i}\cdot Y)g'(W_{j}\cdot Y) Y^{\otimes 2}\,.
\end{align}
The per-layer Hessian for the loss on a given sample is then given by  
\begin{align}\label{eq:XOR-Hessian}
\nabla_{vv}^{2}L & =g(WY)^{\tensor2}\widehat{y}(1-\widehat{y})\,, \nonumber \\
\nabla_{W_{i}W_{j}}^{2}L & =(-\delta_{ij}v_{i}g''(W_{i}\cdot Y)(y-\widehat{y})+v_{i} v_j g'(W_{i}\cdot Y)g'(W_{j}\cdot Y)\widehat{y}(1-\widehat{y}))Y^{\tensor2}\,.
\end{align}

Note that we may also write the diagonal block in $W$ of the form,
\[
\nabla_{W}^{2}L=-(y-\widehat{y})\text{diag}(v_{i}g''(W_{i}\cdot Y))\tensor Y^{\tensor2}+\widehat{y}(1-\widehat{y})(v\diag(g'(WY)))^{\tensor2}\tensor Y^{\tensor2}\,.
\]
We assume that  $W_i\not\equiv 0$ for all $i$, a set which we call $\mathcal W_0^c$. Then recall that for ReLu activation $g$, we have $g''=\delta_{0}$
in the sense of distributions. Thus as long as $W\in \mathcal W_0^c$, we can drop $g''$ in \pef{eq:XOR-Hessian} for any finite $\lambda<\infty$.

Our aim is to approximate the expected values of~\pef{eq:XOR-Hessian} by their "$\lambda = \infty$" versions, and show that they are within $O(\lambda^{-1/2})$ of one another. 
Let $\bar F(t)=P(G>t)$ where $G\sim N(0,1)$ be the tail probability
of a standard Gaussian. Let us also introduce the notations for $\vartheta \in \{\pm \mu,\pm \nu\}$,  
\begin{align*}
    m_i^\vartheta = W_i \cdot \vartheta\,,  \qquad \text{and}\qquad R_{ii} = W_i \cdot W_j \quad \text{for $1\le i,j\le K$}\,.
\end{align*}
These notations will reappear as summary statistics for the analysis of the SGD in the following sections. 
\begin{lem}
For $X\sim\cN(\vartheta,I_d/\lambda)$ and $W$ a $K\times d$
matrix, we have that
\begin{align}
\norm{\E [g(WX)^{\otimes 2}]-g(W\vartheta)^{\otimes 2}}_{\op} & \lesssim_{K}\psi_{2}(\vartheta,W,\lambda)\,,\label{eq:gW-opnorm-conc}
\end{align}
where 
\begin{align}\label{eq:psi-2}
\psi_{2}(\vartheta,W,\lambda)& =\max_{1\le i,j\le K}(\abs{m_i^\vartheta}^{3}+\lambda^{-3/2}R_{ii}^{3/2})^{1/3}\cdot(\abs{m_j^\vartheta}^{3}+\lambda^{-3/2}R_{jj}^{3/2})^{1/3} \nonumber\\
& \qquad \times \left(\bar F\left(\sqrt\frac{\lambda}{R_{ii}}|m_i^\vartheta|\right)+\bar F\left(\sqrt\frac{\lambda}{R_{jj}}|m_j^\vartheta|\right)\right)^{1/3}+\frac{W_i\cdot W_j}{\lambda}\,.
\end{align}
\end{lem}

\begin{proof}
By the equivalence of norms in finite dimensional vector spaces, it
suffices to control the norm entry-wise at the price of a constant
that depends at most on $K$. 

\textbf{Case 1: }$W_{i}\cdot \vartheta,W_{j}\cdot \vartheta\geq0$. In this case we
have that
\begin{align*}
\E [ & g(W_{i}\cdot X)g(W_{j}\cdot X)]-g(W_{i}\cdot \vartheta)g(W_{j}\cdot \vartheta) \\
 &= \mathbb E[(W_i \cdot X)(W_j \cdot X) - (W_i \cdot \vartheta)(W_j\cdot \vartheta)] - \mathbb E[(W_i \cdot X)(W_j\cdot X)\mathbf 1_{W_i \cdot X <0 \cup W_j\cdot X<0 }] \\ 
& =-\E[(W_{i}\cdot X)(W_{j}\cdot X)\mathbf 1_{W_{i}\cdot X<0\cup W_{j}\cdot X<0}]+\frac{W_{i}\cdot W_{j}}{\lambda}\,.
\end{align*}
The absolute value of the first term is bounded, by Holder's inequality and a
union bound, by 
\[
\mathbb E[|W_{i}\cdot X|^3]^{1/3} \mathbb E[|W_{j}\cdot X|^3]^{1/3} (\mathbb P(W_{i}\cdot X<0)+\mathbb P(W_{j}\cdot X<0))^{1/3}\,.
\]
 \textbf{Case 2:} $W_{i}\cdot \vartheta<0$ (the case $W_j \cdot \vartheta <0$ is symmetrical). Here we have 
\[
\E [g(W_{i}\cdot X)g(W_{j}\cdot X)] =\E[(W_{i}\cdot X)(W_{j}\cdot X)\indicator{W_{i}\cdot X,W_{j}\cdot X>0}]\,,
\]
which, by Holder's inequality, is bounded by
\[
\mathbb E[|W_{i}\cdot X|^3]^{1/3} \mathbb E[|W_{j}\cdot X|^3]^{1/3}\mathbb P(W_{i}\cdot X>0)^{1/3}\,.
\]
combining the two cases yields the desired.
\end{proof}

\begin{lem}
There exists a universal $c>0$ such that for $X\sim\cN(\vartheta,I_d/\lambda)$ and $W\in\R^{K\times d}$, 
\begin{equation}
\E[\norm{g(WX)^{\tensor2}}_{\op}^{2}]^{1/2}\leq\norm{W \vartheta}^{4}+6\frac{\norm{W \vartheta}^{2}}{\lambda}+\frac{c}{\lambda^{2}}\norm{W}_{F}^{2}\label{eq:gW-op-norm}\,,
\end{equation}
where $\|\cdot\|_F$ is the Frobenius norm. 
\end{lem}

\begin{proof}
Let $Z_{i}=W_{i}\cdot Z$. Then 
\begin{align*}
\E(W_{i}\cdot X)^{4} & =\E[(m^\vartheta_{i}+Z_{i})^{4}] =(m_{i}^\vartheta)^{4}+\frac{6(m_{i}^\vartheta)^{2}}{\lambda}+\frac{c\norm{W_i}}{\lambda^{2}}\,,
\end{align*}
so that 
\begin{align*}
\E[\norm{g(WX)^{\tensor2}}_{\op}^{2}] =\E[\norm{g(WX)}^{4}]
 &  \le \norm{W\vartheta}_{4}^{4}+6\frac{\norm{W\vartheta}^{2}}{\lambda}+\frac{c}{\lambda^{2}}\norm{W}_{F}^{2}\,,
\end{align*}
with the bound the following from the fact that $\norm{\cdot}_{p}\geq\norm{\cdot}_{q}$
for $p\geq q$. 
\end{proof}

\begin{lem}
For $X\sim\cN(\vartheta,I_d/\lambda)$ and $v\in\R^{K},$$W\in\R^{K\times d}$ we have
\begin{align}
\E[(\sigma(v\cdot g(WX))^{2}-\sigma(v\cdot g(W\vartheta))^{2})^{2}]\vee\E[(\sigma'(v\cdot g(WX))-\sigma'(v\cdot g(W\vartheta)))^{2}] & \lesssim_{K}\psi_{3}(\frac{1}{\lambda}\norm{v}^{2}\norm{W}_{F}^{2})\,,\label{eq:y-l2-bound}
\end{align}
where $\psi_{3}(x)=x(1+x)$. 
\end{lem}

\begin{proof}
Observe that 
\begin{align}\begin{split}\label{e:tb}
|\sigma(v\cdot & g(WX))^{2} -\sigma(v\cdot g(W\vartheta))^{2}| \\
& \le \sigma'(v\cdot g(W\vartheta)) |v\cdot (g(WX) - g(W\vartheta))|+ O(|v\cdot g(WX) - g(W\vartheta)|^2) \\ 
& \lesssim |v\cdot (g(WX) - g(W\vartheta))| + |v\cdot(g(WX) - g(W\vartheta))|^2\,,
\end{split}\end{align}
and similarly for $|\sigma'(v\cdot g(WX)) - \sigma'(v\cdot g(W\vartheta))|$
where we used that $\sigma,\sigma',\sigma''$ are all uniformly bounded. By Jensen's inequality, it will suffice to bound the expectation of the quadratic terms. 
We begin by noting that 
\[
\norm{g(WX)-g(W\vartheta)}^{2}\leq\sum_{1\le i \le K}(W_{i}\cdot Z)^{2}\,,
\]
so that 
\begin{align*}
\E[(v\cdot(g(WX)-g(W\vartheta)))^{4}] & 
 \leq\norm{v}^{4}\E[\norm{WZ}^{4}] 
  \leq\norm{v}^{4}\frac{1}{\lambda^{2}}\norm{W}_{F}^{4}\,.
\end{align*}
Consequently, taking the expectation of the right-hand side of \pef{e:tb} squared, gives
\begin{align*}
\mathbb E[(\sigma(v\cdot g(WX))^2 - \sigma(v\cdot g(W\vartheta))^2)^2] & \lesssim\E[(v\cdot(g(WX)-g(W\vartheta))^{2}]+\E[(v\cdot(g(WX)-g(W\vartheta)))^{4}] \\
&  \lesssim_{K}\frac{1}{\lambda}\sqrt{\norm{v}_{2}^{4}\norm{W}_{F}^{4}}+\frac{1}{\lambda^{2}}\norm{v}_{2}^{4}\norm{W}_{F}^{4}=\psi_{3}(\frac{1}{\lambda}\norm{v}_{2}^{2}\norm{W}_{F}^{2})\,,
 \end{align*}
and the analogue of it for $\E[(\sigma'(v\cdot g(WX))-\sigma'(v\cdot g(W\vartheta)))^{2}]$ follows similarly. 
\end{proof}

\subsection{Population Hessian estimates}
In this section, we study the population Hessian matrices. Our large $\lambda$ approximations will be uniform over compact sets in parameter space. Thus, we will use $B$ to denote a ball in parameter space, so that any constant dependencies on the choice of $B$ are just dependencies on its radius.

\begin{lem}\label{lem:population-Hessian-approximation}
For $(v, W)\in B$ the first layer block of the population Hessian matrix satisfy 
  \begin{align}\label{e:Hvterm}
\norm{\E[\nabla_{vv}^{2}L]-\frac{1}{4}\sum_{\vartheta\in\{\pm\mu,\pm \nu\}}\sigma'(v\cdot g(W\vartheta))g(W\vartheta)^{\tensor2}}_{\op} & \lesssim_{K,B}\psi_{2}(\vartheta,W,\lambda)+\frac{1}{\lambda}\,,
  \end{align}
  where $\psi_2$ was defined in~\pef{eq:psi-2}. For $(v,W)\in B$ with $W\in \mathcal W_0^c$, the diagonal second layer blocks of the population Hessian satisfy 
\begin{align}\label{e:HWterm}
    \norm{\E [\nabla^2_{W_iW_i}L]-&\frac{v_i^2}{4}\sum_{\vartheta\in\{\pm \mu,\pm \nu\}}\sigma(v g(m_i^\vartheta))(1-\sigma(v g(m_i^\vartheta)))F\left(\sqrt{\frac{\lambda}{R_{ii}}} m_i^\vartheta \right)\vartheta^{\tensor2}}_{\op}\nonumber\\
    &\lesssim\max_{\vartheta\in\{\pm \mu,\pm \nu\}}\left(\frac{1}{\sqrt{\lambda}}\cdot \bar F\left(\sqrt{\frac{\lambda}{R_{ii}}} |m_i^\vartheta|\right)^{1/2}+\frac{1}{\lambda}\right)\,.
    \end{align}
where we remind the reader that $m_i^\vartheta = W_i \cdot \vartheta$, $R_{ii} = W_i \cdot W_i$, and $F$ and $\bar F$ are the cdf and tail of a standard Gaussian respectively. 
\end{lem}
\begin{proof}
We begin with the estimate on the $vv$ block. Recall from~\pef{eq:XOR-Hessian} that
\begin{align}
\E[\nabla_{vv}^{2}L]=\E[\sigma'(v\cdot g(WY))g(WY)^{\tensor2}]\,.
\end{align}
By conditioning on the value of $\vartheta$, writing $\mathbb E_{\vartheta}$ for the conditional expectation, it suffices to bound
\[
\norm{\E_{\vartheta}[\sigma'(v\cdot g(W Y))g(WY)^{\tensor2}]-\sigma'(v\cdot g(W\vartheta))g(W\vartheta)^{\tensor2}}_{\op}\,,
\]
for each $\vartheta\in \{\pm \mu,\pm\nu\}$. To this end, we write 
\begin{align*}
\E_{\vartheta} &\left[\sigma'(v\cdot g(WY))g(WY)^{\tensor2}\right]-\sigma'(v\cdot g(W\vartheta))g(W\vartheta)^{\tensor2}\\  &=\E_\vartheta[\sigma'(v\cdot g(W\vartheta)))\cdot(g(WY)^{\tensor2}-g(W\vartheta)^{\tensor2})]  +\E_{\vartheta}[(\sigma'(v\cdot g(WY))-\sigma'(v\cdot g(W\vartheta)))g(WY)^{\tensor2}]\\
 & =:(i)+(ii)\,.
\end{align*}
By uniform boundedness of $\sigma'$, we then have by \pef{eq:gW-opnorm-conc}
that $\norm{(i)}_{\op}\lesssim_{K}\psi_{2}(\vartheta,W,\lambda).$ On the
other hand, by \pef{eq:gW-op-norm} and \pef{eq:y-l2-bound},
\begin{align*}
\norm{(ii)}_{\op} & \leq (\E[(\sigma'(v\cdot g(WY))-\sigma'(v\cdot g(W\vartheta)))^2])^{1/2}(\E\norm{g(WY)}_{\op}^{2})^{1/2}
  \lesssim_{B,K}\frac{1}{\lambda}\,.
\end{align*}

We now turn to the $W_iW_i$ block of the population Hessian. Recall from~\pef{eq:XOR-Hessian} and the fact that $W\in \mathcal W_0^c$, that 
\[
\E[\nabla_{W_{i}W_{i}}^{2}L]=\E[\widehat{y}(1-\widehat{y}) v^2_{i}g'(W_{i}\cdot Y)Y^{\otimes 2}]\,,
\]
(since $(g'(x))^2 = g'(x)$). 
We now aim to show the following two bounds, the second of which will give the desired bound of~\pef{e:HWterm}: for each $i$, we have that 
\begin{align}\label{e:yhat-prelim-bound}
\norm{\E [v_{i}^{2}g'(W_{i}\cdot Y)Y^{\tensor2}]-\frac{v_i^2}{4}\sum_{\vartheta\in\{\pm \mu,\pm \nu\}} F\Big(\sqrt{\frac{\lambda}{R_{ii}}} m_i^\vartheta\Big)\vartheta^{\tensor2}}\lesssim\max_{\vartheta\in\{\pm \mu,\pm \nu\}}\Big(\frac{1}{\sqrt{\lambda}} \bar F\Big(\sqrt{\frac{\lambda}{R_{ii}}} |m_i^\vartheta|\Big)^{1/2}\!+\frac{1}{\lambda}\Big)\,,
\end{align}
and 
\begin{align}\label{e:yhat}
    \norm{\E [\widehat{y}(1-\widehat{y}) v_{i}^{2}g'(W_{i}\cdot Y)Y^{\tensor2}]- & \frac{v_i^2}{4}\sum_{\vartheta\in\{\pm \mu,\pm \nu\}}\sigma(v \cdot g(W\vartheta))(1-\sigma(v \cdot g(W\vartheta))) F\Big(\sqrt{\frac{\lambda}{R_{ii}}} m_i^\vartheta\Big)\vartheta^{\tensor2}} \nonumber\\
    &\lesssim\max_{\vartheta\in\{\pm \mu,\pm \nu\}}\Big(\frac{1}{\sqrt{\lambda}}\cdot \bar F\Big(\sqrt{\frac{\lambda}{R_{ii}}} |m_i^\vartheta|\Big)^{1/2}+\frac{1}{\lambda}\Big).
\end{align}

Conditioning on $\vartheta$ and pulling out $v_i^2$ in~\pef{e:yhat-prelim-bound}, it suffices to bound
\begin{align*}
\E_{\vartheta}[g'(W_{i}\cdot Y)Y^{\tensor2}] & =\mathbb P_{\vartheta}(W_{i}\cdot Y>0)\vartheta^{\otimes 2}+2\Sym\E_{\vartheta}[\indicator{W_{i}\cdot Y>0}Z]\tensor \vartheta+\E_{\vartheta}[\indicator{W_{i}\cdot Y>0}Z^{\otimes 2}]\\
 & =:(i)+(ii)+(iii)\,.
\end{align*}
Term (i) is exactly the term to which we are comparing. In
particular 
\[
\mathbb P_{\vartheta}(W_{i}\cdot Y>0)=P(W_{i}\cdot \vartheta>W_{i}\cdot Z)= F\Big(\sqrt{\frac{\lambda}{R_{ii}}}m_i^\vartheta\Big)\,.
\]
We thus wish to bound the operator norms of (ii) and (iii). 
For term (ii), observe that for $u\in\bS^{d-1}$, since $\E_{\vartheta}\left\langle u,Z\right\rangle =0$,
we have that 
\[
\E_{\vartheta}[\indicator{W_{i}\cdot Y>0}\left\langle u,Z\right\rangle] =-\E_{\vartheta}[\indicator{W_{i}\cdot Y<0}\left\langle u,Z\right\rangle]\,.
\]
Thus we may bound this, by Cauchy-Schwarz by 
\begin{align*}
|\E_{m}[\indicator{W_{i}\cdot Y>0}\left\langle u,Z\right\rangle] | & \leq \mathbb E[\langle u,Z\rangle^2]^{1/2} \mathbb P_\vartheta((W_{i}\cdot Y) \sgn(m_i^\vartheta)<0)^{1/2}  =\frac{1}{\sqrt{\lambda}}\cdot \bar F\left(\sqrt{\frac{\lambda}{R_{ii}}} |m_i^\vartheta|\right)^{1/2}\,.
\end{align*}
It follows, using that $\|\vartheta\|=1$, that  
\[
\norm{(ii)}_{\op}=\sup_{u}\abs{\left\langle u,(ii)u\right\rangle }\leq\frac{2}{\sqrt{\lambda}}\cdot \bar F\left(\sqrt{\frac{\lambda}{R_{ii}}} |m_i^\vartheta|\right)^{1/2}.
\]
Finally, for  term $(iii)$ we do a naive bound to
get that for every $u\in \mathbb{S}^{d-1}$, 
\begin{align*}
\abs{\left\langle u,(iii)u\right\rangle } & \leq\abs{\E_{\vartheta}[\indicator{W_{i}\cdot Y>0}\left\langle u,Z\right\rangle ^{2}]}\leq(\E\left\langle u,Z\right\rangle ^{4})^{1/2}\lesssim\frac{1}{\lambda}\,.
\end{align*}

In order to show~\pef{e:yhat}, we can rewrite it as
\begin{align}\label{e:yhat2}
    &\phantom{{}={}}\norm{\E_\vartheta [\widehat{y}(1-\widehat{y}) g'(W_{i}\cdot Y)Y^{\tensor2}]-\sigma(v \cdot g(W\vartheta))(1-\sigma(v \cdot g(W\vartheta)))F\left(\sqrt{\frac{\lambda}{R_{ii}}} m_i^\vartheta\right)\vartheta^{\tensor2}}\\
 & \qquad =\norm{\E_\vartheta[ \widehat{y}(1-\widehat{y}) g'(W_{i}\cdot Y)Y^{\tensor2}]-\sigma(v \cdot g(W\vartheta))(1-\sigma(v \cdot g(W\vartheta)))\E_\vartheta  [g'(W_{i}\cdot Y)Y^{\tensor2}]} \nonumber\\
 & \qquad +\sigma(v g(W\vartheta))(1-\sigma(v g(W\vartheta)))\norm{\E_\vartheta [g'(W_{i}\cdot Y)Y^{\tensor2}]-F\left(\sqrt{\frac{\lambda}{R_{ii}}} m_i^\vartheta\right)\vartheta^{\tensor2}}=:(i)+(ii)\,. \nonumber
\end{align}
Since $\sigma(v g(W\vartheta))(1-\sigma(v g(W\vartheta)))\leq 1$, (ii) on the right-hand side in~\pef{e:yhat2} can be bounded as in \pef{e:yhat-prelim-bound} (without the $v_i$ terms). 
Term (i) on the right-hand side in~\pef{e:yhat2}, can be bounded as  
\[
(i)\leq\mathbb E_{\vartheta}[(\sigma(v\cdot g(W_{i}\cdot Y))(1-\sigma(v\cdot g(W_{i}\cdot Y)))-\sigma(v\cdot g(W_{i}\cdot \vartheta))(1-\sigma(v\cdot g(W_{i}\cdot Y))))^2]^{1/2}\sup_{\norm{u}_{2}\leq1}\E_{\vartheta}[\left\langle u,Y\right\rangle^{4}]^{1/2}
\]
where we used that $|g'(W_i\cdot Y)|\leq 1$.
Since $\E_{\vartheta}[\langle u,Y\rangle ^{4}]\lesssim1+\frac{1}{\lambda}+\frac{1}{\lambda^{2}}$
for $\norm{u}_{2}\leq1$, we have by \pef{eq:y-l2-bound} that $\norm{(i)}\lesssim_{B}\frac{1}{\lambda}$, 
from which the desired follows.
\end{proof}

\subsection{Bounds on population G-matrix}
We now turn to obtaining analogous $\lambda$-large approximations to the blocks of the population G-matrix. In what follows, let $y_\vartheta = 1$ if $\vartheta \in \{\pm \mu\}$ and $y_\vartheta = 0$ if $\vartheta \in \{\pm \nu\}$ be the label given that the mean is $\vartheta$.

\begin{lem}\label{lem:population-Gram-approximation}
For all $(v,W)\in B$, the first layer block of the population G-matrix satisfies
  \begin{align}\label{e:Gvterm}
\norm{\E[\nabla_v L^{\otimes 2}]-\frac{1}{4}\sum_{\vartheta\in \{\pm \mu, \pm \nu\}}(y_\vartheta -\sigma(v\cdot g(W\vartheta)))^{2}g(W\vartheta)^{\tensor2}}\lesssim_{K,B}\psi_{2}(\vartheta,W,\lambda)+\frac{1}{\lambda}\,,
  \end{align}
  and the second layer blocks satisfies,
   \begin{align}\label{e:GWterm}
\norm{\E[\nabla_{W_{i}} L]^{\otimes 2}-v_{i}^2A}\lesssim\frac{1}{\sqrt{\lambda}}\max_{\vartheta\in\{\pm \mu,\pm \nu\}} \bar F\left(\sqrt{\frac{\lambda}{R_{ii}}} |m_i^\vartheta|\right)^{1/2}+\frac{1}{\lambda}\,,
  \end{align}
where 
\begin{equation}
A=\frac{1}{4}\sum_{\vartheta\in\{\pm\mu,\pm \nu\}}(y_{\vartheta}-\sigma(v\cdot g(W\vartheta)))^{2}F\left(\sqrt{\frac{\lambda}{R_{ii}}} m_i^\vartheta\right)\vartheta^{\tensor2}\,.
\label{eq:A0}
\end{equation}
\end{lem}

\begin{proof}
We begin with the $v$-block. Recalling~\pef{e:gv},  and conditioning on the mean being $\vartheta$, it suffices to show for $\vartheta\in \{\pm \mu,\pm\nu\}$, 
\[
\norm{\E_{\vartheta}[\sigma(v\cdot g(WY))^{2}g(WY)^{\tensor2}]-\sigma(v\cdot g(W\vartheta))^{2}g(W\vartheta)^{\tensor2}}\lesssim_{K,B}\psi_{2}(\vartheta,W,\lambda)+\frac{1}{\lambda}\,,
\]
and similarly with negative signs in the sigmoids, to account for the $y_\vartheta$ term when relevant. 
The proofs are similar, so we just do the first, with a fixed choice of $\vartheta$. 
We begin by writing  
\begin{align*}
&\phantom{{}={}}\E_{\vartheta}[\sigma(v\cdot g(WY))^{2}g(WY)^{\tensor2}]-\sigma(v\cdot g(W\vartheta))^{2}g(W\vartheta)^{\tensor2} \\ & \qquad =\E_{\vartheta}\{[\sigma(v\cdot g(WY))^{2}-\sigma(v\cdot g(W\vartheta))^{2}]g(WY)^{\tensor2}\} +\E_{\vartheta}\{\sigma(v\cdot g(W\vartheta))[g(WY)^{\tensor2}-g(W\vartheta)^{\tensor2}]\}\\
 & \qquad =:(i)+(ii)\,.
\end{align*}
We begin with bounding the operator norm of term (i). To this end, by Cauchy--Schwarz, then~\pef{eq:gW-op-norm} and~\pef{eq:y-l2-bound}, 
\begin{align*}
\norm{(i)}_{\op} & \leq\E_{\vartheta}[\norm{g(WY)^{\tensor2}}_{\op}^{2}]^{1/2} \mathbb E_\vartheta[(\sigma(v\cdot g(WY))-\sigma(v\cdot g(W\vartheta)))^2]^{1/2}\,.
 \lesssim_{B}\frac{1}{\lambda}
\end{align*}
For term (ii), since $|\sigma|\le 1$,  
\[
\norm{(ii)}_{\op}\le\norm{\E_{\vartheta}[g(WY)^{\tensor2}]-g(W\vartheta)^{\tensor2}}_{\op}\lesssim_{K}\psi_{2}(\vartheta,W,\lambda)\,,
\]
by \pef{eq:gW-opnorm-conc}. Combining these two and averaging over $\vartheta$ yields~\pef{e:Gvterm}. 

Recalling~\pef{e:gw}, and conditioning on the mean being $\vartheta$, it suffices to control the operator norm of 
\begin{align*}
    \mathbb E_{\vartheta}[(\widehat y - y_\vartheta)^2 g'(W_i \cdot Y) Y^{\otimes 2}] - A_\vartheta
\end{align*}
where $A_\vartheta$ is the corresponding summand in~\pef{eq:A0}. Fix a $\vartheta$ with $y_\vartheta=0$, the choices being analogous. 
In this case we wish to bound the difference, 
\[
\E_{\vartheta}[\sigma(v\cdot g(WY))^{2}g'(W_{i}\cdot Y)Y^{\tensor2}]-\sigma(v\cdot g(W\vartheta))^{2}F(\sqrt{\frac{\lambda}{R_{ii}}}m_i^\vartheta)\vartheta^{\tensor2}=(i)+(ii)\,,
\]
where
\begin{align*}
(i) & :=\E_{\vartheta}[(\sigma(v\cdot g(WY))^{2}-\sigma(v\cdot g(W\vartheta))^{2})g'(W_{i}\cdot Y)Y^{\tensor2}]\,,\\
(ii) & :=\sigma(v\cdot g(W\vartheta))^{2}\left(\E_{\vartheta}[g'(W_{i}\cdot Y)Y^{\tensor2}]-F\left(\sqrt{\frac{\lambda}{R_{ii}}} m_i^\vartheta\right)\vartheta^{\tensor2}\right)\,.
\end{align*}
By~\pef{e:yhat-prelim-bound} and the
boundedness of $\sigma$,
\[
\norm{(ii)}_{\op}\lesssim\max_{\vartheta\in\{\pm \mu,\pm \nu\}}\left(\frac{1}{\sqrt{\lambda}}\cdot \bar F\left(\sqrt{\frac{\lambda}{R_{ii}}} |m_i^\vartheta|\right)^{1/2}+\frac{1}{\lambda}\right)\,.
\]
For $(i)$, we bound its operator norm by Cauchy--Schwarz, as  
\[
\norm{(i)}_{\op}\leq\mathbb E[(\sigma(v\cdot g(W_{i}\cdot Y))^{2}-\sigma(v\cdot g(W_{i}\cdot \vartheta))^{2})^2]^{1/2} \sup_{\norm{u}_{2}\leq1}\E_{\vartheta}[\left\langle u,Y\right\rangle ^{4}]\,,
\]
where we used that $|g'(W_i\cdot Y)|\leq 1$.
Since $\E_{\vartheta}[\left\langle u,Y\right\rangle ^{4}]\lesssim1+\frac{1}{\lambda}+\frac{1}{\lambda^{2}}$
for $\norm{u}_{2}\leq1$, we have by \pef{eq:y-l2-bound} that $\norm{(i)}\lesssim_{B}\frac{1}{\lambda}$, 
 from which the desired follows.
\end{proof}

\section{Analysis of the SGD trajectories}\label{sec:SGD-analysis}

Our goal in this section is to show the following results for the SGD for our two classes of classification tasks. The first is our main result on stochastic gradient descent for the mixture of $k$-GMM's via a single-layer network. 

\begin{prop}\label{prop:k-GMM-SGD-result}
    For every $\epsilon,\beta>0$, there exists $T_0$ such that for all $T_f>T_0$ and all $\lambda$ large, the SGD for the $k$-GMM with step size $\delta = O(1/d)$, initialized  from $\cN(0,I_d/d)$, does the following for all $c\in [k]$ with probability $1-o_d(1)$:
    \begin{enumerate}
        \item There exists $L(\beta)$ (independent of $\epsilon,\lambda$) such that $\|\mathbf{x}_\ell^c\|\le L$ for all $\ell \in [T_0 \delta^{-1},T_f\delta^{-1}]$;
        \item There exists $\eta(\beta)>0$ (independent of $\epsilon,\lambda$) such that $\mathbf{x}_\ell^c$ is within $O(\epsilon + \lambda^{-1})$ distance of a point in $\text{Span}(\mu_1,...,\mu_k)$ having $\|\mathbf{x}^c\|>\eta$.  
    \end{enumerate}
\end{prop}

The following is our analogous main result for the XOR GMM with two-layer networks. 

\begin{prop}\label{prop:XOR-GMM-SGD-result}
    For every $\epsilon>0$, there exists $T_0$ such that for all $T_f>T_0$ and all $\lambda$ large, the SGD for the XOR GMM with width $K$ fixed, with step-size $\delta = O(1/d)$, initialized from $\cN(0,I_d/d)$ in its first layer weights, and $\cN(0,1)$ i.i.d.\ in its second layer weights, with probability $1-o_d(1)$, satisfies: 
    \begin{enumerate}
        \item There is an $L(K,\beta,v(\mathbf{x}_0))$ (independent of $\epsilon,\lambda$) such that $\|W_i(\mathbf{x}_\ell)\|\le L$ for all $i\le K$, and $v(\mathbf{x}_\ell)\le L$ for all $\ell \in [T_0\delta^{-1},T_f \delta^{-1}]$;
        \item If $\beta<1/8$, there exists $\eta(K,\beta,v(\mathbf{x}_0))>0$ (independent of $\epsilon,\lambda$) such that $\mathbf{x}_\ell$ has 
        \begin{align*}
             |v_i(\mathbf{x}_\ell)|>\eta \quad \text{and}\quad  \max_{\vartheta \in \{\pm\mu,\pm\nu\}} W_i(\mathbf{x}_\ell)\cdot \vartheta > \eta\quad \text{for all $\ell \in [T_0\delta^{-1}, T_f\delta^{-1}]$ and $i\le K$.}
        \end{align*}
        and furthermore, it has $W_i(\mathbf{x}_\ell)$ lives in $\text{Span}(\mu,\nu)$  and $v(\mathbf{x}_\ell)$ lives in $\text{Span}((g(W(\mathbf{x}_\star) \vartheta))_{\vartheta \in \{\pm\mu,\pm\nu\}}$ up to error $\epsilon + \lambda^{-1/2}$. 
    \end{enumerate}
    Note, the dependencies of $L$ and $\eta$ on the absolute initial second layer value $|v(\mathbf{x}_0)|$ are continuous. 
\end{prop}

Our approach goes by first recalling the main result of~\citet{BGJ22} regarding a limit theorem as $d\to\infty$ for the trajectories of summary statistics of stochastic gradient descent. We then apply this result to the task of classifying $k$ Gaussian mixtures, obtaining ballistic limits for the SGD in Theorem~\ref{thm:k-GMM-ballistic-limit} that may be of independent interest. We then analyze the ballistic limits to establish Proposition~\ref{prop:k-GMM-SGD-result}. Then, we recall the ballistic limits for the XOR Gaussian mixture from \citet{BGJ22} and similarly establish Proposition~\ref{prop:XOR-GMM-SGD-result}. Since we are imagining the dimensions of the parameter space and dimension space, and the step size to all be scaling with relation to one another, we use a dummy index $n$, in this section to encode their mutual relationship, namely $d = d_n, p=p_n$ and $\delta = \delta_n$ and then $n\to\infty$.  

Also, to compress notation throughout this section, we will combine the loss function with the regularizer $\frac{\beta}{2} \|\mathbf{x}\|^2$ so that the stochastic gradient descent updates are indeed making gradient updates with respect to the new loss: 
\begin{align*}
    \mathbf{x}_\ell = \mathbf{x}_{\ell - 1} - \delta \nabla \bar{L}(\mathbf{x}_{\ell -1},\mathbf{Y}^\ell) \quad \text{where} \quad \bar{L}(\mathbf{x}) = L(\mathbf{x}) + \tfrac{\beta}{2}\|\mathbf{x}\|^2\,.
\end{align*}

\subsection{Recalling the effective dynamics for summary statistics}
Suppose that we are given a sequence of functions
$\mathbf{u}_n\in C^{1}(\R^{p_n};\R^{k})$  for some fixed $k,$
where $\mathbf{u}_{n}(x)=(u_{1}^n(x),...,u_{k}^n(x))$,
and our goal is to understand the evolution of $\mathbf{u}_n(\mathbf{x}_\ell)$.

In what follows, let $H(\mathbf{x},\mathbf{Y})=\bar{L}_n(\mathbf{x},\mathbf{Y})-\Phi(\mathbf{x})$, where $\Phi(\mathbf{x})  = \mathbb E [ \bar{L}_n(\mathbf{x},\mathbf{Y})]$. (Note that since the regularizer term is non-random, $H$ is the same with $L$ in place of $\bar{L}$.)
Throughout the following, we suppress the dependence of $H$ on
$\mathbf{Y}$ and instead view $H$ as a random function of $\mathbf{x}$, denoted $H(\mathbf{x})$. We let 
$V(\mathbf{x})=\E_{\mathbf{Y}}\left[\nabla H(\mathbf{x})\tensor\nabla H(\mathbf{x})\right]$
denote the covariance matrix for $\nabla H$ at~$\mathbf{x}$.

In order to develop a theory for the high-dimensional limiting trajectories of the functions $\mathbf{u}_n$, which we will call summary statistics following~\cite{BGJ22}, we need to assume: 
\begin{enumerate}
    \item A certain amount of regularity of moments of these functions and their derivatives, which will be relative to the step size $\delta_n$, and be called \emph{$\delta_n$-localizability}; 
    \item That in the dimension to infinity limit, the drift and volatility of the evolution of $\mathbf{u}_n$ are asymptotically expressible as functions of $\mathbf{u}_n$ themselves, rather than needing the entire vector in parameter space. We call this \emph{asymptotic closability} of the function family. 
\end{enumerate}
We now give the precise form of these two definitions before moving on to state the general theorem of~\cite{BGJ22}, which we will apply to the $k$-GMM classification task.  

\begin{defn}\label{defn:localizable}
A triple $(\mathbf{u}_n,L_n,P_n)$ is \textbf{$\delta_n$-localizable}
with localizing sequence $(E_K)_K$ if there is an exhaustion by compacts $(E_{K})_K$ of $\R^{k}$, and constants $C_K$ (independent of $n$)
such that 
\begin{enumerate}
\item $\max_{i} \sup_{\mathbf{x}\in\mathbf{u}_{n}^{-1}(E_{K})}\norm{\nabla^{2}u_{i}^n}_{\op}\leq C_K\cdot\delta_n^{-1/2}$,  and $\max_{i} \sup_{\mathbf{x}\in\mathbf{u}_{n}^{-1}(E_{K})}\norm{\nabla^{3}u_{i}^n}_{\op}\leq C_K$;
\smallskip
\item $\sup_{\mathbf{x}\in\mathbf{u}_{n}^{-1}(E_{K})}\|\nabla\Phi\|\le C_K$, and 
$\sup_{\mathbf{x}\in\mathbf{u}_{n}^{-1}(E_{K})}\mathbb{E}[\|\nabla H\|^{8}]\le C_K\delta_{n}^{-{ 4}}$;
\smallskip
\item $\max_{i}\sup_{\mathbf{x}\in\mathbf{u}_{n}^{-1}(E_{K})}\E[\langle \nabla H,\nabla u_{i}^n\rangle ^{4}]\leq C_K\delta_{n}^{-2}$, and \hfill 

$\max_{i} \sup_{\mathbf{x}\in \mathbf u_n^{-1}(E_{K})} \mathbb E[\langle \nabla^2 u_i^n, \nabla H\otimes \nabla H - V\rangle^2] = o(\delta_n^{-{ 3}})$.
\end{enumerate}
\end{defn}

Define the following first and second-order differential operators, 
\begin{align}\label{eq:A-L-operators}
\cA_n = \sum_{i} \partial_i \Phi \partial_i\,, \qquad \mbox{and} \qquad 
\cL_{n}=\frac{1}{2}\sum_{i,j} V_{ij}\partial_{i}\partial_{j}\,.
\end{align}
Alternatively written,  $\cA_n = \langle \nabla \Phi, \nabla\rangle$ and $\mathcal L_n = \frac{1}{2} \langle V, \nabla^2\rangle$. Let $J_n$ denote the Jacobian matrix $\nabla \mathbf{u}_n$.

\begin{defn}\label{defn:asympotically-closable}
A family of summary statistics $(\mathbf{u}_n)$ are \textbf{asymptotically closable} for step-size $\delta_n$ if $(\mathbf{u}_n, L_n, P_n)$ are $\delta_n$-localizable with localizing sequence $(E_K)_K$, and furthermore there exist locally Lipschitz functions $\mathbf{h}:\mathbb R^k\to\mathbb R^k$ 
and $\mathbf{\Sigma}:\R^{k}\to \mathbb R^{k\times k}$, such that
\begin{align}
\sup_{\mathbf{x}\in \mathbf{u}_n^{-1}(E_K)} \big\| \big( - \cA_n + \delta_n \mathcal L_n\big) \mathbf{u}_n(\mathbf{x}) - \mathbf h(\mathbf{u}_n(\mathbf{x}))\big\| & \to 0\,,  \label{eq:eff-drift} \\
\sup_{\mathbf{x}\in\mathbf{u}_{n}^{-1}(E_{K})}\|\delta_{n}J_{n} V J_{n}^{T}-\mathbf{\Sigma}(\mathbf{u}_{n}(\mathbf{x}))\| & \to0\,. \label{eq:diffusion-matrix}
\end{align}
In this case we call $\mathbf{h}$ the \emph{effective drift}, and $\mathbf{\Sigma}$ the \emph{effective volatility}.
\end{defn}

For a function $f$ and measure $\mu$ we let $f_{*}\mu$
denote the push-forward of $\mu$. The main result of~\citet{BGJ22} was the following limit theorem for SGD trajectories as $n\to\infty$ .

\begin{thm}[{\citet[Theorem 2.2]{BGJ22}}]\label{thm:main-BGJ22}
Let $(\mathbf{x}_{\ell}^{\delta_{n}})_{\ell}$ be stochastic gradient descent initialized from $\mathbf{x}_{0}\sim\mu_{n}$
for $\mu_{n}\in \cM_{1}(\mathbb{R}^{p_{n}})$ with learning
rate $\delta_{n}$ for the loss $L_{n}(\cdot,\cdot)$ and data distribution
$P_{n}$. For a family of summary statistics $\mathbf{u}_n = (u_i^n)_{i=1}^k$, let $(\mathbf{u}_n(t))_t$ be the linear interpolation of $(\mathbf{u}_n(\mathbf{x}_{\lfloor t\delta_{n}^{-1}\rfloor}^{\delta_{n}}))_{t}$. 

Suppose that $\mathbf{u}_n$ are asymptotically closable with learning rate $\delta_n$, effective drift $\mathbf{h}$, and effective volatility $\mathbf{\Sigma}$, and that the pushforward of the initial data has $(\mathbf{u}_n)_*\mu_n \to \nu$ weakly for some $\nu\in \cM_1(\mathbb R^k)$. Then $(\mathbf{u}_{n}(t))_{t}\to(\mathbf{u}_t)_t$ weakly as $n\to\infty$, where
$\mathbf{u}_t$ solves
\begin{equation}\label{eq:effective-dynamics}
d\mathbf{u}_{t}= \mathbf{h}(\mathbf{u}_{t}) dt+\sqrt{\mathbf{\Sigma}(\mathbf{u}_{t})}d\mathbf{B}_{t}\,.
\end{equation}
initialized from $\nu$, where $\mathbf{B}_{t}$ is a standard Brownian
motion in $\mathbb{R}^{k}$. 
\end{thm}

As a result, we can read off in the ballistic regime the following finite-$n$ approximation result. 

\begin{lem}\label{lem:closeness-to-ballistic-trajectory}
    In the setting of Theorem~\ref{thm:main-BGJ22}, if the limiting dynamics have $\mathbf{\Sigma} \equiv 0$, then for every $T$, we have with probability $1-o_d(1)$, 
    \begin{align*}
        \|\mathbf{u}(\mathbf{x}_{\lfloor \delta^{-1}t\rfloor}) - \mathbf{u}_t\|_{C[0,T]}  = o_d(1)\,.
    \end{align*}
\end{lem}

\subsection{1-layer networks: ballistic limits}
Our first aim in this section  is to show that the limit theorem of Theorem~\ref{thm:main-BGJ22} indeed applies to the mixture of $k$ Gaussians. 

Recall the definitions of the data distribution, and the cross-entropy loss for $x\in \mathbb R^{kd}$ (i.e., $x^a \in \mathbb R^d$ for $a\in[k]$) from~\pef{eq:cross-entropy-loss}. Recall that $\overline{m}_{ab} = \mu_a\cdot \mu_b$, and that in order for the problem to be linearly classifiable (and therefore solvable with this loss) we assume that the means are linearly independent. For this task, the input dimension $d = n$, the parameter dimension $p_n = k n$ and the step-size will be taken to be $\delta = O(1/n) = O(1/d)$, and we will tend to use $d$ rather than the dummy index $n$.
It will be helpful to write $Y_a = \mu_a + Z_\lambda$, and recall for $a,c\in [k]$, the Gibbs probability $\pi_{Y_a}(c) = \pi_{Y_a}(c;\mathbf{x})$ as defined in~\pef{eq:pi-dist}. 

\subsubsection{The summary statistics}
We will show that the following family of functions form a set of localizable summary statistics $\mathbf{u}_n(x) = ((m_{ab}(x))_{a,b},(R_{ab}^\perp)_{ab})$. 
\begin{align*}
	\mathbf{m} = (m_{ab})_{a,b\in [k]}  \qquad \text{where} \qquad m_{ab} & = x^a \cdot \mu_b \\ 
	\mathbf{R}^\perp = (R^\perp_{ab})_{a,b\in [k]} \qquad \text{where} \qquad R_{ab}^\perp & = x^{a,\perp} \cdot x^{b,\perp}
\end{align*}
where $x^{a,\perp}$ denotes the part of $x^a$ orthogonal to $\text{Span}(\mu_1,...,\mu_k)$, i.e., $x^{a,\perp} = \mathsf{P}^\perp x^a$, where we use $\mathsf{P}^\perp$ to be the projection operator into the orthogonal complement of $\text{Span}(\mu_1,...,\mu_k)$.

It is not hard to see that the law of the full sequence $(\pi_{Y_a}(c))_{a,c}$ (and therefore its moments etc...) only depend on $x$ through $\mathbf{m}$ and $\mathbf{R}^\perp$, and therefore they have no finite $d$ dependence. The following describes the ODEs obtained by taking the simultaneous $d\to\infty$ and $\delta = O(1/d)$, limit of the SGD trajectory in its summary statistics.

\begin{thm}\label{thm:k-GMM-ballistic-limit}
	If the step sizes $\delta= c_\delta/d$, then the summary statistics $\mathbf{u}_n = (\mathbf{m}, \mathbf{R}^\perp)$ are $\delta$-localizable, and satisfy the following ballistic limit as $d\to\infty$:
	\begin{align}
	 	\dot{m}_{ab}(t)& = p_a \overline{m}_{ab} - \beta m_{ab} - \sum_{c\in [k]} p_c \big(\overline{m}_{cb} P_a^c + Q_a^{c,\mu_b}\big)\,, \label{eq:k-GMM-ballistic-limit-m}\\ 
		\dot{R}_{ab}^\perp(t) & =  -  \beta R_{ab}^\perp + \sum_{c\in [k]} \big( p_a Q_{a}^{c,R_{bb}^\perp} + p_b Q_b^{c,R_{aa}^\perp}) -  \sum_{c\in [k]} \tfrac{c_\delta p_c}{\lambda} \mathbb E\big[(\pi_{Y_c}(a)  - \mathbf 1_{a=c})(\pi_{Y_c}(b)- \mathbf 1_{b=c})\big] \,, \label{eq:k-GMM-ballistic-limit-R}
	\end{align}
where $P_{a}^c, Q_{a}^{c,v}$ are the following Gaussian integrals: 
	\begin{align}
	P_{a}^c(\mathbf{m},\mathbf{R}^\perp) & = \mathbb E[\pi_{Y_c}(a)] \quad \text{and}\quad 
 Q_{a}^{c,v} (\mathbf{m},\mathbf{R}^\perp) =   \mathbb E[\langle Z_\lambda,v\rangle \pi_{Y_c}(a) ] \label{eq:P-Q-Gaussian-integrals}
	\end{align} 
 and where if $v\perp \text{Span}(\mu_1,...,\mu_k)$ we use the shorthand $Q_{a}^{c,r}$ for $r=\|v\|^2$ since for such $v$, $Q_a^{c,v}$ only depends on $v$ through $\|v\|^2$. 
	The case where $\delta = o(1/d)$ is read-off by formally setting $c_\delta = 0$. 
\end{thm}

\begin{rem}
    While a priori, it may appear that the quantity $Q_{a}^{c,v}$ in~\pef{eq:P-Q-Gaussian-integrals} depends on the dimension $d$, if $v$ is one of $(\mu_1,...,\mu_k)$, or if it is simply orthogonal to $\text{Span}(\mu_1,...,\mu_k)$, then $Q_{a}^{c,v}$ does not depend on $d$. Such cases are the only ones appearing in~\pef{eq:k-GMM-ballistic-limit-m}--\pef{eq:k-GMM-ballistic-limit-R}. The same can be said for the expectation in the last term of~\pef{eq:P-Q-Gaussian-integrals}.  
\end{rem}

\begin{proof}Theorem~\ref{thm:k-GMM-ballistic-limit} will follow from an application of Theorem~\ref{thm:main-BGJ22} to the summary statistics $\mathbf{m},\mathbf{R}^\perp$ for the $k$-GMM, so we start by verifying that the problem fits the assumptions of $\delta$-localizability and asymptotic closability. 

\medskip
\noindent 
\emph{Verifying $\delta$-localizability}
Our aim is to verify the conditions of $\delta$ localizability from Definition~\ref{defn:localizable}. Let us begin with some calculations, recalling that the Jacobian matrix $J = \nabla \mathbf{u}$. Let $\nabla_a$ denote the derivative in the $\mathbb R^d$ coordinates corresponding to $x^a$. For $a,b,c\in [k]$,  
\begin{align}\label{eq:Jacobian}
	\nabla_c m_{ab} = \begin{cases}
	    \mu_b & \quad \text{if $c=a$}  \\ 0 & \quad \text{else}
	\end{cases}\,, \qquad \text{and} \qquad
	\nabla_c R_{ab}^{\perp}  = \begin{cases} x^{a,\perp}  & a\ne b=c \\ x^{b,\perp} & b\ne a = c \\  2x^{a,\perp} & a=b=c \\ 0 & \text{else}\end{cases}\,.
\end{align}
Continuing, all higher derivatives of $m_{ab}$ are zero. The Hessian of $R_{ab}^\perp$ is given by 
\begin{align}\label{eq:k-GMM-sum-stat-Hessian}
	\nabla_{ab} m_{cd}= 0 \quad \text{and} \quad \nabla_{ab} R_{ab}^\perp = \begin{cases} \mathsf{P}^{\perp}
     & \text{if $a\ne b$} \\ 2\mathsf{P}^\perp & \text{if $a=b$}\end{cases}\,.
\end{align}
(and other blocks are zero), and higher derivatives of $R_{ab}^\perp$ are also zero. 

Let us also express the loss function as equal in distribution to a random variable whose law depends only on $\mathbf{m}$ and $\mathbf{R}^\perp$. 
For ease of notation, let 
\begin{align*}
	R_{ab} = \langle x^a,x^b\rangle  \qquad \text{and} \qquad Y_a = \mu_a + Z_\lambda\,. 
\end{align*}
Recalling \pef{eq:cross-entropy-loss} and adding in the regularizer, we can write for each fixed $\mathbf{x}$, 
\begin{align}\label{eq:k-GMM-loss-in-summary-variables}
	\bar L(\mathbf{x},\mathbf{Y}) \stackrel{(d)}= - \sum_{a\in [k]} y_a \big(m_{aa} + G_a\big) +  \sum_{b\in [k]} y_b \log \sum_{a\in [k]}  e^{ m_{ab} + G_a} + \tfrac{\beta}{2} \sum_{a\in [k]} R_{aa}\,,
\end{align}
where $(G_a)_a$ is a Gaussian vector with covariance matrix $(R_{ab})$, and $y$ is a uniformly drawn $1$-hot vector in $\mathbb R^k$. 
Observe that the law of $\bar L$ only depends on $\mathbf{x}$ through the values of the summary statistics.

\bigskip

\begin{lem}\label{lem:k-GMM-localizable}
    Suppose $\delta = O(1/d)$, $\lambda>0$ is fixed (or growing with $d$) and $\beta>0$ is fixed. Then the family of summary statistics $\mathbf{u} = (\mathbf{m},\mathbf{R}^\perp)$ are $\delta$-localizable with balls. 
\end{lem}

\begin{proof}\noindent \emph{Item (1) of $\delta$-localizability.} The first part is easily seen to be satisfied by \pef{eq:k-GMM-sum-stat-Hessian} since the Hessians of statistics in $\mathbf{m}$ are all zero, and the Hessian of $R_{ab}^\perp$ is bounded in operator norm by $2 + 2k$ using the triangle inequality and $\|\mu_b\|=1$. 

\bigskip
\noindent \emph{Item (2) of $\delta$-localizability.} We begin with the bound on the population loss's norm. By taking expectation of~\pef{eq:k-GMM-loss-in-summary-variables} 
\begin{align}\label{eq:k-GMM-population-loss}
	\Phi(\mathbf{x}) = - \sum_{a\in [k]} p_a m_{aa}  + \sum_{a\in [k]} p_a \mathbb E\Big[\log \sum_{b\in [k]} e^{ \langle x^b,Y_a\rangle} \Big] + \tfrac{\beta}{2} \sum_{a\in [k]} R_{aa}\,.
\end{align}	
Taking the derivative of this, we get for each $c\in [k]$, 
\begin{equation}\label{eq:grad-phi-gmm}
	\nabla_c \Phi(\mathbf{x}) = - p_c \mu_c + \sum_{a\in [k]} p_a \mathbb E[ \pi_{Y_a}(c)Y_a] + {\beta} \mathbf{x}^c\,,
\end{equation}
 where $Y_a = \mu_a + Z_\lambda$ and $\pi$ is  as in \pef{eq:pi-dist}. 
Considering the norm of $\|\nabla \Phi\|$, we have 
\begin{align*}
	\|\nabla \Phi\| \le k \max_{c} \|\nabla_c \Phi\| \lesssim k(1+1+ \lambda^{-1/2} + \beta \max_{a} R_{aa}) 
\end{align*}
which is bounded by a constant $C(K)$ for $\mathbf{m},\mathbf{R}^\perp$ in a ball of radius $K$.  

Moving on to the bound on $\mathbb E[\|\nabla H\|^8]$, first observe using $\nabla_c H = \nabla_c \bar L - \nabla_c \Phi$, that  
\begin{align}\label{eq:k-GMM-H}
	\nabla_c H  = (p_c \mu_c - y_c (\mu_c + Z_\lambda )) + \Big( \pi_{Y}(c)Y  -   \mathbb E[ \pi_{Y}(c)Y]   \Big)
\end{align}

Taking the norm and the $8$'th moment, we use the fact that $\|u+v\|^8 \lesssim (\|u\|^8 + \|v\|^8)
$, and the fact that $\pi_{Y_a}(\cdot)$ is a probability mass function and therefore at most $1$, to upper bound 
\begin{align*}
	\max_{c\in [k]} \mathbb E[\|\nabla_c H\|^8]\lesssim \sup_a \|\mu_a\|^8+\E\|Z_\lambda\|^8 \le  1 + O((d/\lambda)^4)\,.
\end{align*}
As long as $\delta_n = O(1/d)$, since $\lambda$ is uniformly bounded away from zero, this will be bounded by a constant times $\delta_n^{-4}$ as required.  

\bigskip
\noindent \emph{Item (3) of $\delta$-localizability}. 
We next turn to bounding the fourth moments of the directional derivatives, starting with the statistics $m_{ab}$: 
\begin{align*}
	\langle \nabla_c H,\nabla_a m_{ab}\rangle  &  = \langle (p_c -y_c) \mu_c , \mu_b\rangle - \langle y_c Z_\lambda, \mu_b\rangle + p_b \mathbb E[\pi_{Y_b}(c)] - y_b \pi_{Y_b}(c) \\
	& \qquad + \sum_{l\in [k]} \Big( p_l \mathbb E [ \pi_{Y_l}(c) \langle \mu_b, Z_\lambda \rangle] + y_l \pi_{Y_l}(c)\langle \mu_b,Z_\lambda\rangle\Big)\,.
\end{align*}
Taking the fourth moment, again bounding things by the fourth moments of the terms individually up to a universal constant, then taking an expected value, we see that the first term is bounded by~$1$, the second by the fourth moment of a Gaussian random variable with variance $1/\lambda$, i.e., by $C/\lambda^2$, the third and fourth by $1$ since $\pi$ is a probability distribution, and the summands individually have fourth moments bounded by $C/\lambda^2$ for the same reason. Altogether, we get 
\begin{align*}
	\mathbb E[\langle \nabla_c H,\nabla_a m_{ab}\rangle^4 ]\lesssim 1+ O(1/\lambda^{2})\,,
\end{align*}
satisfying the requisite bound with room to spare. 
Turning to the directional derivative in the direction of $R_{ab}^\perp$, note that derivatives of $R_{ab}^\perp$ are orthogonal to $(\mu_1,...,\mu_k)$ so for $a\ne b$, 
\begin{align*}
	\langle \nabla_c H,\nabla_b R_{ab}^\perp \rangle  = \langle \nabla_c H,\nabla_b R_{ba}^\perp\rangle   =  \Big( \mathbb E[\pi_{Y}(c) \langle x^{a,\perp}, Z_\lambda\rangle ]  -   \pi_{Y}(c)\langle x^{a,\perp}, Z_\lambda\rangle \Big) - y_c \langle x_a^\perp ,Z_\lambda\rangle\,.
\end{align*}
Considering the fourth moment of the above, using that $\pi$ is bounded by $1$, and $\langle x^{a,\perp}, Z_\lambda \rangle$ is distributed as a Gaussian random variable with variance $R_{aa}^\perp/\lambda$, we  obtain 
\begin{align*}
	\mathbb E[\langle \nabla_c H,\nabla_b R_{ab}^\perp\rangle^4]\lesssim (R^\perp_{aa}/\lambda)^2\,,
\end{align*}
which is bounded by $C(K)$ while $R^\perp_{aa}$ is bounded by $K$. The diagonal $a=b$ is the same up to a factor of $2$. 
The last thing to check is the second part of item (3) in $\delta_n$-localizability. For this purpose, recall that $V(x) = \mathbb E[\nabla H^{\otimes 2}]$, and notice from~\eqref{eq:k-GMM-H} that for the $k$-GMM we have 
\begin{align*}
    \nabla_c H \otimes \nabla_d H = ((\pi_Y(c) - y_c) Y - \mathbb E[(\pi_Y(c) - y_c) Y])\otimes ((\pi_Y(d) - y_d) Y - \mathbb E[(\pi_Y(d) - y_d) Y])
\end{align*}

Taking expected values, we get that the $cd$-block of $V$ is given by
\begin{align}\label{eq:XOR-V-matrix}
    V_{cd}=\mathbb E[\nabla_c H \otimes \nabla_d H]  = \text{Cov}\Big((\pi_Y(c) - y_c)Y, (\pi_Y(d) - y_d) Y\Big)\,, 
\end{align}
where $\text{Cov}$ is the covariance matrix associated to the vectors inside it. 

For the second part of item (3) in localizability, we only need to consider the statistics $R_{ab}^\perp$ since the second derivatives of $m_{ab}$ are all zero. For the ballistic limit it is sufficient to use the bound 
\begin{align*}
	\mathbb E[\langle \nabla^2 R_{ab}^\perp, \nabla_c H \otimes \nabla_d H - V_{cd}\rangle^2 ] \lesssim \|\nabla^2 R_{ab}^\perp\|_{\op}^2 \cdot \mathbb E[\| \nabla_c H\|^4]\,.
\end{align*}
The first term is bounded by $2$ per \pef{eq:k-GMM-sum-stat-Hessian}. The second can be seen to be bounded via the $8$th moment above by $O((d/\lambda)^2)$ which is $o(\delta^{-3})$ when $\delta = O(1/d)$. 
\end{proof}

\begin{rem}
	The reader might notice that there is a lot of room in the bounds above as compared to the allowed thresholds from the $\delta$-localizability conditions. The weaker conditions in localizability are to allow taking diffusive limit theorems about saddle points by rescaling the summary statistics by $d$-dependent factors, which can help with understand timescales to escape fixed point regions of the limits of Theorem~\ref{thm:k-GMM-ballistic-limit}. This was explored in much detail for matrix and tensor PCA in \citet{BGJ22}. 
\end{rem}

\medskip
\noindent 
\emph{Calculating the drift and corrector} 
Now that we have verified that the $\delta$-localizability conditions apply to the summary statistics $\mathbf{u} = (\mathbf{m},\mathbf{R}^\perp)$, we compute the limiting ODE one gets in the $d\to\infty$ limit for these statistics. We will establish individual convergence of $\cA u$ to some $f_u(\mathbf{u})$ and convergence of $\delta \cL u$ to some $g_u(\mathbf{u})$ for each $u\in \mathbf{u}$, whence $h$ from Definition~\ref{defn:asympotically-closable} equals $-f + g$.

Recall the differential operator $\mathcal{A}$ from \pef{eq:A-L-operators}, the expression for $\nabla \Phi$ from \pef{eq:grad-phi-gmm}, and consider 
\begin{align*}	
	\mathcal A m_{ab} = \sum_{c\in [k]} \langle \nabla_c \Phi, \nabla_c m_{ab}\rangle & = \langle \nabla_a \Phi, \mu_b\rangle 
	 = - \langle p_a \mu_a, \mu_b\rangle  +  \mathbb E[\pi_{Y}(a)\langle Y,\mu_b\rangle]  + \beta \langle x^a,\mu_b\rangle\,.
\end{align*}
Recalling that $\overline{m}_{ab} = \langle \mu_a,\mu_b\rangle$,  we get 
\begin{align*}
	\mathcal A m_{ab} = -p_a \overline{m}_{ab} + \beta m_{ab} + \sum_{l\in [k]} p_l \Big( \overline{m}_{lb} \mathbb E[{\pi}_{Y_l}(a)] + \mathbb E[\pi_{Y_l}(a) \langle Z_\lambda, \mu_b\rangle]\Big)\,.
\end{align*}
Notice that the two expected values are Gaussian expectations that only depend on $x$ through $(m_{cl})_{c}$ and $(R_{lm}^\perp)_{l,m}$. In particular, we can take the limit as $d\to\infty$ to get the limiting drift function 
\begin{align*}
	f_{m_{ab}}(\mathbf{m},\mathbf{R}^\perp) = -p_a \overline{m}_{ab} + \beta m_{ab} + \sum_{l\in [k]} p_l \Big( \overline{m}_{lb} P_{a}^{l} + Q_{a}^{l,\mu_b}\Big)\,.
\end{align*}
where $P_{a}^l$ and $Q_{a}^{l,\mu_b}$ are defined per \pef{eq:P-Q-Gaussian-integrals}.  The contribution coming from $\delta\cL m_{ab}$ vanishes in the $d\to\infty$ limit 
 since the second derivative of $m_{ab}$ is zero, i.e., $g_{m_{ab}} =0$. 

For the drift function for $R_{ab}^\perp$, since $\langle x^{a,\perp},\mu_b\rangle =0$ for all $a,b$, we get 
\begin{align*}
	\mathcal A R_{ab}^\perp = \beta R_{ab}^\perp + \sum_{l\in [k]} (p_a Q_a^{l,x^{b,\perp}}  + p_b Q_b^{l,x^{a,\perp}})\,.
\end{align*}
In particular, we get 
\begin{align*}
	f_{R_{ab}^\perp} = \beta R_{ab}^\perp + \sum_{l\in [k]} \big( p_a Q_{a}^{l,R_{bb}^\perp} + p_b Q_b^{l,R_{aa}^\perp} \big)\,,
\end{align*}
using  $Q_a^{l,R_{bb}^\perp} = Q_a^{l,x^{b,\perp}}$ since $Q_a^{l,v}$ only depended on $v$ through its norm when $v \perp \text{Span}(\mu_1,...,\mu_k)$. 

There is also a contribution from $\delta \cL R_{ab}^\perp$; recall that $\mathcal L  = \tfrac{1}{2} \langle V, \nabla^2\rangle$ and notice that 
\begin{align*}
	\mathcal L R_{ab}^\perp = \tfrac{1}{2} \big(\langle V_{ab}, \mathsf{P}^\perp\rangle + \langle V_{ba},\mathsf{P}^\perp\rangle \big)\,.
\end{align*}
Plugging in for $V_{ab}$ from~\pef{eq:XOR-V-matrix}, and expanding this out, a calculation yields  
\begin{align*}
	\mathcal L R_{ab}^\perp = \sum_{c\in [k]} p_c \Big(\mathbb E\big[\|Z_\lambda\|^2  & (\pi_{Y_c}(a)-\mathbf{1}_{a=c})(\pi_{Y_c}(b)-\mathbf{1}_{b=c})\big]   \\ 
 & \quad - \langle \mathbb E[Z(\pi_{Y_c}(a) - \mathbf{1}_{a= c}), \mathbb E[Z (\pi_{Y_c}(b) - \mathbf 1_{b=c})]\rangle\Big) + O(1)\,,
\end{align*}
where to see that the extra terms are $O(1)$, we notice that the inner products of the $\mu$'s with each other and $\mu$'s with $Z$'s are all order $1$. 
By Cauchy--Schwarz, the inner product of the expectations is at most $O(\sqrt{d/\lambda})$ and will vanish when multiplied by $\delta  = O(1/d)$ and the $d\to\infty$ limit is taken;  on the other hand, the second moment term is of order $d/\lambda$. 
 In particular, if $\delta = c_\delta/d$, we get 
\begin{align*}
    \delta \mathcal L R_{ab}^\perp = c_\delta \sum_{c\in [k]} p_c \mathbb E\Big[\|d^{-1/2}Z_\lambda\|^2 (\pi_{Y_c}(a) - \mathbf 1_{a=c})(\pi_{Y_c}(b) - \mathbf{1}_{b=c})\Big] + o(1)\,.
\end{align*}
We claim that the $d\to\infty$ limit of this gives 
\begin{align}\label{eq:limiting-volatility}
	g_{R_{ab}^\perp} = \frac{c_\delta}{\lambda}\sum_{c\in [k]} p_c \mathbb E\Big[(\pi_{Y_c}(a) - \mathbf 1_{a=c})(\pi_{Y_c}(b)- \mathbf 1_{b=c})\Big]\,.
\end{align} 
Notice that the expectation is bounded by $1$ in absolute value and thus the above goes to $0$ as $\lambda \to \infty$. 
In order to show~\pef{eq:limiting-volatility}, let us consider the term coming from $\pi_{Y_c}(a)\pi_{Y_c}(b)$, the other terms being analogous or even easier since the indicator is deterministic. Using the fact that $\pi$ are probabilities and therefore bounded by $1$, note that 
\begin{align*}
    \big|\mathbb E[\|d^{-1/2} Z\|^2 \pi_{Y_c}(a)\pi_{Y_c}(b)] & - \frac{1}{\lambda} \mathbb E[\pi_{Y_c}(a) \pi_{Y_c}(b)] | \le \mathbb E\Big[\Big|\|d^{-1/2} Z\|^2 - \frac{1}{\lambda}\Big|\Big]\,,
\end{align*}
which goes to $0$ from the standard fact that for a standard Gaussian vector $G_d$ in $\mathbb R^d$, one has $\mathbb E\big[\big(\|\tfrac{G_d}{\sqrt{d}}\|^2 - 1\big)^2\big] = O(1/d)$.

The last thing to check is that in the ballistic regime where our summary statistics are not rescaled, the limiting dynamics are indeed an ODE, i.e., there is no limiting stochastic part. Towards that purpose, notice that in any ball of values for $(\mathbf{m},\mathbf{R})$, 
\begin{align*}
	\|J V J^T\| = O(1)\,,
\end{align*}
using an $O(1)$ bound on the operator norm of $V$, and noticing that $\|J\|^2$ is bounded by a constant plus the sums of $R_{ab}^\perp$. Therefore when multiplied by $\delta = o(1)$ this vanishes in the limit and therefore the diffusion matrix $\mathbf{\Sigma}$ is identically zero. 
\end{proof}

The following confines this $\lambda$-finite dynamical system to a compact set for all times. 
\begin{lem}\label{lem:kGMM-confined-to-compact-region}
    For every $\beta>0$, there exists $L(\beta)$ such that for all $\lambda$ large, the dynamics of Theorem~\ref{thm:k-GMM-ballistic-limit} stays inside the $\ell^2$-ball of radius $L$ for all time. 
\end{lem}

\begin{proof}
    Notice that a naive bound on $Q_{a}^{d,v}$ from~\pef{eq:P-Q-Gaussian-integrals} is at most $\sqrt{\|v\|/\lambda}$ since $\pi_{Y_c}(a)$ is bounded, and $P_{a}^d$ is always at most $1$. Plugging these bounds into Theorem~\ref{thm:k-GMM-ballistic-limit}, together with the definition of $\overline{m}_{ab}$ and the fact that $\|\mu_a\|=1$ for all $a$, yields the inequalities for all $a,b$, 
    \begin{align*}
        |\dot m_{ab}(t) + \beta m_{ab}| & \lesssim 1+\sqrt{1/\lambda}\,, \\
        |\dot{R}_{aa}^\perp   + \beta R_{aa}^\perp| & \lesssim \sqrt{R_{aa}^\perp/\lambda} + \lambda^{-1} + \lambda^{-2}\,.
    \end{align*}
    By these, for $\lambda$ larger than a fixed constant, we have $\dot{R}_{aa}^\perp \le 1 - \beta R_{aa}^\perp/2$ which Gronwall's inequality ensures will be bounded by $1+ R_{aa}^\perp(0)$ for all times $t\ge 0$. Similarly, we get that for $\lambda$ sufficiently large, $|m_{ab}(t)|$ is bounded by $3+m_{ab}(0)$. 
\end{proof}

\subsubsection{The zero-noise limit}
We now send $\lambda\to\infty$, or simply Taylor expand in the large $\lambda$ limit, to understand the behavior of the limiting ODE's we derived when $\lambda$ is large. Let 
\begin{align}\label{eq:pi-bar}
	\overline{\pi}_{c}(a) = \overline{\pi}_c(a;\mathbf{x}) := \frac{e^{m_{ac}}}{\sum_{b\in [k]} e^{m_{bc}}}\,.
\end{align}
This is the "$\lambda = \infty$" value of $\pi_{Y_c}(a)$. 
The aim of this subsection is to establish the following. 

\begin{prop}\label{prop:k-GMM-ballistic-zero-noise}
	The $\lambda \to\infty$ limit of the ODE system from Theorem~\ref{thm:k-GMM-ballistic-limit} is the following dynamical system: 
	\begin{align}
		\dot m_{ab}(t) & = p_a \overline{m}_{ab} - \beta m_{ab} - \sum_{c\in [k]} p_c {\overline{\pi}}_{c}(a) \overline{m}_{cb}\,,  \label{eq:kGMM-mdot-zero-noise}\\
		\dot{R}_{ab}^\perp(t)  &  = - \beta R_{ab}^\perp\,.  \label{eq:kGMM-Rdot-zero-noise}
	\end{align}
    Moreover, at large finite $\lambda$, the difference of the drifts in~\pef{eq:k-GMM-ballistic-limit-m}--\pef{eq:k-GMM-ballistic-limit-R} to the above drifts is $O(\lambda^{-1})$. 
\end{prop}

The main thing to prove is the following behavior of integrals of $\pi_{Y_c}(a)$ as $\lambda \to \infty$. 
\begin{lem}\label{lem:P-Q-integral-large-lambda}
	Recalling $P_a^c$ and $Q_{a}^{c,v}$ from~\pef{eq:P-Q-Gaussian-integrals}, we have  
	\begin{align*}
		P_{a}^c & = \overline{\pi}_c(a) + O(1/{\lambda})\,. \\ 
		Q_{a}^{c,v} & = O(\|v\|/{\lambda})\,.
		\end{align*}
\end{lem} 

\begin{proof}
By Taylor expanding, we can write 
\begin{align*}
    \pi_{Y_c}(a) =  \overline{\pi}_{c}(a) + \frac{(x^a \cdot Z) e^{ x^a\cdot\mu_c}} {\sum_{b\in [k]} e^{ x^b\cdot \mu_c}} -  \frac{e^{ x^a\cdot\mu_c}(\sum_{b\in [k]} (x^b \cdot Z) e^{ x^b \cdot \mu_c})} {(\sum_{b\in [k]} e^{ x^b\cdot \mu_c})^2} + O((\max_{b\in[k]} x^b \cdot Z)^2)\,.
\end{align*}
Taking an expectation of the right-hand side, noting that $x^b \cdot Z = \cN(0,\frac{R_{bb}}{\lambda})$, everything on the right-hand side after $\overline{\pi}_d(a)$ is $O(1/\lambda)$ for $R_{bb}^\perp$ that is $O(1)$. 
    For $Q_{a}^{c,v}$, using the Gaussian integration by parts formula and \pef{eq:grad-pi}
    \begin{align*}
        Q_{a}^{c,v} = \frac{1}{\lambda} \Big(\mathbb E[ (x^a\cdot v) \pi_{Y_c}(a) ] - \mathbb E[\langle x^B\cdot v \rangle_{\pi_{Y_c}} \pi_{Y_c}(a)] \Big)
    \end{align*}
    Since $\pi\leq 1$ and $ x^b\cdot v \le \sqrt{R_{bb}^\perp} \|v\|$, this is easily seen to be $O(\|v\| \sqrt{R_{bb}^\perp}/\lambda)$ as claimed. 
    \end{proof}

\begin{proof}[\textbf{\emph{Proof of Proposition~\ref{prop:k-GMM-ballistic-zero-noise}}}]
	For any $K$, uniformly over all $\mathbf{m},\mathbf{R}$ in a ball of radius $K$ about the origin, we claim that the limit of the drifts for each of those variables converge to the claimed $\lambda \to\infty$ limiting drifts. This is obtained by applying the above lemma to the $P$ and $Q$ terms in the drifts in Theorem~\ref{thm:k-GMM-ballistic-limit}, and finally the observation that $\text{Var}(\pi_{Y_c}(a)\pi_{Y_c}(b))\le 1$ so that taking $\lambda \to\infty$ the last two terms in the drift for $R_{ab}^\perp$ in Theorem~\ref{thm:k-GMM-ballistic-limit} also vanish. 
\end{proof}

The following gives a quantitative approximation of the ODE by its $\lambda = \infty$ limit. 

\begin{cor}\label{cor:k-GMM-lambda-finite-close-to-lambda-infinite}
    The trajectories of the ODEs of Theorem~\ref{thm:k-GMM-ballistic-limit} and Proposition~\ref{prop:k-GMM-ballistic-zero-noise} are within distance $O(t/\lambda)$ of one another. 
\end{cor}

\begin{proof}
    Let $\mathbf{u}$ and $\widetilde{\mathbf{u}}$ be the two solutions to the $\lambda$ finite and $\lambda = \infty$ ballistic dynamics respectively. Then, while $\|\mathbf{u}\|,\|\widetilde{\mathbf{u}}\|\le L$, we have by Proposition~\ref{prop:k-GMM-ballistic-zero-noise} that 
    \begin{align*}
        \|\dot{\mathbf{u}}- \dot{\widetilde{\mathbf{u}}}\|\lesssim_L \lambda^{-1}\,.
    \end{align*}
    Per Lemma~\ref{lem:kGMM-confined-to-compact-region}, both dynamics remain confined for a large enough $L(\beta)$ for all times, and therefore integrating the above gives the claim. 
\end{proof}

\subsection{Living in subspace spanned by the means}
We now wish to show that the SGD trajectory lives in the span of the means.
This can be done by showing that on the one hand, $\mathbf{R}^\perp$ will stay as small as we want after an $O(1)$ burn-in time, and on the other hand, towards the error being multiplicative in Definition~\ref{def:lives-in-span}, for every $a$, $\mathbf{x}^a_\ell$ needs to be non-negligible. 

We first establish that this happens for the $\lambda = \infty$ dynamics, then pull it back to the dynamics at $\lambda$ finite but large via Corollary~\ref{cor:k-GMM-lambda-finite-close-to-lambda-infinite}. 

\begin{lem}\label{lem:kGMM-ODE-lives-in-span}
    The solution to the dynamical system of Proposition~\ref{prop:k-GMM-ballistic-zero-noise} is such that for all $t\ge T_0(\epsilon)$, it is within distance $\epsilon$ of a point having $R_{aa}^\perp= 0$ and $\max_{b} |m_{ab}|>c_\beta >0$ for every $a\in [k]$.  
\end{lem}

\begin{proof}
    By the expression from Proposition~\ref{prop:k-GMM-ballistic-zero-noise} for the drift of $R_{aa}^\perp$ for $a\in [k]$, the dynamical system has $R_{aa}^\perp(t) = e^{ - \beta t} R_{aa}^\perp(0)$. In particular, for any $\epsilon>0$, the has $R_{aa}^\perp(t) <\epsilon$ for all $t\ge T_0(\epsilon)$. 
    
    We need to show for every $a$, in the solution of the dynamical system, some $(m_{ab})_b$ is bounded away from zero after some small time. Let $\cM_0 = \bigcup_{a} \bigcap_{c} \{m_{ac} =0\}$ be the set of $(\mathbf{m})$ values we would like to ensure the dynamics stays away from. First observe that the $\lambda =\infty$ dynamical system of Proposition~\ref{prop:k-GMM-ballistic-zero-noise} is a gradient system for the energy function 
    \begin{align*}
        \mathcal H(\mathbf{m},\mathbf{R}^\perp) = - \sum_{a\in [k]} p_a m_{aa} + \sum_{a\in [k]} p_a \log \sum_{b\in [k]} e^{m_{ab}} + \frac{\beta}{2}(\|\mathbf{m}\|^2 + \|\mathbf{R}^\perp\|^2)\,,
    \end{align*}
    so it has no recurrent orbits. It thus suffices to show that for every point in $\cM_0$, the quantity $\max_{b} |m_{ab}|$ has a drift strictly bounded away from zero. If we show that, then the dynamics is guaranteed to leave a $c_\beta$-neighborhood of the set $\cM_0$ in a finite time (uniform by continuity considerations and we are already guaranteed that the dynamics stays in a compact set by Lemma~\ref{lem:kGMM-confined-to-compact-region}. 

    Consider a point such that $m_{ab}=0$ for all $b\in [k]$. There, by Proposition~\ref{prop:k-GMM-ballistic-zero-noise}, 
    \begin{align*}
        \dot{m}_{ab} = \langle p_a \mu_a - \sum_{c\in [k]} p_c \overline{\pi}_c(a) \mu_c, \mu_b \rangle\,. 
    \end{align*}
    We need to show that the maximum over $b$, of the absolute values of these, is non-zero. Indeed, if it were $0$ for all $b\in [k]$, then $p_a \mu_a = \sum_c p_c \overline{\pi}_c(a) \mu_c$ because the difference of these vectors would be in $\text{Span}(\mu_1,...,\mu_k)$ while being orthogonal to all of $\mu_1,...,\mu_k$. In turn, however, this is impossible by our assumption that $\mu_1,...,\mu_k$ are linearly independent. Therefore, in the ball of radius $L(\beta)$ about the origin, for every $a$, at least one of the drifts $\dot{m}_{ab}$ is bounded away from zero uniformly.  
\end{proof}

\begin{proof}[\textbf{\emph{Proof of Proposition~\ref{prop:k-GMM-SGD-result}}}]
    We have shown that the limit dynamics for the summary statistics of the SGD initialized from $\cN(0,I_d/d)$ is the dynamical system of Theorem~\ref{thm:k-GMM-ballistic-limit} initialized from the deterministic initialization $m_{ab}\sim \delta_0$ and $R_{ab}^\perp \sim \delta_0$ if $a \ne b$ and $\delta_1$ if $a=b$.  

    By Lemma~\ref{lem:kGMM-confined-to-compact-region}, there exists $L(\beta)$ such that that dynamical system is confined to a ball of radius at most $L$ for all time. Since $\|\mathbf{x}^c\|^2$ is encoded by a smooth function of the summary statistics $\mathbf{R}_{cc}^\perp, (m_{cb})_b$ appearing in the dynamical system of Theorem~\ref{thm:k-GMM-ballistic-limit}, 
    this is transferred to the SGD $\mathbf{x}_\ell$ via Lemma~\ref{lem:closeness-to-ballistic-trajectory} for a different constant $L(\beta)$ for $d$ sufficiently large. 

    For the second part, by Lemma~\ref{lem:kGMM-ODE-lives-in-span}, the solution to the dynamical system of Theorem~{thm:k-GMM-ballistic-limit} is at distance $O(\lambda^{-1})$ from the solution of the dynamical system (with the same initialization) of Proposition~\ref{prop:k-GMM-ballistic-zero-noise}, for which Lemma~\ref{lem:kGMM-ODE-lives-in-span} applies.  By Lemma~\ref{lem:closeness-to-ballistic-trajectory}, these get pulled back to the summary statistics applied to the SGD itself (at $d$ sufficiently large depending on $\epsilon,\lambda$). We therefore deduce that for all $\ell \ge T_0(\epsilon) \delta^{-1}$ steps, the SGD has $R_{cc}^\perp(\mathbf{x}_\ell) \le O(\epsilon+ \lambda^{-1})$ and has $\max_{b} |\langle \mathbf{x}^c_\ell,\mu_b\rangle|> c_\beta>0$. These imply the claim using that $\|\mathbf{x}^c \|\ge \max_b |\langle \mathbf{x}^c ,\mu_b\rangle|$ and that $R_{cc}^\perp(\mathbf{x}^c)$ is by definition the projection of $\mathbf{x}^c$ orthogonal to $(\mu_1,...,\mu_k)$. 
\end{proof}

\subsubsection{Fixed point analysis with orthogonal means}\label{subsubsec:kGMM-orthogonal-means}
The behavior of the ODE system of Proposition~\ref{prop:k-GMM-ballistic-zero-noise} can be sensitive to the relative location of the means $(\mu_1,...,\mu_k)$. In order to be able to make more precise statements about the alignment of the SGD with specific eigenvectors rather than just living in $\text{Span}(\mu_1,...,\mu_k)$, we specialize to the case where the means form an orthonormal family of vectors. Here, we can explicitly characterize the fixed points of Proposition~\ref{prop:k-GMM-ballistic-zero-noise}.  

Namely, in this subsection, assume that $
\overline{m}_{ab} = \mathbf 1_{a=b}$. Then, by~\pef{eq:kGMM-mdot-zero-noise} any fixed point must satisfy 
\begin{align}\label{eq:kGMM-orthogonal-means-fixed-point}
	\beta m_{ab} = p_a \mathbf{1}_{a=b}  - p_b \frac{e^{m_{ab}}}{\sum_{c}e^{m_{cb}}}\,.
\end{align}
At a fixed point, the function $\sum_{c} m_{cb}$ thus must equal $0$. Also, if $a,c\ne b$ then 
	\begin{align*}
		m_{ab} - m_{cb} =  - \frac{p_b}{\beta}  \cdot \frac{e^{m_{ab}} - e^{m_{cb}}}{\sum_{d} e^{m_{db}}}\,.
	\end{align*}
	Since the function $x+ce^x$ is strictly increasing, the only solutions to this are at $m_{ab} = m_{cb}$ (so long as $p_b>0$). Combining the above two observations, at a fixed point, 
 \begin{align*}
     m_{ab} = - \frac{1}{k-1} m_{bb} \quad \text{ for all $a\ne b$}\,.
 \end{align*}
	Plugging this in to~\pef{eq:kGMM-orthogonal-means-fixed-point}, we find that at a fixed point 
	\begin{align*}
		m_{bb} = \frac{p_b}{\beta} \Big(1- \frac{e^{m_{bb}}}{(k-1)e^{ - \frac{1}{k-1}m_{bb}} + e^{m_{bb}}}\Big)\,.
	\end{align*}
	This can easily be seen to have a unique solution, and that solution must have $m_{bb}\in (0,\frac{p_b}{\beta})$. 
	
	Therefore, as long as $(p_b)_{b\in [k]}$ are all positive, the dynamical system of Proposition~\ref{prop:k-GMM-ballistic-zero-noise} has a unique fixed point, call it $\mathbf{u}_\star$ at $(m_{ab})_{a,b}$ as above, and $R_{ab}^\perp = 0$ for all $a,b$. As observed in the proof of Lemma~\ref{lem:kGMM-ODE-lives-in-span}, the dynamical system never leaves a ball of some radius $L(\beta)$ and is also a gradient system for an energy function $\mathcal H$. Combining these facts with continuity of the drift functions, it means that for every $\epsilon>0$, there is a $T_0(\epsilon)$ such that the solution to the SGD gets within distance $T_0$ of that unique fixed point and stays there for all $t\ge T_0$. 

 Altogether, this leaves us with the following stronger form of Proposition~\ref{prop:k-GMM-SGD-result}. 

 \begin{prop}\label{prop:k-GMM-SGD-result-orthonormal-means}
    When the means $(\mu_1,...,\mu_k)$ are orthonormal, beyond Proposition~\ref{prop:k-GMM-SGD-result}, we further have that $\mathbf{x}$ is at distance $O(\epsilon + \lambda^{-1})$ of a point $\mathbf{x}_\star$ such that for each $c$, $\mathbf{x}_\star^c$ has positive (bounded away from zero uniformly in $\epsilon,\lambda$) inner product with $\mu_c$ and negative inner product with $(\mu_a)_{a\ne c}$. 

    If $(p_c)_{c\in [k]}$ are assumed to all be equal, then furthermore at $\mathbf{x}_\star$,
    \begin{align*}
        \bar\pi_{b}(a) =  \frac{1}{k-1}(1-\bar\pi_{c}(c)) \quad \text{for all $a,b,c$ such that $a\ne b$}\,.
    \end{align*}
\end{prop}

\subsection{2-layer networks}
We now turn to the analysis of the SGD in the case of multilayer networks for the XOR GMM problem~\pef{eq:XOR-data-distribution}. In this problem, the input dimension is $d=n$, the parameter dimension is $p = K d + K$, and the step-size is again taken to be $\delta = O(1/d)$. 
Consider the following family of $4K + \binom{K}{2}$ summary statistics of $\mathbf{x}$: for $1\le i\le j\le K$ and $\vartheta \in \{\mu,\nu\}$,
\begin{align}\label{eq:XOR-GMM-summary-stats}
v_{i}\,,\qquad m_{i}^{\vartheta}  =W_{i}\cdot\vartheta\,, \qquad
R_{ij}^{\perp} =W_{i}^{\perp}\cdot W_{j}^{\perp}\,,
\end{align}
where $W_i^\perp  = W_i - \sum_{\vartheta\in \{\mu,\nu\}} m_i^\vartheta \vartheta$. Use $\mathbf{u} = (\mathbf{v},\mathbf{m}^\mu,\mathbf{m}^\nu,\mathbf{R}^\perp)$ for these families.

In~\citet{BGJ22} it was shown that this family of summary statistics is $\delta$-localizable and asymptotically closable with respect to the loss for the XOR GMM of \pef{eq:XOR-loss}, and the following convergence to ODE's was established.   
For a point $\mathbf{x}= (v,W) \in \mathbb R^{K+Kd}$, define the quantity 
\begin{align}\label{eq:XOR-A-function}
	\mathbf{A}_i = \mathbb E\big[Y\mathbf 1_{W_i \cdot Y\ge 0} \big(-y+\sigma(v\cdot g(WY))\big)\big]\,,
\end{align}
(where we recall that $g$ is the ReLU function and $\sigma$ is the sigmoid function) and let 
\begin{align*}
	\mathbf{A}_i^\vartheta  = \vartheta \cdot \mathbf{A}_i\,, \qquad \mathbf{A}_{ij}^\perp = W_j^\perp \cdot \mathbf{A}_i\,.
\end{align*}
Furthermore, let 
\begin{align}\label{eq:XOR-B-function}
	\mathbf{B}_{ij} = \mathbb E\big[\mathbf 1_{W_i \cdot Y\ge 0} \mathbf 1_{W_j\cdot Y\ge 0} \big(-y+\sigma(v\cdot g(WY))\big)^2\big]\,.
\end{align}
It can be observed that these functions are expressible as functions of $\bu$ alone. 

Proposition 5.1 of~\citet{BGJ22} proved the following effective ballistic dynamics.
\begin{prop}\label{prop:XOR-effective-dynamics-finite-lambda}
	Let $\bu_n$ be as in~\pef{eq:XOR-GMM-summary-stats} and fix any $\lambda>0$ and $\delta = c_\delta/d$. Then $\bu_n(t)$ converges weakly to the solution of the ODE system $\dot{\bu}_t = -\mathbf{f}(\bu_t) + \mathbf{g}(\bu_t)$, initialized from $\lim_n (\bu_n)_* \mu_n$ with 
	\begin{align*}
		f_{v_i} &  = m_i^\mu \mathbf A_i^\mu (\mathbf u) + m_i^\nu \mathbf{A}_i^\nu (\bu) + \mathbf{A}_{ii}^\perp (\bu) + \beta v_i\,, &  f_{m_i^\mu} & = v_i \mathbf{A}_i^\mu + \beta m_i^\mu\,, \\
		f_{R_{ij}^{\perp}}  & = v_i \mathbf{A}_{ij}^\perp(\bu) +v_j \mathbf{A}_{ji}^\perp(\bu) + 2\beta R_{ij}^\perp\,, & f_{m_i^\nu} & = v_i \mathbf{A}_i^\nu + \beta m_i^\nu\,.
	\end{align*}
	and correctors $g_{v_i} = g_{m_i^\mu} = g_{m_i^\nu} = 0$, and $g_{R_{ij}^\perp} = c_\delta \frac{v_i v_j}{\lambda} \mathbf B_{ij}$ for $1\le i\le j\le K$. The case where $\delta  = o(1/d)$ is read-off by formally setting $c_\delta = 0$. 
\end{prop}

\subsubsection{Large $\lambda$ behavior}
We now wish to investigate the large $\lambda$ behavior of the dynamical system in Proposition~\ref{prop:XOR-effective-dynamics-finite-lambda}. Our approach to doing this is to give a large $\lambda$ approximation to the drifts in the above, and then use that to show that the trajectory is close to its $\lambda = \infty$ version. 

The first thing we do is give large $\lambda$ approximations to the quantities $\mathbf{A}_i$ and $\mathbf{B}_{ij}$. In what follows, we use $F(x)$ to denote the cumulative distribution function of a standard Gaussian random variable. 

\begin{lem}\label{lem:large-lambda-expansion-expectations}
    Suppose $Y\sim \vartheta+ Z_\lambda$ for a fixed vector $\vartheta\in \{\mu,-\mu,\nu,-\nu\}$ and fix a unit vector $b\in \mathbb R^d$ and a vector $v\in \mathbb R^K$. Then 
    \begin{align*}
        \Big|\mathbb E[(b\cdot Y) \mathbf{1}_{W_i \cdot Y >0} \sigma( v\cdot g(WY))] & - (b\cdot \vartheta) F\Big(m_i^\vartheta\sqrt{\frac{\lambda}{R_{ii}}}\Big)  \sigma(v\cdot g(m^\vartheta))\Big| \\ 
        & \le  \max_{j} (1+v_j)(m_i^\vartheta + \sqrt{\frac{R_{ii}}{\lambda}}) e^{ - (m_j^\vartheta)^2 \lambda/(16 R_{jj})} + O(\lambda^{-1})\,.
    \end{align*}
\end{lem}
\begin{proof}
    Let us start with a Taylor expansion of $\sigma$: For simplicity, let us use $\sigma_Y$ and $\sigma_\vartheta$ to denote $\sigma(v\cdot g(WY))$ and $\sigma(v\cdot g(W \vartheta))$ respectively and similarly for $\sigma'$. We then have 
    \begin{align}\label{eq:sigmoid-Taylor-expansion}
        \sigma_Y - \sigma_\vartheta = \sigma'_\vartheta \cdot(v\cdot (g(W\vartheta) - g(WY))) + \sigma''(o) (v\cdot (g(W\vartheta) - g(WY)))^2\,,
    \end{align}
    for some point $o$ between $v\cdot g(WY)$ and $v\cdot g(W\vartheta)$. Therefore, the left-hand side of the lemma is
    \begin{align}\label{eq:E-decomposition}
        & \mathbb E\Big[\mathbf 1_{W_i \cdot Y>0} (b\cdot \vartheta) \sigma_\vartheta \Big]  + \mathbb E\Big[\mathbf 1_{W_i \cdot Y>0}(b\cdot Z) \sigma_\vartheta\Big]   + \mathbb E\Big[\mathbf{1}_{W_i \cdot Y>0} (b\cdot Y)\sigma'_\vartheta (v\cdot (g(W\vartheta) - g(WY)))\Big]  \nonumber \\
        & \qquad + \mathbb E\Big[\mathbf{1}_{W_i \cdot Y>0} (b\cdot Y) \sigma''(o) (v\cdot (g(W \vartheta) - g(WY)))^2\Big)\Big]\,.
    \end{align}
    The first term of~\pef{eq:E-decomposition} is exactly the term we are comparing to in the left-hand side of the lemma statement.     
    Since $\mathbb E[(b\cdot Z)]=0$, the absolute value of the second term of~\pef{eq:E-decomposition} is bounded via Cauchy--Schwarz as follows: 
    \begin{align*}
        | \mathbb E\Big[\mathbf 1_{W_i \cdot Y>0} (b\cdot Z) \sigma_\vartheta\Big]|= |\mathbb E\Big[(\mathbf 1_{W_i \cdot Y>0} - \mathbf{1}_{m_i^\vartheta>0}) (b\cdot Z) \sigma_\vartheta\Big]| \le \frac{\sigma_\vartheta}{\sqrt{\lambda}} F(- |m_{i}^\vartheta|\sqrt{\lambda/R_{ii}})^{1/2}
    \end{align*}
    For the next two terms we can start by rewriting  
    \begin{align}\label{eq:g-difference}
        g(W_j \cdot \vartheta) - g(W_j \cdot Y) = (m_j^\vartheta {+} W_j \cdot Z)\mathbf 1_{W_j \cdot Y<0,m_j^\vartheta>0} -(W_j \cdot Y)\mathbf 1_{W_j\cdot Y>0,m_j^\vartheta <0} {-} (W_j \cdot Z)\mathbf 1_{m_j^\vartheta>0}\,.
    \end{align}
    Using this, the third expectation in~\pef{eq:E-decomposition} is 
    \begin{align*}
        \sigma'_\vartheta \mathbb E\Big[(b\cdot Y)\mathbf 1_{W_i\cdot Y>0}\sum_j v_j \Big((m_j^\vartheta {+} W_j \cdot Z)\mathbf 1_{W_j \cdot Y<0,m_j^\vartheta>0} -(W_j \cdot Y)\mathbf 1_{W_j\cdot Y>0,m_j^\vartheta <0} {-}  (W_j \cdot Z)\mathbf 1_{m_j^\vartheta>0} \Big)\Big]\,.
    \end{align*}
    The first and second terms in the parentheticals are such that the Cauchy--Schwarz inequality can be applied to bound their total contributions by  
    \begin{align*}
        K \sigma'_a \|v\|_\infty  \max_j \big((m_j^\vartheta)^2 + \frac{R_{jj}}{\lambda})^{1/2} F(-|m_j^\vartheta|\sqrt{\lambda/R_{jj}})^{1/2}\,.
    \end{align*}
    The last term will contribute {(up to a net sign)} 
    \begin{align*}
        \sigma_\vartheta' v_j \mathbb E[(b\cdot \vartheta) (W_j \cdot Z)(\mathbf 1_{W_i \cdot Y>0}- \mathbf 1_{m_i^\vartheta>0}) \mathbf 1_{m_j^\vartheta>0}] + \sigma_\vartheta' v_j \mathbb E[(b\cdot Z) (W_j \cdot Z)\mathbf 1_{W_i \cdot Y>0} \mathbf 1_{m_j^\vartheta>0}]\,.
    \end{align*}
    The absolute value of the first term is bounded similarly to an earlier one by Cauchy--Schwarz. The absolute value of the second term is bounded by dropping the indicator functions and applying Cauchy--Schwarz to see that it is $O(1/\lambda)$ so long as $\|b\|,\|W_j\|=O(1)$. 

    Finally, for the fourth term of~\pef{eq:E-decomposition}, the square allows us to put absolute values on every term, and immediately apply the Cauchy--Schwarz inequality, to bound its absolute value by 
    \begin{align*}
        \|\sigma''\|_\infty \|v\|_\infty \mathbb E[(b\cdot Y)^2]^{1/2} \max_j \mathbb E[(g(W_j \cdot \vartheta) - g(W_j \cdot Y))^4]^{1/2}\,.
    \end{align*}
    The fourth moment in the above expression is bounded, up to constant, by the fourth moments of each of the individual terms in~\pef{eq:g-difference}. The first two of those will be at most some constant times $F(- |m_j^a|\sqrt{\lambda/R_{jj}})^{1/4}$. For the last of them, we use $\mathbb E[(W_j \cdot Z)^4]^{1/2} \le O(1/\lambda)$.

    Combining all of the above bounds, and naively bounding the cdf of the Gaussian via 
    \begin{align*}
        F(-|m_j^\vartheta|\sqrt{\lambda/R_{jj}})^{1/4} \le e^{ - (m_j^\vartheta)^2 \lambda/(16 R_{jj})}
    \end{align*}
    we arrive at the claimed bound. 
\end{proof}

As a consequence of Lemma~\ref{lem:large-lambda-expansion-expectations}, we can deduce that the quantities appearing above satisfy the following large $\lambda$ behavior. 

\begin{cor}\label{cor:large-lambda-approixmation}
For every $\mathbf{x}$ such that $m_i^\mu,m_i^\nu \ge \frac{\log \lambda}{\sqrt{\lambda}}$ for all $i$,  
\begin{align*}
	m_i^\mu \mathbf{A}_i^\mu & =- \frac{1}{4} g(m_i^\mu) \sigma(-v\cdot g(m^\mu)) - \frac{1}{4} g(-m_i^\mu) \sigma(-v\cdot g(-m^\mu)) + O(\lambda^{-1}) \\ 
	m_i^\nu \mathbf{A}_i^\nu & = \frac{1}{4} g(m_i^\nu) \sigma(v\cdot g(m^\nu)) + \frac{1}{4} g(-m_i^\nu) \sigma(v\cdot g(-m^\nu)) + O(\lambda^{-1}) \\
	\mathbf{A}_{ij}^\perp & = O(\lambda^{-1})\,.
\end{align*}
Without the assumption on $m_i^\mu, m_i^\nu$, the same holds with $O(\lambda^{-1})$ replaced by $O(\lambda^{-1/2})$, as long as the indicators from $g(m_i^\vartheta) = m_i^\vartheta \mathbf{1}_{m_i^\vartheta>0}$ are replaced by the ``soft indicators" $F(m_i^\vartheta \sqrt{\lambda/R_{ii}})$.
\end{cor}

\begin{proof}
    We begin with the estimate on $m_i^\mu \mathbf{A}_i^\mu$. This quantity can be split into four terms corresponding to whether the mean chosen for $Y$ is $\mu,-\mu,\nu,-\nu$. Namely, if we let $Y_a \stackrel{d}= a+ Z_\lambda$, 
    \begin{align*}
        \mathbf{A}_i^\mu & = \frac{1}{4}\Big( \mathbb E[(\mu\cdot Y_\mu) \mathbf{1}_{W_i\cdot Y_\mu \ge 0} \sigma(- v\cdot g(W Y_\mu)] +  \mathbb E[(\mu\cdot Y_{-\mu}) \mathbf{1}_{W_i\cdot Y_{-\mu} \ge 0} \sigma(- v\cdot g(W Y_{-\mu})] \\ 
        & \qquad + \mathbb E[(\mu\cdot Y_\nu) \mathbf{1}_{W_i\cdot Y_\nu \ge 0} \sigma(v\cdot g(W Y_\nu)] +  \mathbb E[((\mu\cdot Y_{-\nu})\mathbf{1}_{W_i\cdot Y_{-\nu} \ge 0} \sigma(v\cdot g(W Y_{-\nu})]\Big)
    \end{align*}
    Now notice that each of the four quantities are of the form of Lemma~\ref{lem:large-lambda-expansion-expectations}, with $b = \mu$, and $\vartheta = \mu, - \mu,\nu,-\nu$ respectively (the change of the sigmoid possibly having a negative sign on its argument is realized by switching the sign of $v$ since that is the only place it appears). 

    It is easily seen that so long as  $|m_i^\mu| \ge (\log \lambda)/\sqrt{\lambda}$, then 
    \begin{align*}
        F(m_i^\mu \sqrt{\lambda/R_{ii}}) = \mathbf 1\{m_i^\mu>0\} + O(1/\lambda)\,.
    \end{align*}
    By a similar bound on the error term on the right of Lemma~\ref{lem:large-lambda-expansion-expectations}, 
    \begin{align*}
        \mathbb E[(\mu\cdot Y_\mu) \mathbf{1}_{W_i\cdot Y_\mu \ge 0} \sigma(- v\cdot g(W Y_\mu)] =  \mathbf 1_{m_i^\mu>0} \sigma(- v\cdot g(m^\mu)) + O(1/\lambda)\,.
    \end{align*}
    and the cases where it is $Y_\nu$ contribute $0+ O(\lambda^{-1})$ since $\mu \cdot \nu = 0$. Together with the analogous bounds for $m_i^\nu \mathbf{A}_i^\nu$ and $\mathbf{A}_{ij}^\perp$, this gives the first part of the corollary. 

    If we drop the assumption on $m_i^\mu$, we find from the general inequality $x e^{-x^2 L} \le \frac{C}{\sqrt{L}}$ for some uniform $C$, that the errors in the right-hand side of Lemma~\ref{lem:large-lambda-expansion-expectations} become at most $O(1/\sqrt{\lambda})$ as claimed.  Leaving the soft indicator in place, this gives the second part of the corollary. 
\end{proof}

We deduce from the above approximation and a trivial bound of $O(1/\lambda)$ on $\mathbf{B}_{ij}$, a $\lambda = \infty$ limiting dynamics. However, one has to be slightly careful about the $\lambda \to\infty$ limit near the hyperplanes where $m_i^\mu= 0$ or $m_i^\nu=0$; Therefore, in what follows we envision a different limit associated to each orthant of the parameter space (associated to the signs of $m_i^\mu,m_i^\nu$). This is reasonable because in each orthant, the $\lambda = \infty$ dynamical system initialized from that orthant stays in that orthant. The following proposition captures that limiting dynamics (all orthants being written into the below simultaneously, with ambiguities at $m_i^\mu=0$ being therefore omitted). 

\begin{prop}\label{prop:xor-ballistic-ode}
In the $\lambda\to\infty$ limit, the ODE  from Proposition~\ref{prop:XOR-effective-dynamics-finite-lambda} converges to 
\begin{align*}
\dot{v}_{i}& =  \frac{1}{4}\Big(g(m_i^\mu)\sigma(-v\cdot g(m^{\mu})) + g(-m_i^\mu) \sigma(-v\cdot g(-m^{\mu}))\Big) \\ & \qquad -\frac{1}{4}\Big(g(m_i^\nu)\sigma(v\cdot g(m^{\nu}))+ g(-m_i^\nu) \sigma(v\cdot g(-m^{\nu}))\Big)-\beta v_{i}\,,\\
\dot{m}_{i}^{\mu} & =  \frac{v_{i}}{4}\Big(\mathbf{1}_{m_{i}^{\mu}> 0}\sigma(-v\cdot g(m^{\mu}))-\mathbf{1}_{m_{i}^{\mu}<0}\sigma(-v\cdot g(-m^{\mu}))\Big)-\beta m_{i}^{\mu}\,,\\
\dot{m}_{i}^{\nu} & =  -\frac{v_{i}}{4}\Big(\mathbf{1}_{m_{i}^{\nu} > 0}\sigma(-v\cdot g(m^{\nu}))-\mathbf{1}_{m_{i}^{\nu}<0}\sigma(-v\cdot g(-m^{\nu}))\Big)-\beta m_{i}^{\nu}\,,
\end{align*}
and $\dot{R}_{ij}^{\perp} =  -2 \beta R_{ij}^{\perp}$ 
for $1\le i\le j\le K$. 

Moreover, for $\lambda$ large, the difference in the drifts above and those of Proposition~\ref{prop:XOR-effective-dynamics-finite-lambda} are of order $O(1/\lambda)$ if $(|m_i^\mu|,|m_i^\nu|)$ are all at least $\sqrt{(\log \lambda)/\lambda}$, and $O(1/\sqrt{\lambda})$ everywhere. 
\end{prop}

\subsubsection{Confining the SGD to a compact set} 
Similar to the $k$-GMM case, our first aim is to confine the SGD to a compact set so long as $\beta>0$. 

\begin{lem}\label{lem:xor-confined-to-compact-region}
    For every $\beta>0$, there exists $L(\beta, v(\mathbf{x}_0))$ such that for all $\lambda$ large, the dynamics of Proposition~\ref{prop:XOR-effective-dynamics-finite-lambda} stays inside the $\ell^2$-ball of radius $L$ for all time. 
\end{lem}

\begin{proof}
We start by considering the evolution of the $\ell^2$-norm of all the parameters, $(v,m^\mu,m^\nu)$. 
If we use the shorthand 
\begin{align*}
	g^\mu = g(m^\mu), \quad  g^{-\mu} = g(-m^\mu),\qquad \sigma^\mu = \sigma(-v\cdot g^\mu), \quad\sigma^{-\mu} = \sigma(-v\cdot g^{-\mu})
\end{align*}
and similar quantities with $\nu$ instead of $\mu$, a short 
 calculation shows that we get in the $\lambda = \infty$ dynamical system of Proposition~\ref{prop:xor-ballistic-ode}
\begin{align*}
	\dot {\|\mathbf{v}\|}^2   = \frac{1}{2} \big(v\cdot g^\mu \sigma^\mu + v\cdot g^{-\mu} \sigma^{-\mu}\big) - \frac{1}{2} \big( v\cdot g^\nu \sigma^\nu + v\cdot g^{-\nu} \sigma^{-\nu}\big) - 2 \beta \|\mathbf{v}\|^2\,.
\end{align*}
 Notice that $\sigma^\mu$ goes to $0$ as $v\cdot g^\mu \to \infty$, and moreover, $x\sigma(-x) \to 0$ as $x\to\infty$. Therefore, there is a uniform bound of $1$ (in fact, $W(1/e)$ where $W$ is the product log function) on $v\cdot g^\mu \sigma^\mu$. Similar uniform bounds apply to the other three terms, so that in total, 
 \begin{align*}
 \dot {\|\mathbf{v}\|}^2 \le 2- 2\beta \|\mathbf{v}\|^2\,,
\end{align*}
which implies in particular that its drift is negative when $\|\mathbf{v}\|^2$ is at least $L(\beta)$. Similar arguments go through mutatis mutandis for $(m^\mu)^2$ and $(m^\nu)^2$. These then imply that the drift when $\lambda$ is finite are similarly negative when $\|\mathbf{v}\|^2 = L(\beta)$ as long as $\lambda$ is sufficiently large per the approximation from Proposition~\ref{prop:xor-ballistic-ode}. Altogether, this implies that for the dynamical system of Proposition~\ref{prop:XOR-effective-dynamics-finite-lambda}, all of $\|\mathbf{v}\|,\|\mathbf{m}^\mu\|$ and $\|\mathbf{m}^\nu\|$, stay bounded by some $L(\beta)$ for all time.

Using that, if we consider the expression for $\dot R_{ii}^\perp$ and use the boundedness of $\mathbf{A}_i, \mathbf{B}$, the drift in Proposition~\ref{prop:XOR-effective-dynamics-finite-lambda}, gives 
\begin{align*}
    \|\mathbf{R}^\perp\|^2 \le - \beta \|\mathbf{R}^\perp\|^2 + 4  L \|\mathbf{R}^\perp\|\,,
\end{align*}
The right-hand side is at most $C_{L}  - \frac{1}{2} \beta \|\mathbf{R}^{\perp}\|^2$ for some $C_L$, whence by Gronwall, $\|\mathbf{R}^\perp\|^2$ also stays bounded by some constant $L'(\beta)$ for all time under the dynamics of Proposition~\ref{prop:XOR-effective-dynamics-finite-lambda}. Altogether these prove the lemma.  
\end{proof}

\begin{cor}\label{cor:lambda-finite-close-to-lambda-infinite}
    The trajectories of the ODEs of Proposition~\ref{prop:XOR-effective-dynamics-finite-lambda} and Proposition~\ref{prop:xor-ballistic-ode} are within distance $O(e^{ Ct}/\sqrt{\lambda})$ of one another. 
\end{cor}

\begin{proof}
    Let $\mathbf{u}$ and $\widetilde{\mathbf{u}}$ be the two solutions to the $\lambda$ finite and $\lambda =\infty$ ballistic limits respectively. Then 
    \begin{align*}
        \|\dot{\mathbf{u}} - \dot{\widetilde{\mathbf{u}}}\| \lesssim_K  (1+L + 2\beta)\|\mathbf{u} - \dot{\widetilde{\mathbf{u}}}\| + O(\lambda^{-1/2})\,,
    \end{align*}
    where we used that $L$ bounds the norm of $\mathbf{u}$ and $\widetilde{\mathbf{u}}$ for all times by Lemma~\ref{lem:xor-confined-to-compact-region}, and that the Lipschitz constant of the sigmoid is $1$. By Gronwall's inequality, this implies that $|\mathbf{u} - \widetilde{\mathbf{u}}|\le O(e^{Ct}/\sqrt{\lambda})$ as claimed for some $C$ depending only on $\beta$. 
\end{proof}

\subsubsection{Living in the principal directions near fixed points}\label{s:fixpoint}
The last thing to conclude is that the fixed points of the $\lambda = \infty$ dynamical system are indeed living in the principal directions as desired by Theorem~\ref{mainthm:XOR-all-live-in-subspace}. By the above arguments, for all $t\ge T_0(\epsilon)$ the dynamics is within distance $\epsilon$ of one of the fixed points of Proposition~\ref{prop:xor-ballistic-ode}. Let us recall the exact locations of those fixed points from~\citet{BGJ22}. 

If $0<\beta<1/8$, then let $(I_{0},I_{\mu}^{+},I_{\mu}^{-},I_{\nu}^{+},I_{\nu}^{-})$
be any disjoint (possibly empty) subsets whose union is $\{1,...,K\}$.
Corresponding to that tuple $(I_{0},I_{\mu}^{+},I_{\mu}^{-},I_{\nu}^{+},I_{\nu}^{-})$,
is a set of fixed points that have $R_{ij}^{\perp}=0$ for all
$i,j$, and have 
\begin{enumerate}
\item $m_{i}^{\mu}=m_{i}^{\nu}=v_{i}=0$ for $i\in I_{0}$,
\item $m_{i}^{\mu}=v_{i}>0$ such that $\sum_{i\in I_{\mu}^{+}}v_{i}^{2}=\mbox{logit}(-4\beta)$
and $m_{i}^{\nu}=0$ for all $i\in I_{\mu}^{+}$,
\item $-m_{i}^{\mu}=v_{i}>0$ such that $\sum_{i\in I_{\mu}^{-}}v_{i}^{2}=\mbox{logit}(-4\beta)$
and $m_{i}^{\nu}=0$ for all $i\in I_{\mu}^{-}$,
\item $m_{i}^{\nu}=v_{i}<0$ such that $\sum_{i\in I_{\nu}^{+}}v_{i}^{2}=\mbox{logit}(-4\beta)$
and $m_{i}^{\mu}=0$ for all $i\in I_{\nu}^{+}$,
\item $-m_{i}^{\nu}=v_{i}<0$ such that $\sum_{i\in I_{\nu}^{-}}v_{i}^{2}=\mbox{logit}(-4\beta)$
and $m_{i}^{\mu}=0$ for all $i\in I_{\nu}^{-}$.
\end{enumerate}

The following observation is easy to see from the fixed point characterization described above. 
\begin{obs}\label{obs:xor-fixed-points-in-principal-directions}
	Suppose that $x_\star$ is a fixed point amongst the above. Then  
	\begin{itemize}
		\item $W(x_\star)\in\text{Span}(\mu,\nu)$\,,
		\item $v(x_\star)\in\text{Span}(g^\mu(x_\star), g^{-\mu}(x_\star), g^{\nu}(x_\star), g^{-\nu}(x_\star))$\,,
	\end{itemize}  
	(with no error). 
 \end{obs}

\begin{proof}
	The first claim that $W(x_\star)\in\text{Span}(\mu,\nu)$ follows from the fact that $R_{ii}^\perp =0$ for all $i$.  
	Furthermore, if $I_\nu^+,I_\nu^-$ are empty, then it lives in $\text{Span}(\mu)$, and similarly if $I_\mu^+,I_\mu^-$ are empty, then it lives in $\text{Span}(\nu)$. 
	
	For the next claim, observe that $g^\mu(x_\star)$ is the vector that is $m^\mu$ in coordinates belonging to $I_\mu^+$, and $0$ in all other coordinates. $g^{-\mu}(x_\star)$ is the vector that is $-m^\mu$ in coordinates belonging to $I_{-\mu}^+$, and zero in others.    $g^{\nu}(x_\star)$ is $m_i^\nu$ for coordinates in $I_\nu^-$, zero else, and $g^{-\nu}(x_\star)$ is $-m_i^\nu$ on $I_\nu^+$. 
	
	Since the $I$-sets are a partition of $\{1,...,K\}$ it is evident that we can express $$v(x_*) = g^\mu(x_\star) + g^{-\mu}(x_\star) - g^\nu(x_\star)- g^{-\nu}(x_\star)\,,$$
	implying that indeed $v(x_\star)$ lives in $\text{Span}(g^\mu(x_\star), g^{-\mu}(x_\star), g^{\nu}(x_\star), g^{-\nu}(x_\star))$ 
\end{proof}

It was further argued in~\citet[Section 9.4]{BGJ22} that there is a transition at $\beta=1/8$ in the regularization, and as long as $\beta<1/8$, with probability $1$ with respect to the random initialization, the SGD converges to a fixed point having $I_0 = \emptyset$. 

\begin{lem}\label{lem:fixed-point-chosen-by-lambda-infinity-dynamics}
    Suppose $\beta<1/8$ and consider the initialization
    \begin{align}\label{eq:limiting-initialization}
    v_i(0)\sim \cN(0,1)\,, \qquad\text{and}\qquad m_i^\mu(0), m_i^\nu(0), R_{ij}^\perp(0) \sim \delta_0 \quad \text{and} \quad R_{ii}^\perp \sim \delta_1\,,
    \end{align}
    where since the dynamical system of Proposition~\ref{prop:k-GMM-ballistic-zero-noise} is defined per orthant, the $\delta_0$'s are understood as $1/2$-$1/2$ mixtures of $\delta_{0^+}$ and $\delta_{0^-}$. With probability $1$ over this initialization, the dynamical system of Proposition~\ref{prop:k-GMM-ballistic-zero-noise} converges to a fixed point as characterized above, above having $I_0 = \emptyset$. 
\end{lem}

Putting the above together, we can conclude our proof of Proposition~\ref{prop:XOR-GMM-SGD-result}. 

\begin{proof}[\textbf{\emph{Proof of Proposition~\ref{prop:XOR-GMM-SGD-result}}}]
    Per Proposition~\ref{prop:XOR-effective-dynamics-finite-lambda}, the limit of the summary statistics of the SGD along training is the solution of that dynamical system initialized from $m_i^\mu, m_i^\nu, R_{ij}^\perp \sim \delta_0$ for $i\ne j$ (with equal probability of $m_i^\mu,m_i^\nu$ being $\delta_{0^+}$ and $\delta_{0^-}$ if we need to distinguish which orthant it is initialized in), $R_{ii}^\perp \sim \delta_1$, and $v_i\sim \cN(0,1)$ i.i.d.
    
     For the first part, notice that the norms  $\|W_i(\mathbf{x})\|^2 = R_{ii}^\perp + (m_i^\mu)^2 + (m_i^\nu)^2$ and $\|v(\mathbf{x})\|^2 = \sum_{i=1}^K v_i^2$ are  expressible in terms of the summary statistics. Therefore, their boundedness follows from Lemma~\ref{lem:xor-confined-to-compact-region}, pulled back to the summary statistics of the SGD per Lemma~\ref{lem:closeness-to-ballistic-trajectory}.  

     For the second item, by Corollary~\ref{cor:lambda-finite-close-to-lambda-infinite} and Lemma~\ref{lem:fixed-point-chosen-by-lambda-infinity-dynamics} together with Observation~\ref{obs:xor-fixed-points-in-principal-directions}, for every $\epsilon>0$, there is a $T_0$ such that for every $T_f$, for all $t\in [T_0,T_f]$, the dynamical system of Proposition~\ref{prop:XOR-effective-dynamics-finite-lambda} is within distance $\epsilon+ \lambda^{-1/2}$ of a point $\mathbf{u}_\star$ having $|v_i(t)|>\eta$ and $\max\{|m_i^\mu|,|m_i^\nu|\}>\eta$ in and having that its first layer lives in $\text{Span}(\mu,\nu)$ and its second layer lives in $\text{Span}((g^\vartheta(\mathbf{u}_\star))_{\vartheta \in \{\pm\mu,\pm\nu\}})$. This is pulled back to the summary statistics applied to the SGD trajectory per Lemma~\ref{lem:closeness-to-ballistic-trajectory} (with the observation that $(g^\vartheta(\mathbf{x}))_\vartheta$ are functions of only the summary statistics). 
\end{proof}

\section{Concentration of Hessian and G-matrices}\label{sec:Hessian-concentration}
We recall the general forms of the empirical test Hessian matrix $\nabla^2 \widehat{R}(\mathbf{x})$ and G-matrix $\widehat G(\mathbf{x})$ from~\pef{eq:test-Hessian-Gram}. Our aim in this section is to establish concentration of those empirical matrices about their population versions throughout the parameter space.

\subsection{Hessian and G-matrix: 1-Layer}
We first prove the concentration of the empirical Hessian and G-matrix for the $k$-GMM problem about their population versions, which we analyzed in depth in Section~\ref{s:1-layer-pop}--\ref{s:multilayer-pop}. This concentration will be uniform over the entire parameter space. Namely, our aim in this section is to show the following.
\begin{thm}\label{t:concentration1}
Consider the $k$-GMM data model of~\pef{eq:data-distribution} and sample complexity $\widetilde M/d=\alpha$. There are constants $c=c(k),C=C(k)$ (independent of $\lambda$) such that for all $t>0$, the empirical Hessian matrix concentrates as  
\begin{align}\label{e:hessionc}
\sup_{{\mathbf x \in \mathbb{R}^{kd}}}\mathbb P(\norm{\nabla^{2}({\widehat R}(\mathbf x)-\bE[{\widehat R}(\mathbf x)])}_{\op}>t)\le \exp\big(-[c\alpha(t\wedge t^{2})-C]d\big)\,,
\end{align}
and so does the empirical G-matrix
\begin{align}\label{e:gramc}
\sup_{{\mathbf x\in \mathbb{R}^{kd}}}\mathbb P(\norm{{\widehat G}(\mathbf x)-\bE[{\widehat G}(\mathbf x)]}_{\op}>t)\le \exp\big(-[c\alpha(t\wedge t^{2})-C]d\big)\,.
\end{align}
\end{thm}

\begin{proof}In the following we fix $\mathbf{x}\in \bR^{kd}$ and simply write $\widehat R(\mathbf{x}), \widehat G(\mathbf{x})$ as $\widehat R, \widehat G$. Let $\widetilde A=(\widetilde { Y}^{1}\cdots \widetilde { Y}^{\widetilde M})$ denote the test data matrix and let 
\begin{align}
& D^{\rm H}_{bc}=\diag(\pi_{\widetilde{Y}^{\ell}}(c)\delta_{bc}-\pi_{\widetilde {Y}^{\ell}}(c)\pi_{\widetilde{Y}^{\ell}}(b))_{1\leq \ell\leq \widetilde M}\,, \label{eq:Hessian-D-matrix}\\
&D^{\rm G}_{bc}
    =\diag(({\widetilde y}^\ell_c {\widetilde  y}^\ell_b-\pi_{\widetilde{Y}^\ell}(b){\widetilde y}^\ell_c-{\widetilde y}^\ell_b\pi_{\widetilde{Y}^\ell}(c)+\pi_{\widetilde{Y}^\ell}(c)\pi_{\widetilde{Y}^\ell}(b))_{1\leq \ell\leq {\widetilde M}}\,, \label{eq:Gram-D-matrix}
\end{align}
where $\pi_Y(c)$ was defined in~\pef{eq:pi-dist}. 
Then $D^{\rm H}_{bc}, D^{\rm G}_{bc}$ are $\widetilde M\times \widetilde M$ diagonal matrices for each pair $bc\in [k]^2$. We denote the $(k\widetilde M)\times (k\widetilde M)$ matrices $ {\mathbf{D}}^{\rm H}=({D}^{\rm H}_{bc})_{bc}$ and $ {\mathbf{D}}^{\rm G}=({D}^{\rm G}_{bc})_{bc}$, and the $dk\times k\widetilde M$ matrix
\[
\widetilde A^{\times k}=I_{k}\tensor \widetilde A\,.
\]
With these notations, per~\pef{eq:test-Hessian-Gram}, we can rewrite the Hessian and G-matrices as
\begin{align}
\label{e:Hessian}\nabla^2\widehat R&=\frac{1}{\widetilde M}{\widetilde A}^{\times k}{\mathbf{D}}^{\rm H}({\widetilde A}^{\times k})^{T}\,, \\
\label{e:Gram}\widehat G&=\frac{1}{\widetilde M}{\widetilde A}^{\times k}{\mathbf{D}}^{\rm G}({\widetilde A}^{\times k})^{T}\,.
\end{align}

To prove that the operator norm of $\nabla^{2}({\widehat R}-\bE[\widehat R])$ concentrates, we'll
use a net argument over the unit ball in $\R^{dk}$ to show that the following concentrates
\begin{align}\label{e:supnorm}
\sup_{\mathbf{v}\in(\mathbb{R}^{d})^{k},\norm{\mathbf{v}}=1}\big|\big\langle \mathbf{v},\nabla^{2}(\widehat R-\bE[\widehat R])\mathbf{v}\big\rangle \big|\,,
\end{align}
where $\mathbf{v}=(v_c)_{c\in  [k]}\in (\bR^d)^k$.

By plugging \pef{e:Hessian} into \pef{e:supnorm}, we want a concentration estimate for  $F(\mathbf v) = \langle \mathbf{v},\nabla^2 (\widehat R - \mathbb E[\widehat R])\mathbf{v}\rangle$, which we can rewrite as follows. 
\begin{align}\begin{split}\label{e:defF}
F(\mathbf v) & =\frac{1}{\widetilde M}\big\langle {\mathbf v},\widetilde A^{\times k}\mathbf{D}^{\rm H}(\widetilde A^{\times k})^{T}{\mathbf v}\big\rangle -\big\langle {\mathbf v},\bE[\widetilde A^{\times k}\mathbf{D}^{\rm H}(\widetilde A^{\times k})^{T}]{\mathbf v}\big\rangle \\
 & =\sum_{a,b}\frac{1}{{\widetilde M}}\big\langle v_{a},\widetilde AD^{\rm H}_{ab}\widetilde A^{T}v_{b}\big\rangle -\big\langle v_{a},\bE[\widetilde AD^{\rm H}_{ab}\widetilde A^{T}]v_{b}\big\rangle \\
 & =\sum_{a,b}\frac{1}{\widetilde M}\sum_{\ell}d_{\ell}(a,b)\big\langle v_{a},\widetilde{Y}^{\ell}\big\rangle \big\langle v_{b},\widetilde{Y}^{\ell}\big\rangle -\E \big[d_{\ell}(a,b)\big\langle v_{a},\widetilde{Y}^{\ell}\big\rangle \big\langle v_{b},\widetilde{Y}^{\ell}\big\rangle\big]\,,
\end{split}\end{align}
where $d_{\ell}(a,b)=[D_{ab}^{\rm H}]_{\ell\ell}$.
Recalling \citet[Section 4.4.1]{vershynin2018high} that for an $\epsilon$-net $\cN_\epsilon$ of $\{\mathbf{v}:\|\mathbf{v}\|=1\}$, and any  real symmetric matrix $H$ the following holds 
\[
\frac{1}{1-2\epsilon}\sup_{{\mathbf v}\in\cN_{\epsilon}}\abs{\left\langle {\mathbf v},H{\mathbf v}\right\rangle }\geq\norm{H}_{\op}=\sup_{\|{\mathbf v}\|_2=1}\abs{\left\langle {\mathbf v},H{\mathbf v}\right\rangle }\geq\sup_{{\mathbf v}\in\cN_{\epsilon}}\abs{\left\langle {\mathbf v},H{\mathbf v}\right\rangle }\,.
\]
So by a union bound and the $\epsilon$-covering number of $\{\mathbf{v}\in \mathbb R^{kd}: \|v\|=1\}$, we have 
\begin{align}\label{e:supbound}
\bP\Big(\sup_{\|{\mathbf v}\|_2=1}|F({\mathbf v})|>t\Big)\leq|\cN_{\epsilon}|\sup_{{\mathbf v}\in\cN_{\epsilon}}\bP\big(|F({\mathbf v})|>t/2\big)\leq (C/\epsilon)^{kd}\bP\big(|F({\mathbf v})|>t/2\big)\,,
\end{align}
so long as $\epsilon<1/4$, say. 

To control the last quantity $\bP(|F({\mathbf v})|>t/2)$, we notice that $F({\mathbf v})$ is a sum of $O(1)$ many (in fact $k^{2}$ many) terms (corresponding to each pair $a,b$) so
it suffices by a union bound to control the concentration of each
summand. That is, it suffices to understand the concentration of 
\begin{align}\label{eq:need-to-control}
\frac{1}{\widetilde M}\sum_{\ell} \Big(d_{\ell}(a,b)\big\langle v_{a},\widetilde{Y}^{\ell}\big\rangle \big\langle v_{b},\widetilde{Y}^{\ell}\big\rangle  - \mathbb E\big[d_{\ell}(a,b)\big\langle v_{a},\widetilde{Y}^{\ell}\big\rangle \big\langle v_{b},\widetilde{Y}^{\ell}\big\rangle\big]\Big)\,.
\end{align}
We recall that the test data 
\begin{align}\label{e:data}
    \widetilde{Y}^{\ell}=\sum_{a\in [k]} \widetilde y^\ell_{a} \mu_a + Z^\ell_\lambda:=\mu^\ell+Z_\lambda^\ell\,,
\end{align}
where $Z_\lambda^\ell$ are i.i.d.\ $\cN(0,I_d/\lambda)$, and $\mu^\ell\sim \sum_{a\in [k]} p_a \delta_{\mu_a}$.
In this case, note that $d_{\ell}(a,b)=\pi_{\widetilde{Y}^{\ell}}(b)\delta_{ab}-\pi_{\widetilde {Y}^{\ell}}(b)\pi_{\widetilde{Y}^{\ell}}(a)$ are uniformly
bounded by $2$, and that for $1\leq \ell\leq \widetilde M$
\begin{align}
  \langle v_{a},\widetilde {{Y}}^{\ell}\rangle  \stackrel{d}{=}\langle v_a,\mu^\ell\rangle+\langle v_a, Z_\lambda^\ell \rangle\,,
\end{align}
are i.i.d.\ sub-Gaussian (with norm $\OO(1)$) for fixed $v_a$, since $\|v_a\|_2\leq 1$ and $\|\mu_c\|_2=1$ for all $c$, and $\lambda \ge 1$ say.  
Thus, for all $a,b$, the products
$(d_{\ell}(a,b)\langle v_{a},\widetilde {{Y}}\rangle \langle v_{b},\widetilde {{Y}}^{\ell}\rangle )_{\ell}$
are i.i.d.\ uniformly sub-exponential random variables.  As such,~\pef{eq:need-to-control} is a sum of i.i.d.\ centered uniformly sub-exponential random variables, and Bernstein's inequality
yields that there exists a small constant $c=c(k)$
\begin{align}\label{e:Berstein}
\bP(|F(\mathbf v)|>t/2)\leq\exp(-c {\widetilde M}(t\wedge t^{2}))\,.
\end{align}
It follows from plugging \pef{e:Berstein} into \pef{e:supbound} and taking $\epsilon=1/8$,  that
\[
\mathbb P\Big(\sup_{\|\mathbf v\|_2=1}|F(\mathbf v)|>t/2\Big)\leq\exp\big(-c {\widetilde M}(t\wedge t^{2})+ C k d\big)\,,
\]
which yields the concentration bound \pef{e:hessionc} for empirical Hessian matrix, by noticing that $\widetilde M/d=\alpha$.

For the proof of the concentration bound \pef{e:gramc} for empirical G-matrix, thanks to \pef{e:Gram}, we can define $F(\mathbf{v})$ (as in \pef{e:defF}) using $\mathbf{D}^{\rm G}$ instead of $\mathbf{D}^{\rm H}$. Noticing that the entries $\mathbf{D}^{\rm H}_{ab}$ are bounded by $4$, the concentration of this new $F(\mathbf{v})$ follows from Bernstein's inequality by the same argument. Then the concentration estimates of this new $F(\mathbf{v})$ together with a epsilon-net argument gives \pef{e:gramc}.
\end{proof}

\subsection{Hessian and G-matrix: $2$-layer GMM model}

For the reader's convenience, we recall the empirical Hessian and G-matrix from the beginning of Section \ref{s:multilayer-pop}. Our main aim in this subsection is to prove the following analogue of Theorem~\ref{t:concentration1} for the XOR GMM with a 2-layer network. There is a small problem with differentiating the ReLU activation twice exactly at zero, so let us define the set $\mathcal W_{0}^c = \bigcap_{i\le K} \{W_i \not \equiv 0\}$. 

\begin{thm}\label{t:concentration2}
We consider the $2$-layer XOR GMM model from~\pef{eq:XOR-data-distribution}--\pef{eq:XOR-loss} with sample complexity $\widetilde M/d = \alpha$.  For any $L$, there are constants $c=c(K, L),C=C(K,L)$ such that for all $t>0$, the empirical Hessian matrix concentrates as
\begin{align}\label{e:hessionc2}
\sup_{\substack{\|(v,W)\| \le L \\ {W} \in \mathcal W_0^c}}P(\norm{\nabla^{2}({\widehat R}( v, W)-\bE[{\widehat R}(v,W)])}_{\op}>t)\le\exp\{-[c\alpha(t\wedge t^{2})-C]d\}\,,
\end{align}
and the empirical G-matrix concentrates as
\begin{align}\label{e:gramc2}
\sup_{{\|(v,W)\| \le L}}P(\norm{{\widehat G}(v, W)-\bE[{\widehat G}( v, W)]}_{\op}>t)\le\exp\{-[c\alpha(t\wedge t^{2})-C]d\}\,.
\end{align}
\end{thm}
\begin{proof}
We begin by recalling some expressions. Letting $\widehat y^\ell = \sigma(v\cdot g(W \widetilde{Y}^\ell))$ as in~\pef{eq:yhat}, by~\pef{eq:XOR-Hessian}, we have 
\begin{align}\begin{split}\label{e:recall-2lhessian}
\nabla_{vv}^{2}\widehat R & =\frac{1}{\widetilde M}\sum_{\ell=1}^{\widetilde M}\widehat{y}^\ell(1-\widehat{y}^\ell) g(W{\widetilde Y^\ell})^{\tensor2}\,,\\
\nabla_{W_{i}W_{j}}^{2}\widehat R & =\frac{1}{\widetilde M}\sum_{\ell=1}^{\widetilde M}(\delta_{ij}v_{i}g''(W_{i}\cdot {\widetilde Y^\ell})(\widetilde y^\ell-\widehat{y}^\ell)+v_{i}v_j \widehat y^\ell(1-\widehat y^\ell)g'(W_{i}\cdot {\widetilde Y^\ell})g'(W_{j}\cdot {\widetilde Y^\ell})(\widetilde Y^\ell)^{\tensor2}\,,\\
\nabla_{vW_{j}}^{2}\widehat R & =\frac{1}{\widetilde M}\sum_{\ell=1}^{\widetilde M}((\widetilde y^\ell-\widehat{y}^\ell) g'(W_{j}\cdot {\widetilde Y^\ell}){\mathbf e}_j+v_{j}\widehat{y}^\ell(1-\widehat{y}^\ell)g'(W_j\cdot {\widetilde Y^\ell})g(W\cdot {\widetilde Y^\ell}) )\otimes{\widetilde Y^\ell}\,,
\end{split}\end{align}
and by~\pef{e:gv}--\pef{e:gw}, we have 
\begin{align}\begin{split}\label{e:recall-gram}
    \widehat G_{vv}  &=\frac{1}{\widetilde M}\sum_{\ell=1}^{\widetilde M}(\widetilde y^\ell-\widehat{y}^\ell)^2 g(W \widetilde Y^\ell)\otimes g(W\widetilde 
 Y^\ell)\\
\widehat G_{W_i W_j} & =\frac{1}{\widetilde M}\sum_{\ell=1}^{\widetilde M}(\widetilde y^\ell-\widehat{y}^\ell)^2 v_i v_j g'(W_{i}\cdot \widetilde Y^\ell)g'(W_{j}\cdot \widetilde Y^\ell) \widetilde Y^\ell\otimes \widetilde Y^\ell,\\
\widehat G_{v W_j} & =\frac{1}{\widetilde M}\sum_{\ell=1}^{\widetilde M}(\widetilde y^\ell-\widehat{y}^\ell)^2 v_j g'(W_{j}\cdot \widetilde Y^\ell)  g(W\cdot \widetilde Y^\ell)\otimes\widetilde Y^\ell.
\end{split}\end{align}

Recalling the data distribution from \pef{eq:XOR-data-distribution}, so long as $W\in \mathcal W_0^c$,  almost surely, all the coordinates of $WY$ are nonzero: as long as $\lambda<\infty$, 
\begin{align}
    \bP(\exists i, W_i\widetilde Y^\ell=0)=0\,, \quad \text{for all $W\in \mathcal W_0^c$}\,.
\end{align}
For any $W\in \mathcal W_0^c$, we thus can ignore the second derivative term $g''(W_i\cdot\widetilde Y^\ell)$ in $\nabla_{W_{i}W_{j}}^{2}\widehat R$, so a.s.\ 
\begin{align}\label{e:2lhessianalmost}
    \nabla_{W_{i}W_{j}}^{2}\widehat R & =\frac{1}{\widetilde M}\sum_{\ell=1}^{\widetilde M}(v_{i}g'(W_{i}\cdot {\widetilde Y^\ell}))(v_{j}g'(W_{j}\cdot {\widetilde Y^\ell}))\widehat y^\ell(1-\widehat y^\ell)(\widetilde Y^\ell)^{\tensor2}\,.
\end{align}
With these calculations in hand, the proof is similar to that of Theorem \ref{t:concentration1}, we will only emphasize the main differences. Let $B$ be the ball of radius $L$ in parameter space. 
Fix $(v,W)\in B$ and simply write $\widehat R(v,W), \widehat G(v,W)$ as $\widehat R, \widehat G$.
Our goal is to prove concentration of the operator norm 
\[
\sup_{\substack{(\mathbf a, \mathbf u)\in\R^{K}\times(\R^{d})^{K} \\  \|(\mathbf{a,u})\|=1}}\big\langle (\mathbf a, \mathbf u),(\nabla^{2}(\widehat{R}-\bE[\widehat R])(\mathbf a, \mathbf u)\big\rangle, 
\]
where $\mathbf a=(a_1,a_2,\cdots, a_K)$ and $\mathbf u=(u_1,u_2,\cdots, u_K)\in (\bR^d)^K$.
As in the proof of Theorem \ref{t:concentration1}, specifically~\pef{e:supbound}, by an epsilon-net argument, we only need prove a concentration estimate for the following quantity
\begin{align}\label{e:defF2}
F(\mathbf a, \mathbf u)=\big\langle (\mathbf a, \mathbf u),\nabla^{2}(\widehat{R}-\bE[\widehat{R}])(\mathbf a, \mathbf u)\big\rangle \,,
\end{align}
individually per $(\mathbf{a},\mathbf{u}): \|(\mathbf{a},\mathbf{u})\| = 1$. 
The inner product on the right-hand side of \pef{e:defF2} splits into $(K+1)^2$ terms corresponding the  terms in \pef{e:recall-2lhessian} and \pef{e:2lhessianalmost}:
\begin{align}\label{e:2lhessian2}
\langle {\mathbf a}, \nabla_{vv}^{2}\widehat R {\mathbf 
 a}\rangle & =\frac{1}{\widetilde M}\sum_{\ell=1}^{\widetilde M}\widehat{y}^\ell(1-\widehat{y}^\ell) \langle {\mathbf a}, g(W{\widetilde Y^\ell})\rangle^{2}\,, \nonumber\\
\langle u_i, \nabla_{W_{i}W_{j}}^{2}\widehat R u_j\rangle & =\frac{1}{\widetilde M}\sum_{\ell=1}^{\widetilde M}(v_{i}g'(W_{i}\cdot {\widetilde Y^\ell}))(v_{j}g'(W_{j}\cdot {\widetilde Y^\ell}))\widehat y^\ell(1-\widehat y^\ell) \langle u_i,\widetilde Y^\ell\rangle \langle u_j,\widetilde Y^\ell\rangle\,,\\
\langle \mathbf a, \nabla_{v W_{j}}^{2}\widehat R u_j\rangle & =\frac{1}{\widetilde M}\sum_{\ell=1}^{\widetilde M}(a_jg'(W_{j}\cdot {\widetilde Y^\ell})(\widetilde y^\ell-\widehat{y}^\ell)+(v_{j} \widehat{y}^\ell(1-\widehat{y}^\ell)g'(W_{j}\cdot {\widetilde Y^\ell})\langle \mathbf a, g(W\cdot {\widetilde Y^\ell})\rangle)\langle{\widetilde Y^\ell}, u_j\rangle\,. \nonumber
\end{align}
The concentration bound for $F(\mathbf a, \mathbf u)$ follows from concentration bounds for each of the above terms about their respective expected values. Each such term minus its expected value will be a sum of $\widetilde M$ i.i.d.\ centered random variables. It remains to show that the summands are uniformly (in $\lambda$ and $B_L$) sub-exponential. Then the concentration of $F(\mathbf a, \mathbf u)$ follows from Bernstein's inequality as in~\pef{e:Berstein}. 

To that end, we notice that $(\widehat y^\ell)(1-\widehat y^\ell)$, $\widetilde y^\ell$ are bounded by $1$, and $v_i, v_j$ and $\|g'\|_\infty$ are all bounded by a $C(L)$. By the same argument as in the proof of Theorem \ref{t:concentration1}, we have that $\langle u_i,\widetilde Y^\ell\rangle$ and $\langle u_j,\widetilde Y^\ell\rangle$ are uniformly sub-Gaussian.

Next we show that $\langle g(W\widetilde Y^\ell), \mathbf a\rangle$ is also sub-Gaussian.
In law, conditionally on the mean choice among $\pm \mu,\pm \nu$, we have  $W \widetilde Y^\ell\stackrel{d}=W(\pm \mu+Z_{\lambda})$ or $W \widetilde Y^\ell\stackrel{d}=W(\pm \nu+Z_{\lambda})$. Since $(\mathbf {a, u})\in B$ and $\|\mu\| =  \|\nu\|=1$, we have by Cauchy--Schwarz that  $\|W\mu\|, \|W\nu\| \le L$. Then we can write 
\begin{align*}
    \langle g(W \widetilde Y^\ell), {\mathbf a}\rangle
    =\langle g(W  Z_{\la}), {\mathbf a}\rangle
    +\langle g(W\widetilde Y^\ell)-g(WZ_{\lambda}), {\mathbf a}\rangle\,.
\end{align*}
By the uniform Lipschitz continuity of $g$, we have
\begin{align*}
    |\langle g(W \widetilde Y^\ell)-g(WZ_\la), {\mathbf a}\rangle|
    \leq \|{\mathbf a}\|(\|W\nu\|+\|W\mu\|)\le 2\,.
\end{align*}
We also notice that if $Z_1 = \sqrt{\lambda}Z_\lambda$, then 
\begin{align*}
    \nabla_{Z_1} \langle g(WZ_\la), {\mathbf a}\rangle
    &=\frac{1}{\sqrt \lambda}\sum_{i=1}^K  a_i g'(W_iZ_\la)W_i
   \leq \frac{1}{\sqrt \lambda}\|g'\|_{\infty}\|\mathbf a\|_2 \|W\|_{2\rightarrow 2}\lesssim \frac{L}{\sqrt{\lambda}}\,,
\end{align*}
implying that it is a uniformly $L$-Lipschitz function of standard Gaussians, from which it follows that $g(W Z_\lambda)$ is uniformly sub-Gaussian (see \citet[Section 5.2.1]{vershynin2018high}). For the expectation we have
\begin{align*}
    |\bE[\langle g(WZ/\sqrt \la), \mathbf a\rangle]|
    &\leq   |\bE[\langle g(\bm 0), \mathbf a\rangle]|+|\bE\langle g(WZ/\sqrt \la)-g(\bm 0), \mathbf a\rangle|\\
    &\leq |g(0)|\sqrt K\|\mathbf a\|_2+\|g'\|_\infty\bE[\|WZ/\sqrt\lambda\|_2]\|\mathbf a\|_2\\
    &\leq |g(0)|\sqrt K\|\mathbf a\|_2+\frac{1}{\lambda}\|g'\|_\infty\|\mathbf a\|_2\sqrt{\Tr[WW^\top]}\leq C(K,L)\,.
\end{align*}
We conclude that $\langle g(W \widetilde Y^\ell), {\mathbf a}\rangle$ is sub-Gaussian with norm bounded by $\OO(1)$ (depending only on $K, L$.)
As a consequence, each summand in \pef{e:2lhessian2} is sub-exponential, with norm uniformly bounded by $\OO(1)$ (depending only on $K, L$.) Bernstein's inequality gives that there exists a small constant $c=c(K,L)$
\begin{align}\label{e:Berstein2}
\bP(|F(\mathbf{a, u})|>t/2)\leq\exp(-c {\widetilde M}(t\wedge t^{2}))\,.
\end{align}
It follows from  \pef{e:Berstein2} and the epsilon-net argument over the unit ball of $(\mathbf{a, u})\in \bR^K\times \bR^{Kd}$ as in \pef{e:supbound}, that
\[
\bP\Big(\sup_{\|\mathbf (\mathbf{a,u})\|=1}|F(\mathbf{a,u})|>t\Big)\leq\exp\big(-(c {\widetilde M}(t\wedge t^{2})-CK^2d)\big)\,,
\]
which yields the concentration bound \pef{e:hessionc2} for empirical Hessian matrix, by noticing that $\widetilde M/d=\alpha$.

For the proof of the concentration bound \pef{e:gramc2} for the empirical G-matrix, thanks to \pef{e:recall-gram}, we can define $F(\mathbf{v})$ (as in \pef{e:defF2}) using $\widehat G$ instead of $\nabla^2\widehat R$. Then we need concentration bounds for the following quantities 
\begin{align}\begin{split}\label{e:gv2}
    \langle \mathbf a,\widehat G_{vv} \mathbf a\rangle &=\frac{1}{\widetilde M}\sum_{\ell=1}^{\widetilde M}(\widetilde y^\ell-\widehat{y}^\ell)^2 \langle \mathbf a, g(W \widetilde Y^\ell)\rangle^2
 Y^\ell)\,,\\
\langle u_i, \widehat G_{W_i W_j}u_j\rangle & =\frac{1}{\widetilde M}\sum_{\ell=1}^{\widetilde M}(\widetilde y^\ell-\widehat{y}^\ell)^2 v_i v_j g'(W_{i}\cdot \widetilde Y^\ell)g'(W_{j}\cdot \widetilde Y^\ell) \langle u_i,\widetilde Y^\ell\rangle \langle u_j, \widetilde Y^\ell\rangle\,,\\
\langle \mathbf a, \widehat G_{v W_j} u_j\rangle& =\frac{1}{\widetilde M}\sum_{\ell=1}^{\widetilde M}(\widetilde y^\ell-\widehat{y}^\ell)^2 v_j g'(W_{j}\cdot \widetilde Y^\ell)  \langle \mathbf a, g(W\cdot \widetilde Y^\ell)\rangle\langle u_j,\widetilde Y^\ell\rangle\,.
\end{split}\end{align}
By the same argument as that after \pef{e:2lhessian2}, each summand in \pef{e:gv2} is uniformly (with constant only depending on $K,L$) sub-exponential. 
The concentration of this new $F(\mathbf{a,u})$ follows from Bernstein's inequality by the same argument. Then the concentration estimates of this new $F(\mathbf{a,u})$ together with a epsilon-net argument gives \pef{e:gramc2}.
\end{proof}

\section{Proofs of main theorems}\label{s:main-theorem-proofs}
In this section we put together the ingredients we have established in the preceding sections to deduce our main theorems, Theorems~\ref{mainthm:SGD-aligns-with-Hessian}--\ref{mainthm:XOR-all-live-in-subspace}. 

\subsection{1-layer network for mixture of $k$ Gaussians}
We begin with the theorems for classification of the $k$-GMM with a single-layer network. 

\begin{proof}[\textbf{\emph{Proof of Theorem~\ref{mainthm:all-align-with-means}}}]
    We prove the alignment (up to multiplicative error $O(\epsilon + \lambda^{-1})$) of the SGD trajectory, the Hessian, and the G-matrix, with $\text{Span}(\mu_1,...,\mu_k)$ one at a time. 

The results for the SGD were exactly the content of item (2) of Proposition~\ref{prop:k-GMM-SGD-result}. Namely, that item tells us that the part of $\mathbf{x}_\ell^c$ orthogonal to $\text{Span}(\mu_1,...,\mu_k)$ has norm at most $O(\epsilon + \lambda^{-1})$ while $\|\mathbf{x}_\ell^c\|\ge \eta - O(\epsilon + \lambda^{-1}) \ge \eta/2$ for an $\eta$ independent of $\epsilon,\lambda$. Absorbing $\eta^{-1}$ into the big-$O$, this is exactly the definition of living in a subspace per Definition~\ref{def:lives-in-span}. 

We recall from Lemma \ref{l:offblock1}, for $b\ne c$, the $bc$-block of the population Hessian is given by 
\begin{align}\label{e:offblock2}
    \mathbb E[\nabla_{cb}^2 \widehat R(\mathbf{x})] = \sum_{l\in [k]} p_l  \Pi_{cb}^{l}\mu_{l}^{\otimes 2} +\cE^{\rm R}_{cb}\,,
\end{align}
where $\Pi_{cb}^{l}=\bE[\pi_{Y_l}(c)\pi_{Y_l}(b)]$ for $Y_l = \mu_l + Z_\lambda$, and where $\cE^{\rm R}_{cb}$ is a matrix with operator norm bounded by $\OO(1/\lambda)$.
By Lemma \ref{l:dblock1}, the $aa$ diagonal block of the  population Hessian is given by 
\begin{align}\label{e:dbblock2}
    \mathbb E[\nabla_{aa}^2 \widehat R(\mathbf{x})]  = \sum_{l\in [k]} p_l  \Pi_{aa}^{l}\mu_{l}^{\otimes 2}+\cE^{\rm R}_{aa}\,,
\end{align}
where $\Pi_{aa}^{l}=\bE[\pi_{Y_l}(a)(1-\pi_{Y_l}(a))]$ and $\cE^{\rm R}_{aa}$ is a matrix with norm bounded by $\OO(1/\lambda)$. 

The two expressions \pef{e:offblock2} and \pef{e:dbblock2}, together with the concentration of the empirical Hessian matrix Theorem \ref{t:concentration1} implies that blocks of the empirical Hessian concentrate around low rank matrices; i.e., for each $\mathbf{x}$, we have 
\begin{align}\begin{split}\label{e:hessiandecomp}
    \nabla^2_{cb}\widehat R(\mathbf{x})
    &=\sum_{1\leq l\leq k}p_l \Pi_{cb}^{l}\mu_{l}^{\otimes 2}+\cE^{\rm R}_{cb}+(\nabla^2_{cb}\widehat R(\mathbf{x})-\nabla^2_{cb}\bE[\widehat R(\mathbf{x})])\\
    &=:\sum_{1\leq l\leq k}p_l \Pi_{cb}^{l}\mu_{l}^{\otimes 2}+\widehat \cE^{\rm R}_{cb}\,.
\end{split}\end{align}
where for each fixed $\mathbf{x}$, except with probability $e^{-(c\alpha \varepsilon^2-C)d}$
we have $\|\widehat \cE^{\rm R}_{ca}\|\lesssim \varepsilon+\frac{1}{\lambda}$.
Since the test and training data are independent of one another, this also means for any realization of the SGD trajectory $(\mathbf{x}_\ell)_{\ell \le T_f \delta^{-1}}$, and so long as $T_f \delta^{-1} = e^{o(d)}$, by a union bound we get 
\begin{align}\label{eq:uniformly-in-SGD-Hessian-approx}
    \mathbb P\Big( \max_{\ell \le T_f\delta^{-1}} \|\nabla_{cb}^2 \widehat R(\mathbf{x_\ell})  - \sum_{l\in [k]} p_y \Pi_{cb}^l(\mathbf{x}_\ell) \mu_l^{\otimes 2}\| \ge C(\varepsilon + \lambda^{-1})\Big) =  o_d(1)\,.
\end{align}
where we've put in the $\mathbf{x}_\ell$ to emphasize the dependence of $\Pi_{cb}^l$ on the location in parameter space.

Moreover, we recall from item (1) of Proposition~\ref{prop:k-GMM-SGD-result}, that there exists a constant $L$ such that 
    for all $\lambda$ large, with probability $1-o_d(1)$, the SGD trajectory has $\|\mathbf{x}_\ell\|\le L$ for all $T_0 \delta^{-1} \le \ell \le T_f \delta^{-1}$. It follows by definition of $\pi$ that there exists a constant $c=c(L)>0$, such that the coefficients in \pef{e:offblock2} and \pef{e:dbblock2} are lower bounded: $p_l \Pi_{cb}^l(\mathbf{x}_\ell), p_l \Pi_{aa}^l(\mathbf{x}_\ell)\geq c$ for all $T_0 \delta^{-1} \le \ell\le T_f\delta^{-1}$.
Thus the first sum $\sum_{1\leq l\leq k}p_l \Pi_{cb}^{l}\mu_{l}^{\otimes 2}$ on the righthand side of \pef{e:hessiandecomp} is positive definite, and its norm is lower bounded by $c$ (uniformly in $\epsilon,\lambda$). Together with \pef{eq:uniformly-in-SGD-Hessian-approx}, we conclude that the $b,c$ blocks of the test Hessian $\nabla_{bc}^2 \widehat R(\mathbf{x}_\ell)$ live in ${\rm Span}(\mu_1, \mu_2,\cdots, \mu_k)$ up to error $\OO(\epsilon+\lambda^{-1})$. Namely, this is because it satisfies Definition~\ref{def:matrix-lives-in-subspace} with the choice of $M = \widehat{\cE}_{cb}^{\rm R}$, after absorbing $c^{-1}$ into the big-$O$.

By the same argument, thanks to Lemma \ref{l:Gblock1}, the $bc$ block of the population G-matrix is given by 
\begin{align}\label{e:Gblock1}
\delta_{bc} p_b \mu_b^{\tensor 2}
        - p_c \E[\pi_{Y_c}(b)]\mu_c^{\otimes 2} 
        - p_b \E[\pi_{Y_b}(c)]\mu_b^{\otimes 2}+\sum_l p_l \E[\pi_{Y_l}(b)\pi_{Y_l}(c)] \mu_l^{\otimes 2}
        +\cE^{\rm G}_{bc}\,,
\end{align}
where $\cE^{\rm R}_{bc}$ is a matrix with norm bounded by $\OO(1/\lambda)$. 
The expression \pef{e:Gblock1}, together with the concentration of empirical G-matrix Theorem \ref{t:concentration1}, and a union bound over $\ell \le T_f \delta^{-1}$,
implies that blocks of the empirical G-matrix concentrate around 
\begin{align}\label{e:gramdecomp}
    \nabla^2_{bc}\widehat G(\mathbf{x}_\ell)
    &=\delta_{bc} p_b \mu_b^{\tensor 2}
        - p_c \E[\pi_{Y_c}(b)]\mu_c^{\otimes 2} 
        - p_b \E[\pi_{Y_b}(c)]\mu_b^{\otimes 2}  +\sum_l p_l \E[\pi_{Y_l}(b)\pi_{Y_l}(c)] \mu_l^{\otimes 2}
        +\widehat \cE^{\rm G}_{bc}(\mathbf{x}_\ell)\,,
        \end{align}
where except with probability $T_f \delta^{-1} e^{-(c\alpha \varepsilon^2-C)d}$
we have $\|\widehat \cE^{\rm G}_{ca}(\mathbf{x}_\ell)\|\lesssim \varepsilon+\frac{1}{\lambda}$ for all $\ell \le T_f \delta^{-1}$. Namely, the analogue of~\pef{eq:uniformly-in-SGD-Hessian-approx} will apply to $\nabla^2_{bc}\widehat G$ about the low-rank part of the above. In order to deduce the claimed alignment of Theorem~\ref{mainthm:all-align-with-means} it remains to show that each block of the matrix in~\pef{e:gramdecomp} (minus $\widehat{\cE}_{bc}^{\rm G}$ has operator norm bounded away from zero uniformly in $\epsilon,\lambda$. Towards this, recall that we can work on the event that $\|\mathbf{x}_\ell \|\le L$. In the on-diagonal blocks, we can rewrite the part of \pef{e:gramdecomp} that is not $\widehat{\cE}_{bb}^{\rm G}$ as 
\begin{align}\label{e:maint}
    p_b \bE[(1-\pi_{Y_b}(b))^2]\mu_b^{\tensor 2}
        &+\sum_{l\neq b} p_l \E[\pi_{Y_l}(b)^2] \mu_l^{\otimes 2}\,.
\end{align}
This matrix is positive definite and since $\|\mathbf{x}_\ell\|\le L$, there exists $c(L)$ such that the coefficients of $\mu_c$ are all bounded away from zero by $c$, so that the norm of the above is bounded from below by zero for all $(\mathbf{x}_\ell)_{T_0\delta^{-1}\le \ell \le T_f \delta^{-1}}$. 

Thus by \ref{prop:k-GMM-SGD-result}, the coefficients $p_b \bE[(1-\pi_{Y_b}(b))^2]$, $p_l \E[\pi_{Y_l}(b)^2]$ in \pef{e:maint} are lower bounded. Thus \pef{e:maint} is positive definite, there exists a constant $c>0$, such that its norm is lower bounded by $c$.

For $b\neq c$, we can lower bound the operator norm of the matrix
\begin{align}
        - p_c \E[\pi_{Y_c}(b)(1-\pi_{Y_c}(c))]\mu_c^{\otimes 2} 
        - p_b \E[\pi_{Y_b}(c)(1-\pi_{Y_b}(b))]\mu_b^{\otimes 2}+\sum_{l\neq b,c} p_l \E[\pi_{Y_l}(b)\pi_{Y_l}(c)] \mu_l^{\otimes 2}\,,
\end{align}
while $\|\mathbf{x}_\ell\|\le L$ using that if $k>2$ then the last sum contributes some positive portion outside of $\text{Span}(\mu_c,\mu_b)$ which can be used to lower bound the operator norm by some $c(L)$ and if $k=2$, then the first two terms are the only two, are negative definite, and have coefficients similarly bounded away from zero by some $-c(L)$. 
\end{proof}

\begin{proof}[\textbf{Proof of Theorem~\ref{mainthm:SGD-aligns-with-Hessian}}]
    The theorem follows from Theorem~\ref{mainthm:all-align-with-means} together with the observation that the matrices in ~\pef{e:offblock2} and~\pef{e:Gblock1} that are not the error portion $\cE_{cb}^R$ and $\cE_{bc}^G$ respectively, are of rank $k$  since they are sums of $k$ rank-$1$ matrices and as explained in the above proof, each of their eigenvalues are bounded away from zero uniformly in $\epsilon,\lambda$, (in a manner depending only on $L$). 
\end{proof}

\begin{proof}[\textbf{\emph{Proof of Theorem~\ref{mainthm:topeigenvector}}}]
The claims regarding the SGD are proved in Proposition~\ref{prop:k-GMM-SGD-result-orthonormal-means}. 

It remains to prove alignment of the top eigenvectors of the $cc$-blocks of the Hessian and G-matrices with $\mu_c$. Recall from \pef{e:hessiandecomp}, the decomposition of the empirical Hessian matrix 
\begin{align}\begin{split}\label{e:hessiandecomp2}
    \nabla^2_{aa}\widehat R(\mathbf{x})
    &=:\sum_{1\leq l\leq k}p_l \Pi_{aa}^{l}\mu_{l}^{\otimes 2}+\widehat \cE^{\rm R}_{aa}\,,
\end{split}\end{align}
where $\Pi_{aa}^{l}=\bE[\pi_{Y_l}(a)(1-\pi_{Y_l}(a))]$, and with probability $e^{-(c\alpha \varepsilon^2-C)d}$
we have $\|\widehat \cE^{\rm R}_{ca}\|\lesssim \varepsilon+\frac{1}{\lambda}$.

By our assumption that $\mu_1, \mu_2,\cdots,\mu_k$ are orthonormal, thus the first part in the decomposition \pef{e:hessiandecomp2} can be viewed as an orthogonal decomposition. $\nabla^2_{aa}\widehat R(\mathbf{x})$ is a perturbation of $\sum_{1\leq l\leq k}p_l \Pi_{aa}^{l}\mu_{l}^{\otimes 2}$ by $\widehat \cE^{\rm R}_{aa}$. In turn, by the same reasoning as in Lemma~\ref{lem:P-Q-integral-large-lambda}, 
\begin{align*}
    \Pi_{aa}^l = \overline{\Pi}_{aa}^l + O(\lambda^{-1}) \quad \text{where} \quad \overline{\Pi}_{aa}^l = \mathbb E[\bar \pi_{l}(a) (1-\bar \pi_{l}(a))]\,,
\end{align*}
for $\bar \pi_l(a) = e^{m_{al}}/\sum_{b}e^{m_{bl}}$ as in~\pef{eq:pi-bar}. 
By Lemma~\ref{lem:xor-confined-to-compact-region}, there exists $c(\beta)>0$ such that for all $(\mathbf{x}_\ell)_{T_0 \delta^{-1} \le \ell \le T_f \delta^{-1}}$, all the coefficients $p_l \Pi_{aa}^l$ in \pef{e:hessiandecomp2} are lower bounded: $p_l \Pi_{aa}^l\geq c$.
 Thus $\sum_{1\leq l\leq k}p_l \Pi_{aa}^{l}\mu_{l}^{\otimes 2}$ has $k$ positive eigenvalues, $\{p_l \Pi_{aa}^{l}\}_{1\leq l\leq k}$, and each of them is lower bounded by~$c$. The associated eigenvectors are given by $\{\mu_l\}_{1\leq l\leq k}$. We furthermore claim that the one corresponding to $\mu_a^{\otimes 2}$ is separated from the others uniformly in $\epsilon,\lambda$. This follows from the fact that we derived in Proposition~\ref{prop:k-GMM-SGD-result-orthonormal-means} that $\mathbf{x}_\ell$ is within $O(\epsilon + \lambda^{-1})$ distance of a point $\mathbf{x}_\star$ such that $\bar \pi_b(a) = \frac{1}{k-1}(1-\bar \pi_c(c))$ as long as $a\ne b$. This implies that $\bar \pi_l(a)$ for $l\ne a$ is closer to $0$ than $\bar \pi_a(a)$ is close to $1$, whence $\bar \pi_a(a)(1-\bar\pi_a(a))>\bar \pi_l(a)(1-\bar \pi_l(a))$ by an amount that is uniform in $\epsilon,\lambda$. This ensures that along $(\mathbf{x}_\ell)$ the largest eigenvector in $\sum_l \Pi_{aa}^l \mu_a^{\otimes 2}$ is the one with eigenvector $\mu_a$ and the next $k-1$ are those corresponding to $(\mu_l)_{l\ne a}$. Altogether, by eigenvector stability, the top eigenvector of $\nabla^2_{aa}\widehat R(\mathbf{x}_\ell)$ lives in $\text{Span}(\mu_a)$ up to error $\OO(\epsilon+\lambda^{-1})$ and the next $k-1$ all live in $\text{Span}((\mu_l)_{l\ne a})$ up to an error $\OO(\varepsilon+1/\lambda)$ for all $\ell\in [T_0 \delta^{-1},T_f\delta^{-1}]$.

 The statement for the empirical G-matrix is established similarly by using \pef{e:gramdecomp} as input. The part that is different from the above is to see that its top eigenvector in the $aa$-block is the one corresponding approximating by $\mu_a$ and its next $k-1$ are those approximating $(\mu_l)_{l\ne a}$. For that, recall the expression~\pef{e:maint} and use that since $(1-\bar\pi_a(a))>\bar \pi_l(a)$, the coefficient of $\mu_a^{\otimes 2}$ in the $aa$-block is strictly larger than the coefficients of $\mu_l$ for $l\ne a$. 
\end{proof}

\subsection{2-layer network for XOR Gaussian mixture}
We now turn to proving our main theorems for the XOR Gaussian mixture model with a 2-layer network of width $K$. 

\begin{proof}[\textbf{\emph{Proof of Theorem~\ref{mainthm:XOR-all-live-in-subspace}}}]
    We prove the alignment individually for each of the SGD, the Hessian matrix, and the G-matrix. 
    For the SGD trajectory, the claimed alignment was exactly the content of item (2) in Proposition~\ref{prop:XOR-GMM-SGD-result}. 

Letting $F$ denote the cdf of a standard Gaussian, thanks to Lemma~\ref{lem:population-Hessian-approximation},
\begin{align}
  \bE[\nabla_{vv}^2 L]&=\frac{1}{4}\sum_{\vartheta\in\{\pm \mu, \pm \nu\}}\sigma'(v\cdot g(W\vartheta))g(W\vartheta)^{\tensor2}+\cE^{\rm R}_{vv}\,,  \label{e:ELvv2} \\
  \bE[\nabla_{W_i W_i}^2 L]&=\frac{v_i^2}{4}\sum_{\vartheta\in\{\pm \mu, \pm \nu\}}F\left(m_i^\vartheta\sqrt{\frac{\lambda}{R_{ii}}}\right)\vartheta^{\tensor2}+\cE^{\rm R}_{W_iW_i}\,, \label{e:ELWW2}
\end{align}
where the two error matrices satisfy $\|\cE^{\rm R}_{vv}\|, \|\cE^{\rm R}_{W_iW_i}\|=\OO(1/\sqrt \lambda)$. The two expressions \pef{e:ELvv2} and \pef{e:ELWW2},  together with the concentration of empirical Hessian matrix Theorem \ref{t:concentration2} imply that for every $L$, every fixed $\mathbf{x} = (v,W)$ with $W\notin \mathcal W_0^c$, the blocks of the empirical Hessian matrix satisfy
\begin{align}\label{e:vv0}
    \nabla^2_{vv}\widehat R(v,W)
    &=\frac{1}{4}\sum_{\vartheta\in\{\pm \mu, \pm \nu\}}\sigma'(v\cdot g(W\vartheta))g(W\vartheta)^{\tensor2}+\widehat \cE^{\rm R}_{vv}\,,\\
    \nabla^2_{W_iW_i}\widehat R(v,W)
    &=\frac{v_i^2}{4}\sum_{\vartheta\in\{\pm \mu, \pm \nu\}}F\left(m_i^\vartheta\sqrt{\frac{\lambda}{R_{ii}}}\right)\vartheta^{\tensor2}+\widehat \cE^{\rm R}_{W_i W_i}\,,\label{e:WiWi}
\end{align}
where except with probability $e^{-(c\alpha \varepsilon^2-C)d}$,
we have
\begin{align}\label{e:errorla}
    \|\widehat \cE^{\rm R}_{vv}(v,W)\|_{\op},\|\widehat \cE^{\rm R}_{W_iW_i}(v,W)\|_{\op}\lesssim \varepsilon+\frac{1}{\sqrt \lambda}\,.
\end{align}
Using independence of the test and training data, recalling from item (1) of Proposition~\ref{prop:XOR-GMM-SGD-result} that the SGD stays confined to a ball of radius $L$ in parameter space for all $\ell \le T_f \delta^{-1}$, and from item (2) of Proposition~\ref{prop:XOR-GMM-SGD-result} that $\|W_i\|>0$ for all $i$ for all $T_0 \delta^{-1}\le \ell\le T_f \delta^{-1}$, taking a union bound over $\ell\le T_f\delta^{-1}$, we get for $\alpha>\alpha_0$, 
\begin{align}
    \mathbb P\Big( \max_{T_0\delta^{-1}\le  \ell \le T_f\delta^{-1}} & \|\nabla_{vv}^2 \widehat R(\mathbf{x_\ell})  - \frac{1}{4}\sum_{\vartheta\in\{\pm \mu, \pm \nu\}}\sigma'(v\cdot g(W\vartheta))g(W\vartheta)^{\tensor2}\|_{\op} \ge C(\varepsilon + \lambda^{-1/2})\Big) =  o(1)\,. \label{eq:XOR-uniformly-in-SGD-v-Hessian-approx} \\ 
    \mathbb P\Big( \max_{T_0\delta^{-1}\le \ell \le T_f\delta^{-1}} & \|\nabla_{W_iW_i}^2 \widehat R(\mathbf{x_\ell})  - \frac{v_i^2}{4}\sum_{\vartheta\in\{\pm \mu, \pm \nu\}}F\left(m_i^\vartheta\sqrt{\frac{\lambda}{R_{ii}}}\right)\vartheta^{\tensor2}\|_{\op} \ge C(\varepsilon + \lambda^{-1/2})\Big) =  o(1)\,. \label{eq:XOR-uniformly-in-SGD-W-Hessian-approx}
\end{align}
where in the quantity that the blocks of the empirical Hessian are being compared to, $v,W$ evaluated along the SGD trajectory $\mathbf{x}_\ell$.

It remains to show that the low-rank matrices in ~\pef{eq:XOR-uniformly-in-SGD-v-Hessian-approx}--\pef{eq:XOR-uniformly-in-SGD-W-Hessian-approx} have some operator norm uniformly bounded away from zero to deduce the alignment of the empirical test Hessian with the claimed vectors. By item (2) of Proposition~\ref{prop:XOR-GMM-SGD-result}, there exists some constant $c>0$ uniform in $\lambda$, such that with probability $1-o_d(1)$, for all $\ell \le T_f \delta^{-1}$ the SGD $\mathbf{x}_\ell$ is such that 
\begin{align}
    \max_{\vartheta \in \{\pm \mu,\pm \nu\}} \sigma'(v\cdot g(W\vartheta)) \|g(W \vartheta)\|
    =\max_{\vartheta\in \{\pm \mu,\pm \nu\}} \sigma'\big(\sum_i v_i g(m_i^\vartheta)\big) \|g(W \vartheta)\| >c\,,
\end{align}

Thus the deterministic matrix in \pef{eq:XOR-uniformly-in-SGD-v-Hessian-approx} $\sum_{\vartheta\in\{\pm \mu, \pm \nu\}}\sigma'(v\cdot g(W\vartheta))g(W\vartheta)^{\tensor2}$ is positive definite and  its norm is lower bounded away from $0$ for all $(\mathbf{x}_\ell)_{\ell \le T_f \delta^{-1}}$. Together with \pef{eq:XOR-uniformly-in-SGD-v-Hessian-approx}, we conclude that the second-layer test Hessian $\nabla^2_{vv} \widehat R(\mathbf{x}_\ell)$ live in $\text{Span}(g(W(\mathbf{x}_\ell)\cdot \vartheta)_{\vartheta\in \{\pm \mu,\pm \nu\}})$.

For \pef{eq:XOR-uniformly-in-SGD-W-Hessian-approx}, by item (2) of Proposition~\ref{prop:XOR-GMM-SGD-result}, there exists $c>0$ (independent of $\epsilon,\lambda$) such that with probability $1-o_d(1)$, for every $i$, it must be the case that
\begin{align}
     \max_{\vartheta \in \{\pm \mu,\pm\nu\}} \frac{v_i^2}{4}F\big(m_i^\vartheta\sqrt{\frac{\lambda}{R_{ii}}}\big)
    >c\, \quad \text{for all $(\mathbf{x}_\ell)_{\ell \le T_f \delta^{-1}}$}\,.
\end{align}
Then the deterministic matrix in \pef{eq:XOR-uniformly-in-SGD-W-Hessian-approx} is positive definite and has norm uniformly lower bounded away from $0$ for all $\mathbf{x}_\ell$ for $\ell\le T_f \delta^{-1}$. Together with \pef{eq:XOR-uniformly-in-SGD-W-Hessian-approx}, we conclude that the first-layer test Hessian $\nabla^2_{W_iW_i} \widehat R(\mathbf{x}_\ell)$ lives in $\text{Span}(\mu,\nu)$.

By the same argument, thanks to Lemma \ref{lem:population-Gram-approximation}, 
together with the concentration of the empirical G-matrix from Theorem \ref{t:concentration2}, the blocks of the empirical G-matrix concentrate around low rank matrices 
\begin{align}\begin{split}\label{e:vvG}
    \widehat G_{vv}(v,W)
    &=\frac{1}{4}\sum_{\vartheta\in \{\pm \mu, \pm \nu\}}(\sigma(v\cdot g(W\vartheta))-y_\vartheta)^{2}g(W\vartheta)^{\otimes 2}+\widehat \cE^{\rm G}_{vv}\,,\\
    \widehat G_{W_iW_i}(v,W)
    &=v_i^2A+\widehat \cE^{\rm G}_{W_i W_i}\,,
\end{split}\end{align}
where recalling~\pef{eq:A0},
\begin{equation}\label{e:defA}
A=\frac{1}{4}\sum_{\vartheta\in\{\pm\mu,\pm \nu\}}(y_\vartheta - \sigma(v\cdot g(W\vartheta)))^{2}F(m^\vartheta_{i}\sqrt{\frac{\lambda}{R_{ii}}})\vartheta^{\tensor2}\,,
\end{equation}
and for every $\mathbf{x} = (v,W)$ of norm at most $L$, except with probability $e^{-(c\alpha \varepsilon^2-C)d}$
we have $\|\widehat \cE^{\rm G}_{vv}\|,\|\widehat \cE^{\rm G}_{W_iW_i}\|\lesssim \varepsilon+\frac{1}{\sqrt\lambda}$. 
Union bounding over the $T_f\delta^{-1}$ points along the trajectory of the SGD, as in the lead-up to~\pef{eq:XOR-uniformly-in-SGD-v-Hessian-approx}--\pef{eq:XOR-uniformly-in-SGD-W-Hessian-approx}, we get 
\begin{align}
    \mathbb P\Big( \max_{\ell \le T_f\delta^{-1}} & \|\widehat G_{vv}(\mathbf{x_\ell})  - \frac{1}{4}\sum_{\vartheta\in\{\pm \mu, \pm \nu\}}(\sigma(v\cdot g(W\vartheta))-y_\vartheta)^{2}g(W\vartheta)^{\otimes 2}\|_{\op} \ge C(\varepsilon + \lambda^{-1/2})\Big) =  o(1)\,. \label{eq:XOR-uniformly-in-SGD-v-Gram-approx} \\ 
    \mathbb P\Big( \max_{\ell \le T_f\delta^{-1}} & \|\widehat G_{W_i W_i}(\mathbf{x_\ell})  - v_i^2 A\|_{\op} \ge C(\varepsilon + \lambda^{-1/2})\Big) =  o(1)\,. \label{eq:XOR-uniformly-in-SGD-W-Gram-approx}
\end{align}
where in the quantity that the blocks of the empirical G-matrix are being compared to, $v,W$ are evaluated along the SGD trajectory $\mathbf{x}_\ell$.

We recall from item (1) of Proposition~\ref{prop:XOR-GMM-SGD-result} the SGD stays inside the $\ell^2$-ball of radius $L$ for all time. Thus for $\mathbf x=\mathbf x_\ell$, the coefficients  $(\sigma(v\cdot g(W\vartheta))-y_\vartheta)^{2},\sigma(v\cdot g(W\vartheta))^{2},(1-\sigma(v\cdot g(W\vartheta)))^{2}$ in \pef{e:vvG} and \pef{e:defA} are lower bounded away from $0$. By the same argument as for the test Hessian matrix, we conclude that the second-layer test G-matrix $\widehat G_{vv}(\mathbf{x}_\ell)$ lives in $\text{Span}(g(W(\mathbf{x}_\ell) \vartheta)_{\vartheta\in \{\pm \mu,\pm \nu\}})$ and its first layer lives in $\text{Span}(\mu,\nu)$, up to error $O(\epsilon + \lambda^{-1/2})$. 
\end{proof}

\begin{proof}[\textbf{\emph{Proof of Theorem~\ref{mainthm:XOR-lives-in-Hessian}}}]
    The first layer alignment follows from Theorem~\ref{mainthm:XOR-all-live-in-subspace} together with the observation that the part of~\pef{e:WiWi} besides $\widehat{\cE}_{W_i W_i}^R$ is a rank-2 matrix with eigenvectors $\mu,\nu$ and having both eigenvectors bounded away from zero uniformly in $(\epsilon,\lambda)$; this latter fact comes from the observation that the sum of the two cdf's $F(a) + F(-a)=1$ always, and $v_i$ being bounded away from zero per part (2) of Proposition~\ref{prop:XOR-GMM-SGD-result}. A similar behavior ensures the same for the G-matrix's first layer top two eigenvectors per~\pef{e:vvG}--\pef{e:defA} and a uniform lower bound on the sigmoid function over all possible parameters in a ball of radius $L$, as guaranteed by part (1) of Proposition~\ref{prop:XOR-GMM-SGD-result}. 

    For the alignment of the second layer, it follows from Theorem~\ref{mainthm:XOR-all-live-in-subspace} together with the following. First observe that the part of~\pef{e:vv0} that is not $\widehat{\cE}_{vv}^R$ is a rank-4 matrix with eigenvectors $(g(W \vartheta))_{\vartheta\in \{\pm \mu,\pm\nu\}}$. To reason that all of the corresponding eigenvectors are uniformly (in $\epsilon,\lambda$) bounded away from zero, use the uniform lower bound on $\sigma'$ while the parameters are in a ball of radius $L$ about the origin as promised by part (1) of Proposition~\ref{prop:XOR-GMM-SGD-result}, together with the uniform lower bound, for each $i$, on one of $(W_i \cdot \vartheta)_{\vartheta\in \{\pm \mu,\pm \nu\}}$ which holds for the SGD after $T_0(\epsilon)$ per the proof of part (2) of Proposition~\ref{prop:XOR-GMM-SGD-result}. 
\end{proof}

\section{Extension to empirical matrices generated from train data}\label{sec:training-data}
In this section, we discuss how to prove the results for the $k$-GMM model, Theorems~\ref{mainthm:SGD-aligns-with-Hessian}--\ref{mainthm:topeigenvector} in the case the empirical Hessian and empirical G-matrices are generated from the train data itself, assuming $M\gtrsim d\log d$. As the arguments are largely similar we will focus on describing the modifications/new ingredients compared to the proofs in the preceding two sections. In this section, we override the notation to let 
\begin{align}
    \nabla^2 \widehat R(\mathbf{x}) & =  \frac{1}{ M} \sum_{\ell = 1}^{M} \nabla^2 L(\mathbf{x}, {\mathbf{Y}}^\ell)\,, \qquad \text{and}\qquad  \widehat G(\mathbf{x})  = \frac{1}{ M} \sum_{\ell =1}^{ M} \nabla L(\mathbf{x},{\mathbf{Y}}^\ell)^{\otimes 2}\,. \label{eq:train-Hessian-Gram}
\end{align}
both be generated from the train data $\mathbf{Y}^\ell = (y^\ell,Y^\ell)$.

\subsection{Uniform concentration of train matrices}
The key ingredient to extending the results to hold for the train empirical matrices, is to establish a stronger concentration result when $M\gtrsim d\log d$ showing that with high probability, the empirical matrices are close to their population versions everywhere throughout a ball $B_L$ in the parameter space. This will allow us to assume the concentration holds along the SGD trajectory, even though the two are in principle correlated. 

\begin{lem}\label{lem:uniform-in-space-kGMM-concentration}
Consider the $k$-GMM data model of~\pef{eq:data-distribution} and sample complexity $M= \alpha d$. Fix $L$ and let $B_L$ be the ball of radius $L$ about the origin in the parameter space $\mathbb R^{kd}$. There are constants $c=c(k,L),C=C(k,L)$ (independent of $\lambda$) such that for all $t>e^{ - d^{o(1)}}$, the empirical Hessian matrix concentrates as  
\[
\mathbb P\Big(\sup_{x\in B_{r}}\norm{\nabla^2 \widehat{R}(x)-\mathbb E[\nabla^2 \widehat R(x)]}_{\op} \ge t\Big)\leq C\exp(-c([\alpha(t\wedge t^{2})-k\log(d/t)]\wedge1)d)\,.
\]
The same bound holds replacing $\nabla^2 \widehat R$ by $\widehat G$. 
\end{lem}

\begin{proof}
This follows by like the argument Theorem~\ref{t:concentration1} above, but with an additional $\epsilon$-net in the parameter space. We
present only the bound for $\nabla^{2}\hat{R}$ as the bound for $\widehat G$
can be verified analogously. 

Let
$F(x)=\norm{\nabla^{2}\widehat{R}(x)-\mathbb E[\nabla^2 \widehat R(x)]}_{\op}$.
Since we will be performing an $\epsilon$-net in the parameter space, we wish to bound the Lipschitz constant of
$F$; this bound will be the source of the extra $\log d$ we need on $\alpha$ for the probability in the lemma to be small. 
We can then bound 
\[
\abs{F(x)-F(x')}\leq I+\E [I] \quad \text{where} \quad 
I=\sup_{\norm{v}=1}\abs{\langle v,(\nabla^{2}\widehat{R}(x)-\nabla^2\widehat{R}(x'))v\rangle }\,.
\]
 Recall from \pef{e:Hessian} above that we can write 
\[
\nabla^{2}\hat{R}(x)=\frac{1}{M}A^{\times k}\mathbf{D}^H(x)(A^{T})^{\times k}\,,
\]
where $A$ is the matrix whose columns are formed by the data $(Y^\ell)_{\ell =1}^{M}$, and $\mathbf{D}^H$ is the matrix of~\eqref{eq:Hessian-D-matrix}, except with train data instead of test data. 
Thus 
\begin{align*}
I & \leq\frac{1}{M}\sup_{v}\abs{\langle v,A^{\times k}(\mathbf{D}^{H}(x)-\mathbf{D}^{H}(x'))(A^{\times k})^{T}v\rangle }\\
& \leq \norm{\mathbf{D}^H(x)-\mathbf{D}^H(x')}_{\op} \norm{AA^T}_{\op}/M\\
&= \sup_{a,b,\ell}\abs{d_\ell(a,b;x)-d_\ell(a,b;x')} \cdot\norm{AA^{T}}_{\op}/M,
\end{align*}
where as before $d_\ell(a,b;x)=[\mathbf{D}^H_{ab}(x)]_{\ell\ell}$.
Recalling the definition of $\pi_Y$ from~\eqref{eq:pi-dist}, 
\[
\nabla_{x^a} \pi_{Y}(b)=(\pi_Y(b)\delta_{ab}-\pi_Y(a)\pi_Y(b))Y^\ell\,.
\]
Using this and boundedness of $\pi_Y$, we can bound $\sup_{a,b,x} \|\nabla d_\ell(a,b;x)\|\lesssim \|Y^\ell\|$. 

Combining these, we see that for every $x,x'$, we have 
\begin{align*}
I & \lesssim\sup_{\ell}\norm{Y^{\ell}}\cdot(\norm{AA^{T}}_{\op}/M)\cdot \norm{x-x'}\,.
\end{align*}
For all $x,y\in B_L$, we evidently have $\norm{x-x'}\le 2L$.  
Let $E_{K,d}$ denote the event
\[
\{\norm{AA^T}_{\op}/M \leq \sqrt{K}\}\cap \bigcap_{\ell\le M} \{\norm{Y^\ell}\leq\sqrt{K d}\}\,,
\]
on which $I\lesssim LK\sqrt{d} \|x-x'\|$. In order to bound the probability of $E_{K,d}^c$, we first use standard concentration of Gaussian vectors and the fact that means are unit norm to deduce that for every $K>1 + 1/\lambda$, there exists $c(k,K)>0$ such that 
\begin{align*}
    \mathbb P\Big(\bigcup_{\ell \le M} \{\|Y^\ell\|>\sqrt{K d}\}\Big) \le M e^{ - c(k,K)d}\,.
\end{align*}
Similarly, by the fact that the means are unit norm, and the concentration of Gaussian covariance matrices (see e.g.,~\citet[Theorem 4.6.1]{Vershynin}), the probability that $\|A A^T\|_{\op}/M > \sqrt{K}$ is at most $e^{ - c(k,K)d}$, whence a union bound implies for some other $c(k,K)>0$, 
\[
\mathbb P(E_{K,d}^c)\lesssim M e^{-c(k,K)d}\,,
\]
whence for all $x,x'\in B_L$ we have  $ \mathbb P(I\ge L K\sqrt{d}\|x-x'\|) \le M e^{ - c(k,K)d}$. By a similar calculation, we can easily bound the population Hessian's operator norm as $\sup_{x} \mathbb E[\|\nabla^2 \widehat R(x)\|^2]^{1/2}\lesssim_L 1$. 
Thus by Cauchy--Schwarz, we also have for all $x,x'\in B_L$, 
\[
\E [I]\lesssim_L \sqrt{d} \|x-x'\|+ M^{1/2} e^{ - c(k,K)d/2}\,.
\]
So long as $M$ is sub-exponential in $d$, we deduce that for all $x,x'\in B_L$, 
\begin{align*}
    |F(x)- F(x')| \lesssim_L \sqrt{d} \|x-x'\| + e^{ - c d}\,. 
\end{align*}
except with probability $\mathbb P(E_{K,d}^c) \lesssim Me^{ - c(k,K)d}$. 
Fix a $K>2$, say, and take an $\epsilon$-net of $B_L$ called $\cN_\epsilon$, with $\epsilon = t/C\sqrt{d}$ for a large enough constant $C(k,L)$. Then we have that for every $t>e^{-d^{o(1)}}$, say,
\begin{align*}
    \mathbb P(\sup_{x\in B_L} F(x) >t) \le \mathbb P(\sup_{x\in \cN_\epsilon}F(x)>t/2) + Me^{ - c(k)d}
\end{align*}
By a union bound and Theorem~\ref{t:concentration1}, we obtain for some other $C(k,L)$, that 
\begin{align*}
\mathbb P(\sup_{x\in B_L}F(x) & >t)\lesssim \Big(\frac{C\sqrt{d}}{t}\Big)^{dk}e^{ - [c\alpha (t\wedge t^2) - C] d} + Me^{ - cd}\,,
\end{align*}
which is the claimed bound up to change of constants. 
\end{proof}

\subsection{Concluding the alignment proofs for train empirical matrices}
We now reason how to use Lemma~\ref{lem:uniform-in-space-kGMM-concentration} to prove Theorems~\ref{mainthm:SGD-aligns-with-Hessian}--\ref{mainthm:topeigenvector} with train empirical matrices as long as $M\gtrsim d\log d$.  We focus on the modifications one makes to the proofs in Section~\ref{s:main-theorem-proofs}. 

Towards this, let $\mathcal E_0$ be the event that the SGD remains in $B_L$ for all $\ell\le T_f\delta^{-1}$, which holds with probability $1-o_d(1)$ by Lemma~\ref{lem:kGMM-confined-to-compact-region} for a large enough $L(\beta)$. Let $\cE_1$ be the event that $\mathbf{Y}^\ell$ are such that the concentration bound in Lemma~\ref{lem:uniform-in-space-kGMM-concentration} holds, with that choice of $L$, for $t= \varepsilon/2$, say (where $\varepsilon$ is the error we will allow in our alignment claims). This also has probability $1-o_d(1)$ by Lemma~\ref{lem:uniform-in-space-kGMM-concentration} so long as $\alpha \gtrsim \log d$, or equivalently $M\gtrsim d\log d$. We therefore can work on the intersection of $\cE_0 \cap \cE_1$, on which the complement of the event in the probability in~\eqref{eq:uniformly-in-SGD-Hessian-approx} holds for all $\ell \in [T_0 \delta^{-1},T_f \delta^{-1}]$. (This was the step where the independence of the train and test data was previously used.) The analogue for the G-matrix is similarly deduced. With that established, the remainder of the proofs in the $k$-GMM part of Section~\ref{s:main-theorem-proofs} go through unchanged. 

\section{Additional figures}\label{sec:additional figures}
This section includes additional figures for empirical spectra along training in the $k$-GMM model, as well as versions of the figures of Section~\ref{sec:main-results} generated with train data instead of test data. 

\subsection{Additional figures for K-GMM}
In this section, we include more figures depicting the spectral transitions for the $k$-GMM model. 

\begin{figure}[h]
    \centering
 \includegraphics[width=.4\textwidth]{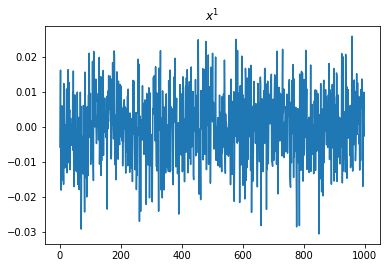}\qquad
\includegraphics[width=.4\textwidth]{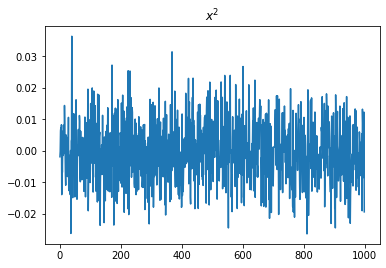}
 \includegraphics[width=.4\textwidth]{KGMM/Xcorr3.png}\qquad
 \includegraphics[width=.4\textwidth]{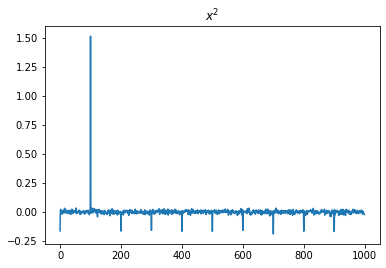}
      \caption{ 
      The SGD coordinates $x^1$ and $x^2$ at initialization $t=0$ (above), and at the end of training (below). Initially $x^i$ is a random vector, but over the course of training it correlates with $\mu_i$ (and anti-correlates with $(\mu_j)_{j\ne i}$), matching the results of Theorem~\ref{mainthm:all-align-with-means} and~\ref{mainthm:topeigenvector} Parameters are the same as in Figure~\ref{fig:KGMM-topspaces}.}
\end{figure}
\begin{figure}[h]
    \centering
 \includegraphics[width=.3\textwidth]{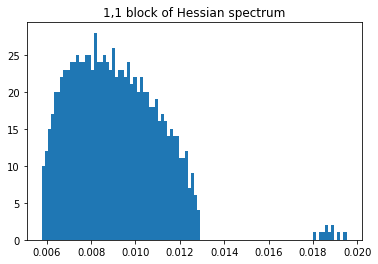}
\includegraphics[width=.3\textwidth]{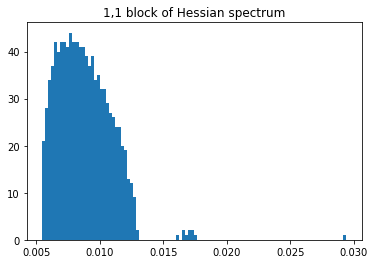}
 \includegraphics[width=.3\textwidth]{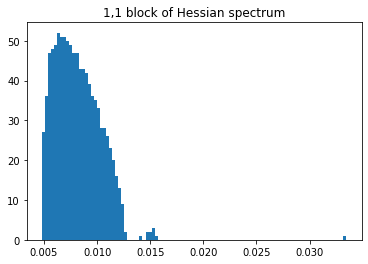}
 
      \caption{ 
     Further illustration of the dynamical spectral transition depicted in Fig.~\ref{fig:DBBP-kgmm}. Here we see that there is initially two components of the spectrum, then over the course of training, the top eigenvalue and the next $9$ separate from each other as proven in Theorem~\ref{mainthm:topeigenvector}.}
\end{figure}

\newpage

\subsection{Training data}\label{sec:additional-figures-train}
We include here variants of plots from Section~\ref{sec:main-results}, in which the empirical matrices are generated using training data, as opposed to independent test data. We begin with the train figures for the $k$-GMM model, then include those for the XOR GMM model. It is easily observed from these that the proven behavior holds just as well for train data as it does for test data empirical matrices.

\begin{figure}[h]\label{fig:KGMM-topspaces-train}
    \centering

   \subfigure{\includegraphics[width=.35\textwidth]{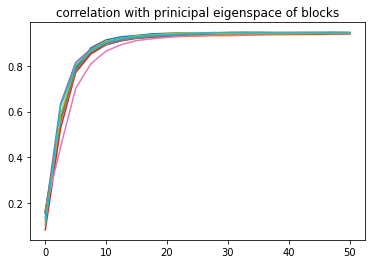}}
   \qquad
\subfigure{        \includegraphics[width=.35\textwidth]{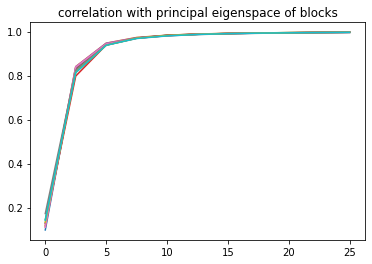}}
      \caption{An analogue of Figure~\ref{fig:KGMM-topspaces} for the situation in which the Hessian (right) and G-matrix (left) are generated with the full batch of train data.}
\end{figure}

\begin{figure}[h]
    \centering
\includegraphics[width=.35\textwidth]{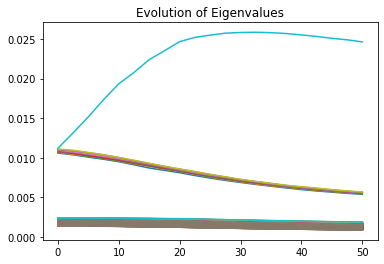}
\qquad 
\includegraphics[width=.35\textwidth]{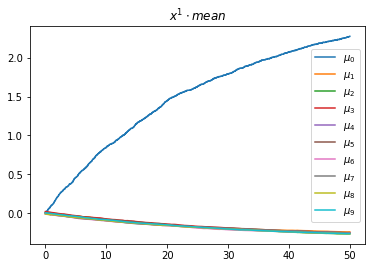}
      \caption{The training data analogue of Fig.~\ref{fig:DBBP-kgmm}. The same phenomenology as with test empirical matrices is easily observed to persist.}
    \label{fig:DBBP-kgmm-train}
\end{figure}

\begin{figure}[h]
    \centering
    \subfigure[]{\includegraphics[width=.235\textwidth]{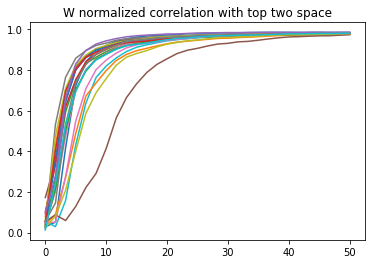}}
           \subfigure[]{ \includegraphics[width=.235\textwidth]{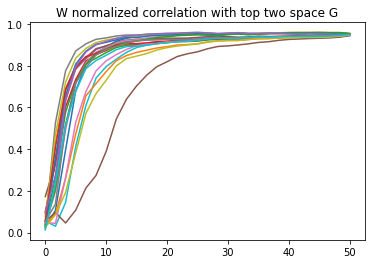}}
\subfigure[]{\includegraphics[width=.235\textwidth]{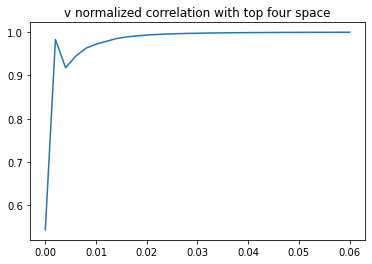}}
   \subfigure[]{\includegraphics[width=.235\textwidth]{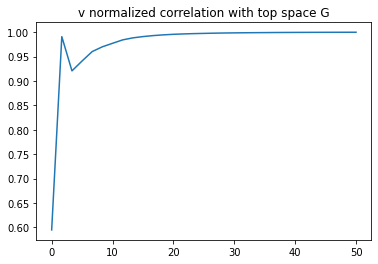}}
      \caption{The training data analogue of Fig.~\ref{fig:XOR-topspaces}. The same phenomenology as with test empirical matrices is easily observed to persist. Parameters are the same as in Figure~\ref{fig:XOR-topspaces}.}
    \label{fig:XOR-topspaces-train} 
\end{figure}

\begin{figure}[h] 
    \centering
 \includegraphics[width=.35\textwidth]{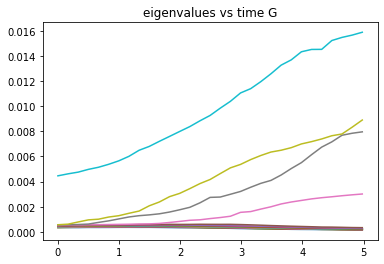}
   \qquad
 \includegraphics[width=.35\textwidth]{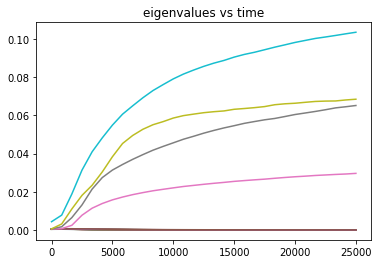}
      \caption{The training data analogue of Fig.~\ref{fig:XOR-topspaces}. The same phenomenology as with test empirical matrices is easily observed to persist.}
    \label{fig:DBBP-XOR-train}
\end{figure}

\subsection*{Acknowledgements}
The authors sincerely thank all anonymous referees for their helpful comments and suggestions. 
G.B.\ acknowledges the support of NSF grant DMS-2134216. 
R.G.\ acknowledges the support of NSF grant DMS-2246780. 
The research of J.H.\ is supported by NSF grants DMS-2054835 and DMS-2331096.
A.J. acknowledges the support of the Natural Sciences and Engineering Research Council of Canada (NSERC) and the Canada Research Chairs programme. Cette recherche a \'et\'e enterprise gr\^ace, en partie, au 
soutien financier du Conseil de Recherches en Sciences Naturelles et en G\'enie du Canada (CRSNG),  [RGPIN-2020-04597, DGECR-2020-00199], et du Programme des chaires de recherche du Canada.

\bibliographystyle{abbrvnat}
\bibliography{references-arXiv-v2.bib}

\end{document}